\documentclass[12pt]{ociamthesis}  

\usepackage{amssymb}
\usepackage{amsthm}
\usepackage{mathtools}
\usepackage{braket}
\usepackage{algorithm}
\usepackage{algpseudocode}
\usepackage{multirow}
\usepackage{hyperref}
\usepackage{booktabs}
\usepackage{pgfplots}
\usepackage{tikz}
\usepackage{caption}
\usepackage{subcaption}
\usetikzlibrary{calc}
\usetikzlibrary{arrows}
\usetikzlibrary{shapes}
\usetikzlibrary{decorations}
\usetikzlibrary{arrows.meta}

\newcounter{definition}
\newenvironment{definition}[1][]{\refstepcounter{definition}\par\bigskip\medskip
   \noindent \textbf{Definition~\thedefinition\ (#1)} \rmfamily}{\medskip}
   
\newcounter{proposition}
\newenvironment{proposition}[1][]{\refstepcounter{proposition}\par\bigskip\medskip
   \noindent \textbf{Proposition~\theproposition\ #1} \rmfamily}{\medskip}
   
\newcounter{example}
\newenvironment{example}[1][]{\refstepcounter{example}\par\bigskip\medskip
   \noindent \textbf{Example~\theexample\ (#1)} \rmfamily}{\medskip}

\usepackage{xcolor}

\title{Learning Logic Programs\\From Noisy Failures}     

\author{John Wahlig}             
\college{St Hilda's College\\Supervised by Dr Andrew Cropper}     

\degree{Master of Science Computer Science}     
\degreedate{Trinity 2021}         

\begin{document}

\baselineskip=18pt plus1pt

\setcounter{secnumdepth}{3}
\setcounter{tocdepth}{3}

\maketitle                  
\begin{acknowledgements}
I first would like to thank Andrew Cropper for his incredible guidance while supervising this project. His insight and expertise was immeasurable. I would next like to thank Rolf Morel for all of his help, advice, and morale boosting conversations. I would also like to thank Brad Hunter for his insightful discussions and for the pleasure of working beside him. I give utmost thanks to my parents for affording me this incredible opportunity and for their endless support and encouragement. Nothing I have achieved would be possible without their generosity and sacrifice. I would lastly like to thank Stephen, Brian, and Abby for inspiring me everyday to be the best that I can be and Alba for just about everything. 
\end{acknowledgements}   
\begin{abstract}
Inductive Logic Programming (ILP) is a form of machine learning (ML) which in contrast to many other state of the art ML methods typically produces highly interpretable and reusable models. However, many ILP systems lack the ability to naturally learn from any noisy or partially missclassified training data. We introduce the \textit{relaxed learning from failures} approach to ILP, a noise handling modification of the previously introduced \textit{learning from failures} (LFF) approach \cite{popper} which is incapable of handling noise. We additionally introduce the novel Noisy Popper ILP system which implements this relaxed approach and is a modification of the existing Popper system \cite{popper}. Like Popper, Noisy Popper takes a generate-test-constrain loop to search its hypothesis space wherein failed hypotheses are used to construct hypothesis constraints. These constraints are used to prune the hypothesis space, making the hypothesis search more efficient. However, in the relaxed setting, constraints are generated in a more lax fashion as to avoid allowing noisy training data to lead to hypothesis constraints which prune optimal hypotheses. Constraints unique to the relaxed setting are generated via hypothesis comparison. Additional constraints are generated by weighing the accuracy of hypotheses against their sizes to avoid overfitting through an application of the minimum description length. We support this new setting through theoretical proofs as well as experimental results which suggest that Noisy Popper improves the noise handling capabilities of Popper but at the cost of overall runtime efficiency.
\end{abstract}          

\begin{romanpages}          
\tableofcontents            
\listoffigures              
\end{romanpages}            

\begin{chapter}
{Introduction}
\section{Motivation}
A major goal for AI is to achieve human-like intelligence through the imitation of our cognitive abilities \cite{lake2016building}. To this end, AI systems often aim to mimic our automatic inductive capacity in which previous (background) knowledge and prior observations are used to infer upon new observations \cite{mitchell1997machine} - a complex task having applications in numerous domains such as image classification \cite{deng2009imagenet} and autonomous navigation \cite{bojarski2016end}. Notably, humans have the innate ability to filter out outlying or incorrect observations, naturally and accurately handling noisy data or incomplete sets of background knowledge. Many machine learning (ML) systems have been implemented to achieve this inductive behavior, capable of identifying patterns among millions of often noisy datapoints. However, traditional ML methods such as neural networks are typically incapable of expressing their models in forms which are easily comprehensible to humans. Additionally, the knowledge learned by many of these systems lacks transferability and cannot be applied to similar problems. For example, an AI system such as AlphaGo \cite{silver2016mastering} which has learned to effectively play the game of Go on a standard $19 \times 19$ size board may struggle greatly on a board of different size. Without comprehensibility and transferability, these systems fail to achieve true levels of human cognition \cite{lake2016building, mitchell2018never}. \textit{Inductive Logic Programming} \cite{muggleton1991inductive} however has been an approach more capable of meeting these additional requirements. 

\bigskip \noindent Inductive logic programming (ILP) is a form of ML wherein a logic program which defines a target predicate is learned given positive and negative examples of the predicate and background knowledge (BK). The target predicate, BK, and examples are represented as logical statements, typically as logic programs in a language such as Prolog whose notation we will use throughout this paper. More precisely, BK defines predicates and ground-truth atoms that the system may use to define the target predicate program. The aim of the system is to learn a program (or hypothesis) that correctly generalizes as many examples as possible, i.e., entails as many positive examples as possible and does not entail as many negative examples as possible.

\begin{example}[ILP Problem] Consider Michalski’s trains problem consisting of 5 trains moving eastbound and 5 moving westbound. We will use this example throughout the paper and Figure 1.1 below visually depicts the original problem. Each train is comprised of a locomotive pulling a variable number of cars, each with distinct characteristics such as length, number of wheels, shape, number of cargo loads, shape of loads, etc. The goal of the problem is to identify a set of one or more rules which distinguishes the eastbound trains from the westbound. For instance, a solution to the original problem shown in Figure 1.1 would be the following rule: \textit{if a train has a car which is short, has two wheels, and its roof is closed, then it is eastbound and otherwise it is westbound}. This problem is easily described in an ILP setting by letting one set of trains, say eastbound, represent positive examples and westbound trains represent negative examples. BK here defines each train and its characteristics with logical predicates for length, number of wheels, shape, etc. Hypotheses to these problems can be easily described using these logical predicates. For example, the rule above would be written as:

\begin{center}
    \texttt{eastbound(A) :- has$\_$car(A,B), short(B), two$\_$wheels(B), roof$\_$closed(B).}
\end{center}

\noindent Here, the hypothesis is claiming that any eastbound train \texttt{A} must have a car \texttt{B} and \texttt{B} must be short, have two wheels, and its roof must be closed. While most modern ILP systems are able to effectively learn classifiers to solve the original trains problem, more complicated variations have been used to compare system capabilities and predictive accuracies.

\begin{figure}[ht]
\centering
\includegraphics[scale = 1.1]{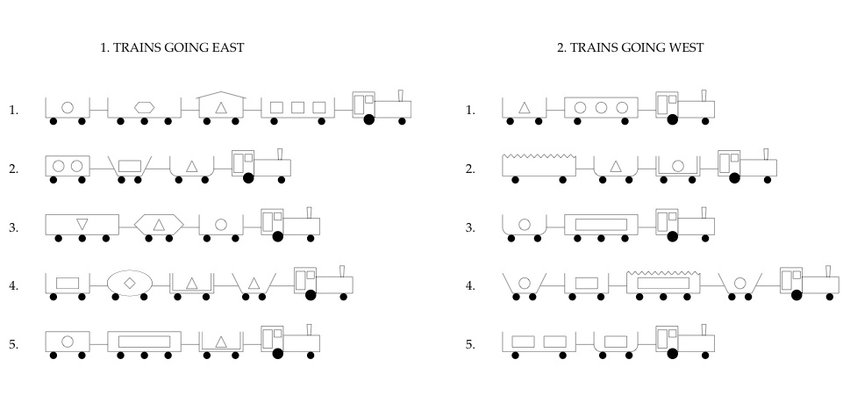}
\caption{Michalski's original east-west trains problem}
\end{figure}
\label{trains}
\end{example}

\bigskip \noindent A strong motivation behind ILP methods is the high level of comprehensibility possessed by logic programs that they return when compared to traditional ML models. Programs described by programming languages possess semantics interpretable by both humans and computers. This falls in line with Michie's \cite{mitchie1988} strong criterion for machine learning and highly explainable AI whereas traditional ML methods often focus solely on improving the weak criterion, i.e., predictive accuracy. Work in this area of ultra-strong machine learning \cite{muggleton2018ultra} has demonstrated how human understanding of a system can be improved by better understanding the  program, thus motivating this desire for comprehensibility. The symbolic nature of the logic programs not only increases their comprehensibility but also allows for ease of lifelong learning and knowledge transfer \cite{torrey2007relational, lin2014bias, cropper2019playgol, cropper2020forgetting}, an essential criteria for human-like AI. Models can be easily reused and complex models can be built up from solving smaller problems whose logic programs are simply placed in the BK. ILP also constitutes a form of code synthesis having applications in code completion, automated programming, and novel program invention. Additionally, unlike many traditional machine learning approaches, ILP can generalize exceptionally well with just a few examples \cite{cropper2020turning}.

\bigskip \noindent At it's core, the problem of ILP is efficiently searching the set of all possible hypotheses, the hypothesis space, which is typically very large. Current ILP systems take many approaches to this problem including but not limited to set covering techniques \cite{quinlan, muggleton1995inverse, nrsample, blockeel1998top, srinivasan2001aleph}, meta-level encodings \cite{cropper2019learning, muggleton2015meta, hexmil}, and neural networks \cite{dilp}, each with various tradeoffs. Simultaneously, as most machine learning methods must, these systems often navigate the issue of noisy datasets (i.e., misclassified examples). While some systems such as TILDE \cite{blockeel1998top}, Progol \cite{muggleton1995inverse}, and Aleph \cite{srinivasan2001aleph} are naturally built to withstand noise to varying degrees, they struggle with building recursive programs, often vital for constructing optimal solutions. Others such as ILASP \cite{ilasp} and Metagol \cite{muggleton2015meta} are inherently incapable of generalizing with noisy data - noise being commonly ignored by ILP systems in exchange for soundness or optimality. However in exchange, ILASP and Metagol are both capable of generating recursive programs and possess varying levels of predicate invention wherein new predicates are created by the system and defined using BK. These both are useful in constructing compact and often optimal solutions. Handling noise is a fundamental issue in machine learning as real-world data is typically endowed with misclassifications and human-errors. As such, machine learning approaches should be able to handle this noise as well as possible, though this problem is never trivial to solve.

\paragraph{Popper} 
Popper \cite{popper} is an ILP system which takes a unique approach of learning from failures (LFF). Popper uses a three stage loop: generate, test, and constrain. However, unlike other three stage approaches \cite{law2018inductive}, Popper uses theta-subsumption \cite{plotkin1972automatic} in conjunction with failed hypotheses to constraint the hypothesis space. Rather than building up a solution to the problem, Popper essentially narrows down the possible solutions significantly enough that the hypothesis space can be efficiently searched. 

\bigskip \noindent In the generate stage, Popper generates a program or hypothesis from the hypothesis space which satisfies a set of hypothesis constraints and subsequently tests the hypothesis against training examples in the test stage. Successful hypotheses are return as solution programs while failed ones are used to generate further hypothesis constraints for subsequent iterations.

\bigskip \noindent While Popper's approach has notable strengths over some existing techniques such as ease of generating recursive programs, it is completely unable to handle noise. Programs generated by Popper necessarily entail all positive examples and no negative ones, clearly overfitting in the presence of any missclassified data. It is the objective of this project to modify Popper's approach to allow it to generalize better to noisy datasets without compromising its overall performance capabilities.

\section{Contribution}
The main contribution of this project is an extension to Popper which can handle noise in exchange for the optimality guarantees of the system. For simplicity, this paper will refer to the original version of Popper as \textit{Normal Popper} and the novel noise handling version of Popper as \textit{Noisy Popper}. The main hypotheses of this project are:
\begin{itemize}
    \item Noisy Popper typically generalizes better than Normal Popper with noisy datasets, being able to return more accurate hypothesis as solutions than Normal Popper which may be incapable of returning any program at all.
    \item Noisy Popper does not lose out in ability to generalize well with noiseless datasets, however, it performs less efficiently than Normal Popper in these environments. 
\end{itemize}

\paragraph{Noise Handling Approach}
Noisy Popper makes several modifications and algorithmic changes to Normal Popper to allow it to better generalize to noise. These changes are as follows:

\begin{itemize}
    \item The first contribution is to alter the LFF framework to one which handles noise. This altered setup is called the Relaxed LFF Framework and it relaxes definitions for solutions and optimally thus also changing how the strictly the hypothesis space is searched.
    \item The second contribution is to introduce theoretically sound constraints in this framework. These sound constraints are used to prune suboptimal hypotheses under the new setting. Some of these use the minimum description length principle to help the system avoid hypotheses which overfit the data. These constraints are described and proved in Propositions~\ref{prop:sound1}-\ref{prop:mdl4} in Chapter 4.
    \item The third contribution is Noisy Popper, which implements this relaxed LFF framework in addition to some enhancements to improve noise handling capacity including any anytime algorithm approach and efficient minimal constraints.
    \item The final contribution is an experimental evaluation of Noisy Popper which demonstrates that Noisy Popper on average generalizes better to noisy datasets than Normal Popper, that Noisy Popper generalizes as well as Normal Popper to noiseless datasets, and how each enhancement used to construct Noisy Popper effects its overall performance.
\end{itemize}

\section{Paper Organization}
This dissertation consists of seven chapters including this brief introduction to the project. The subsequent chapters will be as follows:

\begin{itemize}
    \item Chapter 2 will review related works in the field of ILP including systems which are unable to handle noise, systems which can handle noise, their methods, and how effective they are in practice.
    \item Chapter 3 will cover background information on the LFF framework including a brief review of logic notation and ILP.
    \item Chapter 4 will begin the novel contributions of this paper and discuss the Relaxed LFF framework including the theoretical claims used to justify the framework setup.
    \item Chapter 5 will cover the implementation of Noisy Popper and touch on the implementation of Normal Popper out of necessity.
    \item Chapter 6 will discuss the experiments and empirical results and analysis of the Noisy Popper system.
    \item Chapter 7 is the conclusion and will summarise the paper, its claims and findings, and discuss limitations of the work as well as future work to be considered.
\end{itemize} 
\end{chapter}

\begin{chapter}
{Related Work}
In this chapter, we give a brief overview of the current state of ILP by discussing several systems, their approaches, and their noise handling capacities.

\section{No Noise Handling}
ILP has been a machine learning area of great interest for over three decades \cite{cropper2020turning}. Naturally, many varied approaches to solving ILP problems have been introduced, each with varying degrees of success at handling noisy data, though many take no attempt at all. A common approach to ILP is through the use of metarules \cite{cropper2020logical} which are logical statements defining the syntactic form logic programs may take within the hypothesis space, thus restricting said space. Metagol \cite{muggleton2015meta} is a popular ILP system under this Meta-Interpretive Learning (MIL) setting. Because of the strict nature of these metarules, MIL systems like Metagol often possess higher inductive bias when compared to predicate declarations which simply define which predicates may appear in the head or body of a clause. These predicate declarations are what Popper uses as its language bias (restrictions which define the initial hypothesis space). Metagol additionally allows for automatic predicate invention wherein novel predicates are created using existing predicates and can be used to simplify hypothesis construction. The major drawback of such an approach however is the need for domain expertise as a user typically needs to define the metarules to be used by the system. Additionally, like Popper, Metagol only returns solutions which entail all positive examples and no negative examples, meaning that the system is naturally incapable of generalizing to noise. 

\bigskip \noindent ILASP \cite{ilasp} is another ILP system which cannot handle noise, but takes an Answer Set Programming (ASP) approach. With ASP approaches, the problem itself is encoded as an ASP problem using logical statements or rules. These rules form a type of constraint satisfaction problem which is then solved by a state of the art ASP solver, generating an answer set solution which satisfies the given problem constraints. While effective, these methods carry drawback as all machine learning approaches do. ILASP works in a similar loop to Popper, generating hypothesis and using them to construct ASP constraints to improve the search in subsequent iterations. These constraints are in the form of boolean formulas over the given set of rules. ILASP pre-computes all of these rules using an ASP encoding for the given ILP problem, constructs additional ASP constraints from these encodings, and finally solves an additional ASP problem with these new constraints. Pre-computing all rules is not only computationally expensive, but the system also struggles to learn rules with many body literals. Additionally, ILASP does not typically scale as well as other systems as it requires a large amount of grounding with the programs it generates. With noise, a similar issue to Metagol and Popper exists where the system continues to constrain hypotheses until a solution is found which covers all positive examples and no negative examples.

\bigskip \noindent HEXMIL \cite{hexmil} is an approach which combines the MIL and ASP settings using a HEX-formalism to encode MIL with external sources, reducing the bottleneck produced by the need to ground all atoms. Like the others, this approach fundamentally cannot handle noise as returned hypotheses must entail all positive examples and no negative ones. Contrasting to these approaches, in this project we introduce Noisy Popper which is capable of generalizing to sets of noisy examples as returned solutions may not perfectly entail all positive examples and no negative ones.

\section{Set Covering}
One popular approach to ILP is to use set covering algorithms which progressively learn hypotheses by adding one logical clause at a time, covering a number of positive examples with each. Perhaps the most influential ILP system and one which implements a set covering technique is Progol \cite{muggleton1995inverse}. Its intuitive approach selects a positive example that has not yet been entailed by the program and generates the bottom clause or the most specific clause that entails that example using the minimal Herbrand model of the set of atoms. It then attempts to make this clause as general as possible so that when added to the program constructed so far, it entails as many new positive examples as possible while avoiding entailing negative examples. However, the model contains a noise parameter which controls the quantity of negative examples that are allowed to be entailed. In this way, the system may avoid overfitting the data. However, this hyperparameter leaves much of the noise handling procedure on the user and is not a default mechanism of the system. Significant fine tuning is required for Progol to adequately generalize to noisy datasets, a noticeable burden on the user. Aleph \cite{srinivasan2001aleph} is a popular system based on Progol but built in Prolog. Aleph uses a scoring metric to determine how general to make the bottom clauses. This score can be user defined. As such, it can be selected so that the system is adaptable to noise - a returned solution may not perfectly entail all positive example and no negative examples thus avoiding overfitting. However, like Progol, setting up the hyperparameter environment to accurately learn from noisy data is cumbersome. 

\bigskip \noindent TILDE \cite{blockeel1998top} uses an approach of top-down induction of decision trees \cite{quinlan1986induction} combined with first-order logic to construct a solution as a decision tree. As with traditional binary decision trees, they are constructed to correctly classify the given set of training examples with the left and right splits corresponding to conjunctions in the logical statements constructed, though the model produced is not required to cover all given examples. A tree construction where each example corresponds to a single leaf node/classification is entirely possible and would constitute a form of overfitting, so the system takes steps to avoid this as in a traditional machine learning setting. However, this method again requires fine tuning. Under-pruning the tree can lead to significant overfitting while over-pruning results in small decision trees which do not fit the data well. Noisy Popper however requires no such noise parameters and naturally generalizes to noisy datasets without requiring fine tuning.

\section{Sampling Examples}
Metagol$_{NT}$ \cite{muggleton2018meta} is a noise tolerant extension of the Metagol system, simply acting as a wrapper around the original Metagol algorithm. Metagol$_{NT}$ first generates a random subset of the training examples. It then learns a program which perfectly fits this data using the original Metagol system and finds the accuracy of the resulting program on the remaining unused training data. The system repeats this loop several times and simply returns the program which obtained the highest accuracy. In this way, the returned program will not always perfectly fit all training data as Metagol would, but will often better generalize to noisy datasets. This approach has shown decent results having even been used accurately for some image classification problems \cite{muggleton2018meta}. However, in that same work, the authors address how the approach has limited grasp on noise handling and often fails if noise concentration is too high. The system is largely dependent on a number of factors including the number of training examples, size of the random subsets, and number of candidate programs generated. If too much noise is present in each subset, no program will capture the true underlying data pattern. If the subsets are too small to ensure at least some contain relatively little noise, a similar issues may occur where there are not enough examples to generalize from. Tuning these hyperparameters is not always trivial. Like systems such as Aleph, Metagfol$_{NT}$ also requires a difficult to use noise parameter which determines how many negative examples are permitted to be entailed. Like with most systems, these hyperparameters are difficult to effectively tune, though Noisy Popper lacks them entirely.

\section{Branch and Bound}
ILASP3 \cite{ilasp3} is a noise tolerant version of ILASP taking the form of a branch and bound approach to ILP. Like ILASP, ILASP3 uses an ASP encoding to constrain the search space in a similar generate, test, constrain loop as Popper. With each hypothesis generated, the system tests which positive examples are not covered, determines why, and uses these failed examples to generate additional ASP constrains for the next iteration, pruning the search space. Unlike ILASP, ILASP3 assigns weights to each example. The system then searches for an optimal program which entails the highest sum of weights as possible, rather than simply trying to entail all positive examples and no negative examples. In this way, ILASP3 is designed to handle noisy data. Weights can also be used to correct imbalances in the ratio between positive and negative examples. For example, if there are twice as many positive examples as negative, the negative examples may be weighted twice as much to avoid being ignored as noise by the system, i.e., the system would only focus on entailing as many positive examples as possible since the negative weights are negligible. However, like the original ILASP system, ILASP3 still pre-computes all ASP rules at each iteration leading to large computational cost, causing it to struggle when scaling to large datasets.

\section{Neural Networks}
An alternative to these previous ILP approaches is to take a continuous rather than discrete approach to the problem through the use of neural networks. The $\partial$ILP \cite{dilp} system uses continuous weight values to determine probability distributions of atoms over clauses. The system uses stochastic gradient descent to alter these weights and minimize the cross entropy loss of each classified example. Like with most standard neural network approaches, the system can be tuned with hyperparameters such as a learning rate and initial clause weights. In this way, the system can be trained to handle some amount of noise as the returned program may not have zero loss. As with the previous systems however, this hyperparameter tuning is not always intuitive. Additionally, $\partial$ILP requires program templates to constrain the set of programs searched which is another user defined parameter, requiring some amount of brute-force work in order to generate an efficient search space. 

\section{Applications to Popper}
While the noise handling approaches for these ILP systems are worth studying in their own rights, the unique LFF framework of Popper means that we cannot apply many of these techniques directly. The general concept of scoring hypotheses used in ILASP3 is a concept which can be applied to Popper as a means to compare more than just the accuracy of hypotheses in order to prevent overfitting, e.g., we may want to score a short and highly accurate hypothesis higher than a massive but perfectly accurate one. Ultimately however, a novel approach to noise handling must be taken with Popper, though we aim to show that the theoretical results used can be extended to other systems in the future regardless of whether they fall under the LFF framework. Additionally, many of these noise tolerant systems do so through the use of hyperparameters which often make them cumbersome to use and ineffective under default conditions, i.e., significant tuning is usually required to allow the systems to effectively generalize to noisy data. As such, a goal of the Noisy Popper implementation is to make it as natural of an extension of Normal Popper as possible which requires little to no hyperparameters.

\end{chapter}
\begin{chapter}
{Learning from Failures (LFF) Framework}
This chapter will provide an overview of the LFF framework used by Normal Popper and modified by Noisy Popper. First, we will briefly cover logic programming preliminaries and notation necessary for the rest of the paper, though we will assume some prior knowledge of boolean logic on the part of the reader. Using this notation, we will formally define the ILP problem setting. We will conclude by explaining the LFF framework and its definitions. 

\section{Logic Programming Preliminaries}
To understand the LFF framework, it is necessary to review the framework of logic programming. This section will briefly cover necessary definitions based on those found in \cite{cropper2017efficiently, cropper2020turning}. We will assume some familiarity with the topic, though for a comprehensive overview, interested readers are encouraged to reference \cite{nienhuys1997foundations}.

\subsection{First Order Logic}
We will refer to the following definitions from \cite{cropper2020turning} throughout the paper:

\begin{itemize}
    \item A \textit{variable} is a character string which starts with an uppercase letter, e.g., \texttt{A, B, Var}.
    \item A \textit{function} symbol is a character string which starts with a lowercase letter, e.g., \texttt{f, eastbound, last, head}.
    \item A \textit{predicate} symbol is a character string which starts with a lowercase, like a function symbol. The \textit{arity} \texttt{n} of a predicate symbol represents the number of arguments that it takes and is denoted as \texttt{p/n}, e.g., \texttt{f/1, eastbound/1, last/1, head/2}.
    \item A \textit{constant} symbol is a function or predicate symbol which has arity 0.
    \item A \textit{term} is a variable or constant symbol, or a function or predicate symbol with arity $n$ that is immediately followed by a tuple of $n$ terms.
    \item We call a term \textit{ground} if it contains no variables.
    \item An \textit{atom} is a logical formula $p(t_1, t_2, ..., t_n)$, where $p$ is a predicate symbol of arity $n$ and $t_i$ is a term for $i \in \{1, 2, ..., n\}$, e.g., \texttt{eastbound(train)} where \texttt{eastbound} is a predicate symbol of artiy 1 and \texttt{train} is a constant symbol.
    \item An atom is \textit{ground} is all of its terms are ground, like the example in the definition above.
    \item We represent the \textit{negation} symbol as $\lnot$.
    \item A \textit{literal} is an atom $A$ (a positive literal) or its negation $\lnot A$ (a negative literal), e.g., \texttt{eastbound(train$_1$)} is both an atom and a literal while $\lnot$\texttt{eastbound(train$_1$)} is only a literal as atoms do not contain the negation symbol.
\end{itemize}

\paragraph{Clauses}
We can use these previous definitions as building blocks to construct the logic programs and constraints we will be using.

\begin{definition}[Clause]
A clause if a finite (possibly empty) disjunction of literals. 
\label{def:clause}
\end{definition}

\bigskip \noindent For instance, this following set of literals constitutes a clause:

\begin{center}
    $\{$\texttt{eastbound(A), $\lnot$has$\_$car(A,B), $\lnot$two$\_$wheels(B), $\lnot$roof$\_$closed(B)}$\}$
\end{center}
 
\noindent We assume all variables in a clause are universally quantified, so explicit quantifiers are omitted. As with terms and atoms, clauses are ground if they contain no variables, so the example above is not ground. In logic programming, clauses are typically in reverse implication form:

\begin{center}
    \texttt{h} $\leftarrow$ \texttt{b$_1$ $\land$ b$_2$ $\land$ ... $\land$ b$_n$}.
\end{center}

\bigskip \noindent Put verbally, the above clause states that the literal \texttt{h}, known as the \textit{head literal}, is true only if all literals \texttt{b$_i$}, known as \textit{body literals}, are all true. All the \texttt{b$_i$} literals together are the body of the clause. Note, the head literal must always be a positive literal. We often use shorthand replacing $\leftarrow$ with \texttt{:-} and $\land$ with \texttt{,} to ease writing clauses and make them similar to actual Prolog notation, e.g., the above clause we would write as:

\begin{center}
    \texttt{h :- b$_1$, b$_2$, ..., b$_n$}.
\end{center}

\noindent For simplicity, we will use this Prolog notation throughout this paper. We define a \textit{clausal theory} as a set of clauses. In the LFF setting, we restrict clauses to those which contain no function symbols and where every variable which appears in the head of a clause also appears in its body. These clauses are known as \textit{Datalog} clauses and a set of Datalog clauses constitute a \textit{Datalog theory}. We also define a \textit{Horn clause} as a clause with at most one positive literal, as is the case with the above example. We restrict our setting to only Horn clauses and \textit{Horn theories} which are sets of Horn clauses. \textit{Definite} clauses are Horn clauses with exactly one positive literal while a \textit{Definite logic program} is a set of definite clauses. The logic programs which form our ILP hypothesis spaces will consist of only Datalog definite logic programs.

\paragraph{Substitution}
\textit{Substitution} is an essential logic programming concept and is simply the act of replacing variables $v_0, v_1, ..., v_n$ with terms $t_0, t_1, ..., t_n$. Such a substitution is denoted by: $\theta = \{v_0/t_0, v_1/t_1, ..., v_n/t_n\}$. For instance, the substitution $\theta = \{A/train\}$ to \texttt{eastbound(A) :- has$\_$car(A,B), two$\_$wheels(B), roof$\_$closed(B)} yields \texttt{eastbound(train) :- has$\_$car(train,B), two$\_$wheels(B), roof$\_$closed(B)}. In this example, \texttt{eastbound(train)} would be true if \texttt{train} possess some car \texttt{B} such that \texttt{B} has three wheels and its roof is opened. A substitution $\theta$ \textit{unifies} atoms $A$ and $B$
if $A\theta = B\theta$, i.e., using substitution $\theta$ on atoms $A$ and $B$ obtains equivalent results.

\section{LFF Problem Setting}
This section formally introduces definitions to the LFF framework problem. Most of the definitions are taken from \cite{popper}. Interested readers should refer to this paper for a more thorough explanation.

\subsection{Declaration Bias}
\bigskip \noindent The LFF problem setting is based off of the ILP learning from entailment setting \cite{quinlan} whose goal, as stated in the first chapter, is to take as input sets of positive and negative examples, BK, and a target predicate and return a hypothesis or logic program which in conjunction with the BK entails all positive examples and no negative examples. All ILP approaches boil down to searching a hypothesis space for such a program. For each ILP problem, the hypothesis space is restricted by a \textit{language bias}. Though several language biases exist in ILP, our LFF framework uses \textit{predicate declarations} which declare which predicates are permitted to appear in the head of a clause in a hypothesis and which are permitted to appear in the body. The declarations are defined as follows:

\begin{definition}[Head Declaration]
A \textit{head declaration} is a ground atom of the form $head\_pred(p,a)$ where $p$ is a predicate symbol of arity $a$ \cite{popper}.
\label{def:headdeclaration}
\end{definition}

\bigskip \noindent For example, for our running trains problem, we would have \texttt{head$\_$pred(eastbound,1)}.

\begin{definition}[Body Declaration]
A \textit{body declaration} is a ground atom of the form $body\_pred(p,a)$ where $p$ is a predicate symbol of arity $a$ \cite{popper}.
\label{def:bodydeclatation}
\end{definition}

\bigskip \noindent For example, for the trains example, we would have \texttt{body$\_$pred(has$\_$car,2), \\ body$\_$pred(two$\_$wheels,1)} and \texttt{body$\_$pred(roof$\_$closed,1)} among others.

\bigskip \noindent We can then define a declaration bias $D$ as a pair $(D_h, D_b)$ where $D_h$ is a set of head declatations and $D_b$ is a set of body declarations. The LFF hypothesis space then must only be comprised of programs whose clauses conform to these declaration biases. We define the notion of a \textit{declaration consistent} clause:

\begin{definition}[Declaration Consistent Clause]
Let $D = (D_h, D_b)$ be a declaration bias and $C = h \leftarrow b_1, b_2, ..., b_n$ be a definite clause. We say that $C$ is \textit{declaration consistent} with $D$ if and only if:
\label{def:declarationconsistentclause}
\end{definition}

\begin{itemize}
    \item $h$ is an atom of the form $p(X_1, X_2, ..., X_n)$ such that $head\_pred(p,n) \in D_h$.
    \item every $b_i$ is a literal of the form $p(X_1, X_2, ..., X_m)$ such that $body\_pred(p,m) \in D_b$.
    \item every $X_i$ is a first-order variable.
\end{itemize}
\cite{popper}

\begin{example}[Clause Declaration Consistency]
Let \\ $D = (\{$\texttt{head$\_$pred(eastbound,1)}$\}, \{$\texttt{body$\_$pred(has$\_$car,2), body$\_$pred(two$\_$wheels,1), body$\_$pred(roof$\_$closed,1)}$\})$ be a declaration bias. The following clauses would be declaration consistent with $D$:
\end{example}

\begin{center}
    \begin{tabular}{l}
        \texttt{eastbound(A) :- has$\_$car(A,B).} \\
        \texttt{eastbound(A) :- has$\_$car(A,B), two$\_$wheels(A).} \\
        \texttt{eastbound(A) :- has$\_$car(A,B), roof$\_$closed(B).}
    \end{tabular}
\end{center}

\noindent Conversely, the following clauses are declaration inconsistent with $D$:

\begin{center}
    \begin{tabular}{l}
        \texttt{eastbound(A, B) :- has$\_$car(A,B).} \\
        \texttt{eastbound(A) :- has$\_$car(A,B), eastbound(A).} \\
        \texttt{eastbound(A) :- has$\_$car(A,B), has$\_$load(B,C).}
    \end{tabular}
\end{center}

\noindent With this definition, we can fully define \textit{Declaration consistent hypotheses} which populate our hypothesis space:

\begin{definition}[Declaration Consistent Hypothesis]
Let $D = (D_h, D_b)$ be a declaration bias. A \textit{declaration consistent hypothesis} $H$ is a set of definite clauses where each clause $C \in H$ is declaration consistent with $D$ \cite{popper}. 
\label{def:declarationconsistenthypo}
\end{definition}

\begin{example}[Hypothesis Declaration Consistency]
Again, let $D$ be the same declaration bias as in the example above. Then the following hypotheses are declaration consistent:
\end{example}

\begin{center}
    \begin{tabular}{l}
        \texttt{h$_1$ =} $\left\{\begin{array}{l}
        \texttt{eastbound(A) :- has$\_$car(A,B), two$\_$wheels(B).}\\
        \end{array}\right\}$ \\
        \texttt{h$_2$ =} $\left\{\begin{array}{l}
        \texttt{eastbound(A) :- has$\_$car(A,B), two$\_$wheels(B).}\\
        \texttt{eastbound(A) :- has$\_$car(A,B), roof$\_$closed(B).}
        \end{array}\right\}$
    \end{tabular}
\end{center}

\subsection{Hypothesis Contraints}
\label{sec:hypothesisconstraints}
While declaration biases are how we restrict the initial hypothesis space, the LFF framework revolves around pruning the hypothesis space through hypothesis constraints which we define as in \cite{popper}. We first precisely define a constraint:

\begin{definition}[Constraint]
A \textit{constraint} is a Horn clause without a head, i.e., a denial.
We say that a constraint is violated if all of its body literals are true \cite{popper}.
\label{def:constraint}
\end{definition}

\bigskip \noindent We can proceed with a general definition of a hypothesis constraint:

\begin{definition}[Hypothesis Constraint]
Let $\mathcal{L}$ be a language that defines hypotheses,
i.e., a meta-language. Then a \textit{hypothesis constraint} is a constraint expressed in $\mathcal{L}$ \cite{popper}.
\label{def:hypothesisconstraint}
\end{definition}

\begin{example}[Hypothesis Constraints]
\label{exp:hypothesisconstraints}
In both Normal and Noisy Popper, the meta-language used to encode programs takes a form like this:

\begin{center} 
    \texttt{head$\_$literal(Clause,Pred,Arity,Vars)} \\
\end{center}

\noindent which denotes that the clause \texttt{Clause} possesses a head literal with predicate symbol \texttt{Pred} which has an arity of \texttt{Arity} and whose arguments are defined by \texttt{Vars} (note: \texttt{Vars} would be represented by a tuple of size equal to \texttt{Arity}). The following atom:

\begin{center} 
    \texttt{body$\_$literal(Clause,Pred,Arity,Vars)} \\
\end{center}
 
\noindent analogously defines a body literal appearing in \texttt{Clause}. We can then construct an example of a hypothesis constraint:

\begin{center} 
    \texttt{:- head$\_$literal(C,p,1,$\_$), body$\_$literal(C,p,1,$\_$)} \\
\end{center}

\noindent where '$\_$'s represent wildcards. This constraint simply states that clause \texttt{C} cannot contain a predicate symbol \texttt{p} which appears both in the head and the body of the clause, e.g., the clause \texttt{C = p(A) :- p$_1$(A,B), p(B).}
\end{example}

\bigskip \noindent Like with declaration consistent hypotheses, we can now define a hypothesis which is consistent with all hypothesis constraints:

\begin{definition}[Constrain Consistent Hypothesis]
Let $C$ be a set of hypothesis constraints written in a language $\mathcal{L}$ . A set of definite clauses $H$ is \textit{consistent} with $C$ if, when written in $\mathcal{L}$ , $H$ does not violate any constraint in $C$. \cite{popper}
\label{def:constraintconsistenthypo}
\end{definition}

\subsection{Problem Setting}
\bigskip \noindent Now that we have defined declaration bias and hypothesis constraints, we can fully define the LFF hypothesis space which takes a similar form to most ILP hypothesis spaces:

\begin{definition}[Hypothesis Space] 
Let $D$ be a declaration bias and $C$ be a set of hypothesis constraints. Then, the hypothesis space $\mathcal{H}_{D,C}$ is the set of all declaration and constraint consistent hypotheses. We refer to any element in $\mathcal{H}_{D,C}$ as a \textit{hypothesis} \cite{popper}.
\label{def:hypothesisspace}
\end{definition}

\bigskip \noindent We additionally can define the precise LFF problem:

\begin{definition}[LFF Problem Input]
Our problem input is a tuple $(B,D,C,E^+,E^-)$ where:
\begin{itemize}
    \item $B$ is a Horn program denoting background knowledge
    \item $D$ is a declaration bias
    \item $C$ is a set of hypothesis constraints
    \item $E^+$ is a set of ground atoms denoting positive examples
    \item $E^-$ is a set of ground atoms denoting negative examples
\end{itemize}
\cite{popper}
\label{def:probleminput}
\end{definition}

\bigskip \noindent As in \cite{popper} we will also define several hypothesis outcomes or types commonly used in ILP literature \cite{nienhuys1997foundations} which we will refer to moving forward.

\begin{definition}[Hypothesis Types]
Let $(B,D,C,E^+,E^-)$ be an input tuple and $H \in \mathcal{H}_{D,C}$ be a hypothesis. Then $H$ is:
\begin{itemize}
    \item \textit{Complete} when $\forall e \in E^+, H \cup B \models e$
    \item \textit{Consistent} when $\forall e \in E^-, H \cup B \not \models e$
    \item \textit{Incomplete} when $\exists e \in E^+, H \cup B \not \models e$
    \item \textit{Inconsistent} when $\exists e \in E^-, H \cup B \models e$
    \item \textit{Totally Incomplete} when $\forall e \in E^+, H \cup B \not \models e$
    \item \textit{Totally Inconsistent} when $\forall e \in E^-, H \cup B \models e$
\end{itemize}
\cite{popper}
\label{def:hypothesistypes}
\end{definition}

\bigskip \noindent This terminology also helps us define an LFF solution and LFF failed hypothesis:

\begin{definition}[LFF Solution] 
Given an input tuple $(B,D,C,E^+,E^-)$, a hypothesis $H \in \mathcal{H}_{D,C}$ is a \textit{solution} when $H$ is complete and consistent \cite{popper}.
\label{def:lffsolution}
\end{definition}

\begin{definition}[LFF Failed Hypothesis] 
Given an input tuple \\ $(B,D,C,E^+,E^-)$, a hypothesis $H \in \mathcal{H}_{D,C}$ \textit{fails} (or is a \textit{failed} hypothesis) when $H$ is either incomplete or inconsistent \cite{popper}.
\label{def:lfffailedhypothesis}
\end{definition}

\bigskip \noindent These definition correspond to many ILP system settings we discussed in Chapter's 1 and 2 where a solution entails all positive examples and no negative examples. 

\paragraph{Optimality}
For a given LFF problem, there can naturally be several solutions. For example, consider an east-west trains problem where all eastbound trains are those which possess a car with two wheels and all other trains are westbound. Consider the hypotheses:

\begin{center}
    \begin{tabular}{l}
        \texttt{h$_1$ =} $\left\{\begin{array}{l}
        \texttt{eastbound(A) :- has$\_$car(A,B), two$\_$wheels(B).}\\
        \end{array}\right\}$ \\
        \texttt{h$_2$ =} $\left\{\begin{array}{l}
        \texttt{eastbound(A) :- has$\_$car(A,B), two$\_$wheels(B).}\\
        \texttt{eastbound(A) :- has$\_$car(A,B), two$\_$wheels(B), roof$\_$closed(B).}
        \end{array}\right\}$
    \end{tabular}
\end{center}

\medskip \noindent Both hypotheses would correctly identify all trains. Note that the second clause in \texttt{h$_2$} will entail nothing extra from the first clause, making it redundant. Naturally, we would rather return hypothesis \texttt{h$_1$} as it is simpler, lacking this redudant clause. Though deciding between two solutions is a common and non-trivial problem in ILP, often systems define optimality in terms of length, returning the solution with fewest clauses \cite{muggleton2015meta, hexmil} or literals \cite{corapi2011inductive, ilasp}. While many ILP systems are not guaranteed to return optimal solutions \cite{muggleton1995inverse, srinivasan2001aleph, blockeel1998top}, Normal Popper \cite{popper} is guaranteed to return optimal solutions with minimal number of total literals. Noisy Popper also appeals to this description of optimality as it works closely with the minimum description length (MDL) principle \cite{rissanen1978modeling} which is used to justify several claims later in this paper. As such, we will formally define hypothesis size and solution optimality:

\begin{definition}[Hypothesis Size]
The function $size(H)$ returns the total number of literals in the hypothesis $H$ \cite{popper}.
\label{def:hypothesissize}
\end{definition}

\begin{definition}[LFF Optimal Solution] 
Given an input tuple $(B,D,C,E^+,E^-)$, a hypothesis $H \in \mathcal{H}_{D,C}$ is an \textit{optimal solution} when two conditions hold:
\begin{itemize}
    \item $H$ is a solution
    \item $\forall H' \in \mathcal{H}_{D,C}$ such that $H'$ is a solution, $size(H) \leq size(H')$
\end{itemize} \cite{popper}
\label{def:lffoptimalsolution}
\end{definition}

\subsection{Generalizations and Specializations}
In the LFF framework, the hypothesis constraints are learned from the generalizations and specializations of failed hypotheses. In this way, large sections of the hypothesis space can be pruned for each hypothesis generated and tested. To understand generalizations and specializations, we need to define the notion of $\theta$-subsumption \cite{plotkin1972automatic} which we refer to simply as subsumption.

\begin{definition}[Clausal Subsumption]
A clause $C_1$ \textit{subsumes} a clause $C_2$ if and only if there exists a $\theta$-subsumption such that $C_1 \theta \subseteq C_2$ \cite{popper}.
\label{def:clausalsubsumption}
\end{definition}

\begin{example}[Clausal Subsumption] 
Let \texttt{C$_1$} and \texttt{C$_2$} be defined as:
\begin{center}
    \begin{tabular}{l}
        \texttt{C$_1$ = eastbound(A) :- has$\_$car(A,B).}\\
        \texttt{C$_2$ = eastbound(X) :- has$\_$car(X,Y), two$\_$wheels(Y).}
    \end{tabular}
\end{center}
We say that \texttt{C$_1$} subsumes \texttt{C$_2$} since if $\theta = \{A/B\}$ then \texttt{C$_1$} $\theta \subseteq$ \texttt{C$_2$}.
\end{example}

\bigskip \noindent Importantly, subsumption implies entailment \cite{nienhuys1997foundations}, though the converse does not necessarily hold. Thus, if clause $C_1$ subsumes $C_2$, then $C_1$ must entail at least everything that $C_2$ does. \cite{midelfart1999bounded} extends this idea of subsumption to clausal theories:

\begin{definition}[Theory Subsumption] 
A clausal theory $T_1$ \textit{subsumes} a clausal theory $T_2$, denoted $T_1 \preceq T_2$, if and only if $\forall C_2 \in T_2, \exists C_1 \in T_1$ such that $C_1$ subsumes $C_2$ \cite{popper}.
\label{def:theorysubsumption}
\end{definition}

\begin{example}[Theory Subsumption] 
Let \texttt{h$_1$} and \texttt{h$_2$} be defined as:
\begin{center}
    \begin{tabular}{l}
        \texttt{h$_1$ =} $\left\{\begin{array}{l}
        \texttt{eastbound(A) :- has$\_$car(A,B), two$\_$wheels(B).}
        \end{array}\right\}$ \\
        \texttt{h$_2$ =} $\left\{\begin{array}{l} \\
        \texttt{eastbound(A) :- has$\_$car(A,B), two$\_$wheels(B), roof$\_$closed(B).}
        \end{array}\right\}$ \\
        \texttt{h$_3$ =} $\left\{\begin{array}{l}
        \texttt{eastbound(A) :- has$\_$car(A,B), two$\_$wheels(B).}\\
        \texttt{eastbound(A) :- has$\_$car(A,B), roof$\_$closed(B).}
        \end{array}\right\}$
    \end{tabular}
\end{center}

\noindent Then we can say \texttt{h$_1$} $\preceq$ \texttt{h$_2$}, \texttt{h$_3$} $\preceq$ \texttt{h$_2$}, and \texttt{h$_3$} $\preceq$ \texttt{h$_1$}.
\end{example}

\bigskip \noindent \cite{popper} also proves the following proposition regarding theory subsumption:

\begin{proposition}[(Subsumption implies Entailment)] Let $T_1$ and $T_2$ be clausal theories. If $T_1 \preceq T_2$ then $T_1 \models T_2$ \cite{popper}.
\label{prop:subsumptionimplies}
\end{proposition}

\bigskip \noindent That is, using the programs above, any example that \texttt{h$_1$} entails is also entailed by \texttt{h$_3$}. Using Definition~\ref{def:theorysubsumption} for theory subsumption, we can define the notion of a \textit{generalization}:

\begin{definition}[Generalization] A clausal theory $T_1$ is a \textit{generalization} of a clausal theory $T_2$ if and only if $T_1 \preceq T_2$ \cite{popper}.
\label{def:generalization}
\end{definition}

\bigskip \noindent For example, again using the programs above, \texttt{h$_3$} is a generalization of \texttt{h$_1$} which is a generalizations of \texttt{h$_2$}. Likewise, we can define the notion of a \textit{specialization}:

\begin{definition}[Specialization] A clausal theory $T_1$ is a \textit{specialization} of a clausal theory $T_2$ if and only if $T_2 \preceq T_1$ \cite{popper}.
\label{def:specialization}
\end{definition}

\bigskip \noindent Again using the previous programs as examples, \texttt{h$_2$} is a specialization of \texttt{h$_1$} which is a specialization of \texttt{h$_3$}.

\bigskip \noindent With these definitions, we can describe in the next section how Normal Popper generates hypothesis constraints using generalizations and specializations from failed hypotheses.

\section{Hypothesis Constraints}
The Normal Popper system breaks down the ILP problem into three separate stages: \textit{generate, test,} and \textit{constrain}. Unlike many ILP approaches which refine a clause \cite{quinlan, muggleton1995inverse, blockeel1998top, srinivasan2001aleph, nrsample} or hypothesis \cite{DBLP:conf/ilp/Bratko99, DBLP:conf/ilp/AthakraviCBR13, muggleton2015meta}, Normal Popper refines the hypothesis space itself by learning hypothesis constraints. In the generate stage, Normal Popper generates a hypothesis which satisfies all current hypothesis constraints. These constraints determine the syntactic form a hypothesis may take. In the subsequent test stage, this hypothesis is then tested against the positive and negative examples provided to the system. Should a hypothesis fail, i.e., it is either \textit{incomplete} or \textit{inconsistent}, the system continues on to the constrain stage. Here, the system learns additional hypothesis constraints from the failed hypothesis to further prune the hypothesis space for future hypothesis generation. There are two general types of constraints that both Normal Popper and Noisy Popper are concerned with: generalizations and specializations. We will discuss both here in addition to a third particular type of constrain called elimination constraints.

\subsection{Generalization Constraints} 
Consider a hypothesis $H$ being tested against $E^-$. If $H$ is inconsistent, that is it entails some or all of the examples in $E^-$, we can conclude that $H$ is too general. That is, $H$ is entailing too many examples and not being restrictive enough. Thus, any solution to the ILP problem is necessarily more restrictive than $H$, i.e., it entails less than $H$. We can prune all generalizations of $H$ as these too must be inconsistent \cite{popper} since they only can entail additional examples from $H$. This leads us to the definition of a generalization constraint:

\begin{definition}[Generalization Constraint] 
A generalization constraint only prunes generalizations of a hypothesis from the hypothesis space \cite{popper}.
\label{def:generalizationconstraint}
\end{definition}

\begin{example}[Generalization Constraints]
Suppose we have the following defined:

\begin{center}
    \begin{tabular}{l}
        $E^-$ = \{\texttt{eastbound(train$_1$).}\}\\
        $\texttt{h = \{eastbound :- has\_car(A,B), two\_wheels(B).\}}$
    \end{tabular}
\end{center}

\noindent Additionally, suppose the BK contains facts:

\begin{center}
    \begin{tabular}{l}
        \texttt{has$\_$car(train$_1$,car$_1$).}\\
        \texttt{two$\_$wheels(car$_1$).}
    \end{tabular}
\end{center}

\noindent We can see how \texttt{h} entails the only negative example, indicating that it is too general. As such, all generalizations of \texttt{h} can be pruned, e.g., programs such as:

\begin{center}
    \begin{tabular}{l}
        \texttt{h$_1$ =} $\left\{\begin{array}{l}
        \texttt{eastbound(A) :- has$\_$car(A,B), two$\_$wheels(B).}\\
        \texttt{eastbound(A) :- has$\_$car(A,B), roof$\_$closed(B).}
        \end{array}\right\}$ \\
        \texttt{h$_2$ =} $\left\{\begin{array}{l}
        \texttt{eastbound(A) :- has$\_$car(A,B), two$\_$wheels(B).}\\
        \texttt{eastbound(A) :- has$\_$car(A,B), has$\_$load(B,C), circle(C).}\\
        \texttt{eastbound(A) :- has$\_$car(A,B), short(B).}
        \end{array}\right\}$
    \end{tabular}
\end{center}

\noindent Because \texttt{h$_1$} $\succeq$ \texttt{h} and \texttt{h$_2$} $\succeq$ \texttt{h}, both \texttt{h$_1$} and \texttt{h$_2$} must also entail this one negative example and therefore cannot be LFF solutions. Note that given hypotheses \texttt{h} and \texttt{h'}, if \texttt{h} $\subseteq$ \texttt{h'} then \texttt{h} is a generalization of \texttt{h'}.
\end{example}

\subsection{Specialization Constraints} 
Next, consider a hypothesis $H$ being tested against $E^+$. If $H$ is incomplete, that is it entails only some or none of the examples in $E^+$, we can conclude that $H$ is too specific. That is, $H$ is entailing too few examples and being overly restrictive. Thus, any solution to the ILP problem is necessarily less restrictive than $H$, i.e., it entails more than $H$. We can prune all specializations of $H$ as these too must be incomplete \cite{popper} since they only can entail fewer examples than $H$. This leads us to the definition of a specialization constraint:

\begin{definition}[Specialization Constraint] 
A specialization constraint only prunes specialization of a hypothesis from the hypothesis space \cite{popper}.
\label{def:specializationconstraint}
\end{definition}

\begin{example}[Specialization Constraints]
Suppose we have the following defined:

\begin{center}
    \begin{tabular}{l}
        $E^+$ = \{\texttt{eastbound(train$_2$).}\}\\
        $\texttt{h = \{eastbound :- has\_car(A,B), two\_wheels(B), roof\_closed(B).\}}$
    \end{tabular}
\end{center}

\medskip \noindent Additionally, suppose the BK contains the facts:

\begin{center}
    \begin{tabular}{l}
        \texttt{has$\_$car(train$_2$,car$_2$).}\\
        \texttt{two$\_$wheels(car$_2$).}
    \end{tabular}
\end{center}

\noindent We can see how \texttt{h} does not entails the only positive example, since \texttt{train$_2$} only contains a car which has three wheels, but does not have its roof closed. This indicates that the hypothesis is too specific. As such, all specializations of \texttt{h} can be pruned, e.g., programs such as:

\begin{center}
    \begin{tabular}{l}
        \texttt{h$_1$ =} $\left\{\begin{array}{l}
        \texttt{eastbound(A) :- has$\_$car(A,B), two$\_$wheels(B), roof$\_$closed(B), short(B).}\\
        \end{array}\right\}$ \\
        \texttt{h$_2$ =} $\left\{\begin{array}{l}
        \texttt{eastbound(A) :- has$\_$car(A,B), two$\_$wheels(B), has$\_$load(B,C), circle(C).}\\
        \end{array}\right\}$
    \end{tabular}
\end{center}

\medskip \noindent Because \texttt{h} $\succeq$ \texttt{h$_1$} and \texttt{h} $\succeq$ \texttt{h$_2$}, both \texttt{h$_1$} and \texttt{h$_2$} must also fail to entail this positive example and therefore cannot be LFF solutions.
\end{example}

\subsection{Elimination Constraints}
\label{sec:eliminationconstraints}
Finally, we can consider a specific case where a hypothesis $H$ is \textit{totally incomplete}. In addition to the normal specialization constraint, we can prune a particular set of hypotheses which contain a version of $H$ within themselves. To precisely define these hypotheses, we will need an additional definition:

\begin{definition}[Separable]
A \textit{separable} hypothesis $G$ is one where no predicate symbol in the head of a clause in $G$ occurs in the body of a clause in $G$ \cite{popper}.
\label{def:separable}
\end{definition}

\begin{example}[Non-separable Hypotheses]
Consider the following hypothesis:

\begin{center}
    \begin{tabular}{l}
        \texttt{h =} $\left\{\begin{array}{l}
        \texttt{eastbound(A) :- has$\_$car(A,B), f(B).}\\
        \texttt{f(B) :- two$\_$wheels(B), roof$\_$closed(B)}
        \end{array}\right\}$
    \end{tabular}
\end{center}

\noindent Hypothesis \texttt{h} is non-separable because the predicate symbol \texttt{f} appears in both a head of a clause and in the body of a clause. \cite{popper} shows that if a hypothesis $H$ is totally incomplete, then neither $H$ nor any specialization of $H$ can appear inside any separable \textit{optimal} solution. Thus, all separable hypotheses containing a specialization of $H$ can be pruned. This leads us to the definition of an elimination constraint:
\end{example}

\begin{definition}[Elimination Constraint]
An elimination constraint only prunes separable hypotheses that contain specialisations of a hypothesis from the hypothesis space \cite{popper}.
\label{def:eliminationconstraint}
\end{definition}

\begin{example}[Elimination Constraints]
Consider the set of positive examples:

\begin{center}
    \begin{tabular}{l}
        $E^+ = $\{\texttt{eastbound(train$_1$)., eastbound(train$_2$).}\}
    \end{tabular}
\end{center}

\noindent and consider the candidate hypothesis \texttt{h}:

\begin{center}
    \begin{tabular}{l}
        \texttt{h =} $\left\{\begin{array}{l}
        \texttt{eastbound(A) :- has$\_$car(A,B), two$\_$wheels(B), roof$\_$closed(B).}
        \end{array}\right\}$
    \end{tabular}
\end{center}

\medskip \noindent Additionally, suppose the BK contains the facts:

\begin{center}
    \begin{tabular}{l}
        \texttt{has$\_$car(train$_1$,car$_1$).}\\
        \texttt{two$\_$wheels(car$_1$).}\\
        \texttt{has$\_$car(train$_2$,car$_2$).}\\
        \texttt{short(car$_2$).}\\
    \end{tabular}
\end{center}

\medskip \noindent Clearly, \texttt{h} is totally incomplete and as such, Popper will add an elimination constraint which will prune all separable hypotheses that contain \texttt{h$_1$} or any of its specializations such as:

\bigskip
{\centering
    \begin{tabular}{l}
        \texttt{h$_1$ =} $\left\{\begin{array}{l}
        \texttt{eastbound(A) :- has$\_$car(A,B), two$\_$wheels(B), roof$\_$closed(B).}\\
        \texttt{eastbound(A) :- has$\_$car(A,B), has$\_$load(B,C), circle(C).}
        \end{array}\right\}$\\
        \texttt{h$_2$ =} $\left\{\begin{array}{l}
        \texttt{eastbound(A) :- has$\_$car(A,B), two$\_$wheels(B), roof$\_$closed(B).}\\
        \texttt{eastbound(A) :- has$\_$car(A,B), short(B).}\\
        \texttt{eastbound(A) :- has$\_$car(A,B), two$\_$wheels(B).}
        \end{array}\right\}$\\
        \texttt{h$_3$ =} $\left\{\begin{array}{l}
        \texttt{eastbound(A) :- has$\_$car(A,B), two$\_$wheels(B), roof$\_$closed(B), short(B).}\\
        \texttt{eastbound(A) :- has$\_$car(A,B), long(B).}
        \end{array}\right\}$
    \end{tabular}
}
\end{example}

\medskip \noindent Note that elimination constraints may prune solutions from the hypothesis space. If $E^-$ is empty in the example above, the hypothesis \texttt{h$_2$} above would be a solution to the problem as it entails all positive examples. However, this hypothesis is not optimal as elimination constraints will never prune optimal solutions. An optimal solution to this example for instance would instead be:

\begin{center}
    \begin{tabular}{l}
        \texttt{h$_4$ =} $\left\{\begin{array}{l}
        \texttt{eastbound(A) :- has$\_$car(A,B), short(B).}\\
        \texttt{eastbound(A) :- has$\_$car(A,B), two$\_$wheels(B).}
        \end{array}\right\}$\\
    \end{tabular}
\end{center}

\bigskip \noindent These basic hypothesis constraints allow Normal Popper to perform exceptionally well on many datasets, even those with very few examples. However, these constraints heavily rely on the absence of noise in the example sets. As the next chapter will discuss, incorrectly labelled examples can cause these hypothesis constraints to prune valuable sections of the hypothesis space.

\section{Summary}
In this chapter, we summarized the LFF framework for ILP problem solving. In doing so, we reviewed necessary logic programming concepts and notations as well as formalized terminology we well use frequently moving forward. We additionally discussed the crucial concepts of subsumption, generalizations, and specialization which both Noisy and Normal Popper base their hypothesis constraints on. Finally, we discussed the constraints Normal Popper implements: generalization constraints, specialization constraints, and elimination constraints giving examples of each. In the next chapter, we discuss the modified problem setting for Noisy Popper which we refer to as the Relaxed LFF Framework.

\end{chapter}
\begin{chapter}
{Relaxed LFF Framework}

The LFF setting described in Chapter 3 has a limitation when handling noise in that solutions must perfectly fit the given examples. Its hypothesis constraints may remove highly accurate hypotheses which do not overfit any noisy data in favor of an LFF solution which do overfit. One approach to avoid this is to ignore or \textit{relax} all hypothesis constraints and take a brute force approach, enumerating all hypothesis until one of adequate accuracy is found. This has an obvious inefficiency limitation. The aim of this project if to find a middle-ground between the two solutions which better generalizes to noisy data. 

\bigskip \noindent From this chapter, we will focus on presenting the novel contributions of the project. Here, we first outline the altered Relaxed LFF Framework from which Noisy Popper is built. We start by defining the altered noisy problem setting. We then describe and prove the sound hypothesis constraints within this new setting. Finally, we describe how we can apply the MDL principle to prune overfitting hypotheses through additional sound constraints which take into account hypothesis size.

\section{Relaxed LFF Problem Setting}
In contrast to the LFF problem setting, in the general relaxed setting we do not necessarily wish to find hypotheses that entail all positive examples and no negative examples. Rather, we wish to find hypotheses which optimize some other metric or \textit{score}. In this manner, we can define the relaxed LFF problem input:

\begin{definition}[Relaxed LFF Problem Input]
Our problem input is a tuple $(B,D,C,E^+,E^-,S)$ where:
\begin{itemize}
    \item $B$ is a Horn program denoting background knowledge
    \item $D$ is a declaration bias
    \item $C$ is a set of hypothesis constraints
    \item $E^+$ is a set of ground atoms denoting positive examples
    \item $E^-$ is a set of ground atoms denoting negative examples
    \item $S$ is a scoring function which takes as input $B, E^+, E^-$ as well as a hypothesis $H \in \mathcal{H}_{D,C}$
\end{itemize}
\label{def:relaxedprobleminput}
\end{definition}

\noindent Note that the hypothesis space in this relaxed setting is unchanged from the LFF setting. From here, we can define an solution in this new setting:

\begin{definition}[Relaxed LFF Solution]
Given an input tuple $(B,D,C,E^+,E^-,S)$, a hypothesis $H \in \mathcal{H}_{D,C}$ is a \textit{solution} when $\forall H' \in \mathcal{H}_{D,C}$, $S(H,B,E^+,E^-) \geq S(H',B,E^+,E^-)$.
\end{definition}

\noindent Note that it is possible to model the LFF setting from Chapter 3 in with this definition. To do so, the scoring function $S$ would be defined as:

{\centering
    \[S(H,B,E^+,E^-) = \left\{\begin{array}{l}
    1 \textrm{ if } H \textrm{ is complete and consistent}\\
    0 \textrm{ otherwise}
\end{array}\right\}\]
}

\noindent As in the LFF setting, we define optimality similarly using the size of a hypothesis:

\begin{definition}[Relaxed LFF Optimal Solution] 
Given an input tuple \\ $(B,D,C,E^+,E^-,S)$, a hypothesis $H \in \mathcal{H}_{D,C}$ is an \textit{optimal solution} when two conditions hold:
\begin{itemize}
    \item $H$ is a solution in the relaxed LFF setting
    \item $\forall H' \in \mathcal{H}_{D,C}$ such that $H'$ is a solution, $size(H) \leq size(H')$
\end{itemize}
\label{def:relaxedlffoptimalsolution}
\end{definition}

\noindent With these definition, we can lay out the theoretical contributions of this project through new hypothesis constraints which remain sound in this new setting.

\section{Relaxed LFF Hypothesis Constraints}
\noindent The main difficulty Normal Popper has when dealing with noise is its strict constraints which prune the hypothesis space. If a hypothesis does not entail even just a single positive hypothesis, it is rejected and all of its specializations are pruned. Similarly if a hypothesis entails just one negative example, all of its generalizations are pruned. While this works extremely well under the normal LFF setting, in the presence of noise, being so strict may prune relaxed LFF solutions which do not fit the noisy data and only the underlying patterns. We can illustrate this type of overpruning through an example:

\begin{example}[Overpruning] 
Consider Normal Popper trying to learn the program:
\[\texttt{h =} \left\{\begin{array}{c}
    \texttt{eastbound(A) :- has$\_$car(A,B), short(B), two$\_$wheels(B).}
\end{array}\right\}\]

\noindent Assume that all examples are correctly labelled except one noisy example: \\ \texttt{eastbound(train$_1$)} $\in E^+$ where \texttt{train$_1$} only possess a single long car. That is, in the BK we have among others the facts:

\begin{center}
    \begin{tabular}{l}
        \texttt{has$\_$car(train$_1$,car$_1$).}\\
        \texttt{long(car$_1$).}\\
    \end{tabular}
\end{center}

\noindent Suppose we generate the hypothesis:
\[\texttt{h$_1$ =} \left\{\begin{array}{c}
    \texttt{eastbound(A) :- has$\_$car(A,B), short(B).}
\end{array}\right\}\]

\noindent This program will entail all positive examples in $E^+$ except for the single noisy example as \texttt{train$_1$} does not possess a short car. As such, in the LFF setting, \texttt{h$_1$} is categorized as too specific. Thus, all specializations of \texttt{h$_1$} are pruned which includes the desired solution \texttt{h}.
\end{example}

\bigskip \noindent This overly strict pruning clearly can lead Normal Popper to overfitting any noisy dataset as even a single incorrectly labelled example can cause heavy pruning of the hypothesis space, potentially eliminating relaxed LFF solutions. This can be avoided by not applying any LFF hypothesis constraints. However, we still wish to improve the efficiency of the hypothesis search by removing hypotheses which can not be relaxed LFF solutions, thus motivating relaxed LFF hypothesis constraints. Any hypothesis constraints used in this setting should be sound under some scoring function. Here we define the accuracy scoring function used in this section. We first define the notions of true positive, true negative, false positive, and false negative.

\begin{definition}[True Positive]
Given an input tuple $(B,D,C,E^+,E^-,S)$ and a hypothesis $H \in \mathcal{H}_{D,C}$, we define the true positive function as \\ $tp(H, B, E^+) = \big|\{e\ |\ e \in E^+$ and $H \cup B \models e\}\big|$. 
\label{def:truepositivescore}
\end{definition}

\begin{definition}[True Negative]
Given an input tuple $(B,D,C,E^+,E^-,S)$ and a hypothesis $H \in \mathcal{H}_{D,C}$, we define the true negative function as \\ $tn(H, B, E^-) = \big|\{e\ |\ e \in E^-$ and $H \cup B \not \models e\}\big|$.
\label{def:truenegativescore}
\end{definition}

\begin{definition}[False Positive]
Given an input tuple $(B,D,C,E^+,E^-,S)$ and a hypothesis $H \in \mathcal{H}_{D,C}$, we define the false positive function as \\ $fp(H, B, E^-) = \big|\{e\ |\ e \in E^-$ and $H \cup B \models e\}\big|$. 
\label{def:falsepositivescore}
\end{definition}

\noindent Equivalently, we may choose to write $fp(H,B,E^-) = |E^-| - tn(H,B,E^-)$

\begin{definition}[False Negative]
Given an input tuple $(B,D,C,E^+,E^-,S)$ and a hypothesis $H \in \mathcal{H}_{D,C}$, we define the false negative function as \\ $fn(H, B, E^+) = \big|\{e\ |\ e \in E^+$ and $H \cup B \not \models e\}\big|$.
\label{def:falsenegativescore}
\end{definition}

\noindent Equivalently, we may choose to write $fn(H,B,E^+) = |E^+| - tp(H,B,E^+)$

\bigskip \noindent We will use these functions throughout the remained of the paper. With these, we can define the method with which hypotheses are scored in this section:

\begin{definition}[Accuracy Score]
Given an input tuple $(B,D,C,E^+,E^-,S_{ACC})$ and a hypothesis $H \in \mathcal{H}_{D,C}$, the function $S_{ACC}(H, E^+, E^-, B) = tp(H,B,E^+) + tn(H,B,E^-)$.
\label{def:accuracyscore}
\end{definition}

\bigskip \noindent Note that this scoring function measures training accuracy rather than test accuracy. We now aim to determine situations in which certain hypotheses are known to be suboptimal under this scoring method in the relaxed LFF setting. First, we will consider constraints constructed by comparing two hypotheses. We motivate this through an example:

\begin{example}[Learning by Comparing Hypotheses]
Consider we have relaxed LFF input $(B,D,C,E^+,E^-,S_{ACC})$ and  previously observed the following hypothesis:
\[\texttt{h$_1$ =} \left\{\begin{array}{c}
    \texttt{eastbound(A) :- has$\_$car(A,B), short(B).}
\end{array}\right\}\]

\noindent which has $tp($\texttt{h$_1$}$,B,E^+$) = 5 and $tn($\texttt{h$_1$}$,B,E^-$) = 3. Now, consider a generalization of this hypothesis:

\[\texttt{h$_2$ =} \left\{\begin{array}{l}
    \texttt{eastbound(A) :- has$\_$car(A,B), short(B).} \\
    \texttt{eastbound(A) :- has$\_$car(A,B), two$\_$wheels(B).}
\end{array}\right\}\]

\bigskip \noindent which identically has $tp($\texttt{h$_2$}$,B,E^+$) = 5 and $tn($\texttt{h$_2$}$,B,E^-$) = 3. We can conclude that the clause \texttt{eastbound(A) :- has$\_$car(A,B), two$\_$wheels(B).} is redundant as it entails no additional positive examples than the clause \texttt{eastbound(A) :- has$\_$car(A,B), short(B).} As such, it is not worthwhile to consider any non-recursive generalizations of \texttt{h$_2$} which add additional clauses to \texttt{h$_2$} as we could simply add these same clauses to \texttt{h$_1$}, producing a hypothesis that scores the same but is smaller as it does not contain the redundant clause. To illustrate, consider this generalization: of \texttt{h$_2$}:

\[\texttt{h$_2'$ =} \left\{\begin{array}{l}
    \texttt{eastbound(A) :- has$\_$car(A,B), short(B).} \\
    \texttt{eastbound(A) :- has$\_$car(A,B), two$\_$wheels(B).} \\
    \texttt{eastbound(A) :- has$\_$car(A,B), three$\_$wheels(B).}
\end{array}\right\}\]

\medskip \noindent which we assume scores $tp($\texttt{h$_2'$}$,B,E^+$) = 8 and $tn($\texttt{h$_2'$}$,B,E^-$) = 4. Now, consider this generalization of \texttt{h$_1$}:
\[\texttt{h$_1'$ =} \left\{\begin{array}{l}
    \texttt{eastbound(A) :- has$\_$car(A,B), short(B).} \\
    \texttt{eastbound(A) :- has$\_$car(A,B), three$\_$wheels(B).}
\end{array}\right\}\]

\medskip \noindent which must also score $tp($\texttt{h$_1'$}$,B,E^+$) = 8 and $tn($\texttt{h$_1'$}$,B,E^-$) = 4. Since \texttt{h$_1'$} and \texttt{h$_2'$} score the same, it is redundant to consider both of them and since $size($\texttt{h$_1'$}$) < size($\texttt{h$_2'$}$)$, we can safely prune \texttt{h$_2'$} as a less optimal hypothesis than \texttt{h$_1'$}. Thus, all generalizations of this form can be safely pruned from the hypothesis space.
\end{example}

\bigskip \noindent This example illustrates how we can compare previously observed hypotheses to any new hypotheses in order to identify hypotheses which are suboptimal under the $S_{ACC}$ scoring. The remainder of this section will outline and prove such suboptimal circumstances.

\bigskip \noindent We start by defining some useful propositions relating generalizations and specializations to scoring:

\begin{proposition}[(Scores of Generalizations)] 
Given problem input $(B,D,C,E^+,E^-,S)$, let $H$ and $H'$ be hypotheses in $\mathcal{H}_{D,C}$ where $H'$ is a generalization of $H$. Then:
\begin{itemize}
    \item $tp(H',B,E^+) \geq tp(H,B,E^+)$ and 
    \item $tn(H',B,E^-) \leq tn(H,B,E^-)$
\end{itemize}
\label{prop:scoresofgeneralizations}
\end{proposition}

\begin{proof}
Follows immediately from Proposition~\ref{prop:subsumptionimplies} as subsumption implies entailment.
\end{proof}

\begin{proposition}[(Scores of Specializations)] 
Given problem input $(B,D,C,E^+,E^-,S)$, let $H$ and $H'$ be hypotheses in $\mathcal{H}_{D,C}$ where $H'$ is a specialization of $H$. Then:
\begin{itemize}
    \item $tp(H',B,E^+) \leq tp(H,B,E^+)$ and 
    \item $tn(H',B,E^-) \geq tn(H,B,E^-)$
\end{itemize}
\label{prop:scoresofspecializations}
\end{proposition}

\begin{proof}
Follows immediately from Proposition~\ref{prop:subsumptionimplies} as subsumption implies entailment.
\end{proof}

\bigskip \noindent We now prove when relaxed generalization and specialization hypothesis constraints may be applied:

\begin{proposition} 
Given problem input $(B,D,C,E^+,E^-,S_{ACC})$ and hypotheses $H_1$, $H_2$, and $H_3$ in $\mathcal{H}_{D,C}$ where $H_3$ is a generalization of $H_1$, if \\$S_{ACC}(H_2,B,E^+,E^-) - S_{ACC}(H_1,B,E^+,E^-) > fn(H_1,B,E^+)$ then \\$S_{ACC}(H_2,B,E^+,E^-) > S_{ACC}(H_3,B,E^+,E^-)$.
\label{prop:sound1}
\end{proposition}

\begin{proof}
The improvement in score from $H_1$ to $H_2$ can be quantified by:
\begin{align*}
    S_{ACC}(H_2,B,E^+,E^-) - S_{ACC}(H_1,B,E^+,E^-) &> fn(H_1,B,E^+)\\
    S_{ACC}(H_2,B,E^+,E^-) - [tp(H_1,B,E^+) + tn(H_1,B,E^-)] &> |E^+| - tp(H_1,B,E^+)\\
    S_{ACC}(H_2,B,E^+,E^-) &> |E^+| + tn(H_1,B,E^-) \\
    &\geq tp(H_3,B,E^+) + tn(H_3,B,E^-) \\
    &= S_{ACC}(H_3,B,E^+,E^-)
\end{align*}
The second to last line following from Proposition~\ref{prop:scoresofgeneralizations} and that $H_3$ is a generalization of $H_1$ in addition to the fact $|E^+| \geq tp(H_3,B,E^+)$.
\end{proof}

\begin{proposition} 
Given problem input $(B,D,C,E^+,E^-,S_{ACC})$ and hypotheses $H_1$, $H_2$, and $H_3$ in $\mathcal{H}_{D,C}$ where $H_3$ in is a specialization of $H_1$. If \\$S_{ACC}(H_2,B,E^+,E^-) - S_{ACC}(H_1,B,E^+,E^-) > fp(H_1,B,E^-)$ then \\$S_{ACC}(H_2,B,E^+,E^-) > S_{ACC}(H_3,B,E^+,E^-)$.
\label{prop:sound2}
\end{proposition}

\begin{proof}
The improvement in score from $H_1$ to $H_2$ can be quantified by:
\begin{align*}
    S_{ACC}(H_2,B,E^+,E^-) - S_{ACC}(H_1,B,E^+,E^-) &> fp(H_1,B,E^-)\\
    S_{ACC}(H_2,B,E^+,E^-) - [tp(H_1,B,E^+) + tn(H_1,B,E^-)] &> |E^-| - tp(H_1,B,E^+)\\
    S_{ACC}(H_2,B,E^+,E^-) &> tp(H_1,B,E^+) + |E^-| \\
    &\geq tp(H_3,B,E^+) + tn(H_3,B,E^-) \\
    &= S_{ACC}(H_3,B,E^+,E^-)
\end{align*}
The second to last line following from Proposition~\ref{prop:scoresofspecializations} and that $H_3$ is a specialization of $H_1$ in addition to the fact $|E^-| \geq tn(H_3,B,E^-)$.
\end{proof}

\begin{proposition}
Given problem input  $(B,D,C,E^+,E^-,S_{ACC})$ and hypotheses $H_1$ and $H_2$ in $\mathcal{H}_{D,C}$ where $H_2$ is a generalization of $H_1$, if $tp(H_1,B,E^+) = tp(H_2,B,E^+)$, then given any non-recursive hypothesis $H_2' = H_2 \cup C$ for some non-empty set of clauses $C$, there exists a generalization of $H_1$, say $H_1'$, such that \\ $S_{ACC}(H_1',B,E^+,E^-) \geq S_{ACC}(H_2',B,E^+,E^-)$.
\label{prop:sound3}
\end{proposition}

\begin{proof}
Let $H_1' = H_1 \cup C$. Also let $n = tn(H_1,B,E^-) - tn(H_2,B,E^-)$, \\ i.e., $S_{ACC}(H_2,B,E^+,E^-) = S_{ACC}(H_1,B,E^+,E^-) - n$ (note, $n \geq 0$ by by Proposition~\ref{prop:scoresofgeneralizations}). Assume that $C$ entails $p$ additional positive examples from $H_2$ and $n'$ more negative examples. That is $p = tp(H_2',B,E^+) - tp(H_2,B,E^+)$ and \\ $n' = tn(H_2,B,E^-) - tn(H_2',B,E^-)$ (again noting that $n' \geq 0$ and $p \geq 0$ by Proposition~\ref{prop:scoresofgeneralizations}). Since $H_1'$ contains $C$, it must also entail $p$ additional positive examples from $H_1$ and at most $n + n'$ additional negative examples as it may entail the negative examples that $H_2$ entailed but $H_1$ did not. Thus, we have \\ $S_{ACC}(H_2',B,E^+,E^-) = S_{ACC}(H_1,B,E^+,E^-) + p - n - n'$ and \\ $S_{ACC}(H_1',B,E^+,E^-) \geq S_{ACC}(H_1,B,E^+,E^-) + p - n - n'$ \\ (note that $S_{ACC}(H_1',B,E^+,E^-)$ has an upper bound of $S_{ACC}(H_1,B,E^+,E^-) + p - n$). Thus, $S_{ACC}(H_1',B,E^+,E^-) \geq S_{ACC}(H_2',B,E^+,E^-)$.
\end{proof}

\begin{proposition}
Given problem input $(B,D,C,E^+,E^-,S_{ACC})$ and hypotheses $H_1$ and $H_2$ where $H_1 \subseteq H_2$, if $tp(H_1,B,E^+) = tp(H_2,B,E^+)$, given any non-recursive specialization of $H_2$, $H_2'$, there exists a specialization of $H_1$, say $H_1'$, such that $S_{ACC}(H_1',B,E^+,E^-) \geq S_{ACC}(H_2',B,E^+,E^-)$.
\label{prop:sound5}
\end{proposition}

\begin{proof}
We can write $H_2 = H_1 \cup C$ for some set of clauses $C$. Since $tp(H_1,B,E^+) = tp(H_2,B,E^+)$, the clauses in $C$ entail no additional positive examples from $H_1$, making them redundant to the clauses of $H_1$. That is, every positive example entailed by the clauses in $C$ is already entailed by the clauses in $H_1$. We can construct $H_1'$ as such: if a clause $c$ is in both $H_1$ and $H_2'$, it is also in $H_1'$. If a clause $c$ is in $H_1$ but not in $H_2'$ and instead has been replaced by a specified version of the clause, call it $c'$, then $c'$ is also in $H_1'$. Any clauses in $C$ or specified versions of these clauses are not in $H_1'$. In this way, $H_1'$ will make the same specifications as $H_2$ did to create $H_2'$ on the clauses in $H_1$. Since $H_1'$ makes the same specifications to clauses in $H_1$ as $H_2'$ does, and since $C$ entails no extra positive examples nor will any of its specific clauses in $H_2'$, then $tp(H_1',B,E^+) = tp(H_2',B,E^+)$. Additionally, by Proposition~\ref{prop:scoresofgeneralizations}, $tn(H_1,B,E^-) \geq tn(H_2,B,E^-)$. Since $H_1'$ makes the same specifications as $H_2'$, then $tn(H_1',B,E^-) \geq tn(H_2',B,E^-)$ noting that at best any specifications $H_2'$ makes to clauses in $C$ can only mean $C$ entails no negative examples in $H_2'$. Thus, $S_{ACC}(H_1',B,E^+,E^-) \geq S_{ACC}(H_2',B,E^+,E^-)$.
\end{proof}

\begin{proposition}
Given problem input $(B,D,C,E^+,E^-,S_{ACC})$ and hypotheses $H_1$ and $H_2$ in $\mathcal{H}_{D,C}$ where $H_2$ is a specialization of $H_1$, if $tn(H_1,B,E^-) = tn(H_2,B,E^-)$, then given any non-recursive hypothesis $H_2' = H_2 \cup C$ for some set of clauses $C$ there exists a generalization of $H_1$, say $H_1'$, such that $S_{ACC}(H_1',B,E^+,E^-) \geq S_{ACC}(H_2',B,E^+,E^-)$.
\label{prop:sound4}
\end{proposition}

\begin{proof}
Let $H_1' = H_1 \cup C$. Then, since $tn(H_1,B,E^-) = tn(H_2,B,E^-)$, we know that $tn(H_1',B,E^-) = tn(H_2',B,E^-)$. Since $H_2$ is a specialization of $H_1$, by Proposition~\ref{prop:scoresofspecializations}, $tp(H_1,B,E^+) \geq tp(H_2,B,E^+)$ which means that $tp(H_1',B,E^+) \geq tp(H_2',B,E^+)$. Thus, $S_{ACC}(H_1',B,E^+,E^-) \geq S_{ACC}(H_2',B,E^+,E^-)$.
\end{proof}

\bigskip \noindent Propositions~\ref{prop:sound3}-\ref{prop:sound4} notably only demonstrate $S_{ACC}$ suboptimality for \emph{non-recursive} generalizations and specializations. Though it may appear that additional clauses or literals are redundant and do not improve either the $tp$ or $tn$ scores, they may set up base cases for an $S_{ACC}$-optimal recursive program. We demonstrate this through an example:

\begin{example}[Non-recursive Case Motivation]
Consider trying find a program for the target predicate \texttt{alleven/1} which when given a list of integers returns \texttt{True} if all of them are even and \texttt{False} otherwise. For instance, \texttt{alleven([2,4,10,8]).} evaluates to \texttt{True} and \texttt{alleven([1,2,3]).} evaluates to \texttt{False}. 

\bigskip \noindent Assume that all examples are noiseless and suppose that the BK $B$ contains the following:

\begin{center}
    \begin{tabular}{l}
        \texttt{head([H|$\_$],H).}, i.e., returns \texttt{True} only if \texttt{H} is the first element of the given list\\
        \texttt{tail([$\_$|T],T).}, i.e., returns \texttt{True} only if \texttt{T} is the given list with the first element removed\\
        \texttt{empty([]).}, i.e. returns \texttt{True} only if the given list is empty \\
        \texttt{zero(0).}, i.e., returns \texttt{True} only if the given integer is 0\\
        \texttt{even(A) :- 0 is A mod 2}, i.e. returns \texttt{True} only if the given integer is even
    \end{tabular}
\end{center}

\noindent Suppose we have previously seen hypothesis \texttt{h$_1$ = \{alleven(A) :- head(A,B), even(B).\}} which entailed all positive examples and some negative examples. Now, consider a second hypothesis \texttt{h$_2$}:

\begin{center}
    \begin{tabular}{l}
        \texttt{h$_2$ =} $\left\{\begin{array}{l}
        \texttt{alleven(A) :- head(A,B), even(B).}\\
        \texttt{alleven(A) :- empty(A).}
        \end{array}\right\}$
    \end{tabular}
\end{center}

\noindent which likewise entails all positive examples and the same number of negative examples. By Proposition~\ref{prop:sound5}, since \texttt{h$_1$} $\subseteq$ \texttt{h$_2$} and $tp($\texttt{h$_1$},$B,E^+) = tp($\texttt{h$_2$},$B,E^+)$, all non-recursive specializations of \texttt{h$_2$} are not $S_{ACC}$-optimal. We must specify \emph{non-recursive} as the recursive specialization  \texttt{h$_3$} where:

\begin{center}
    \begin{tabular}{l}
        \texttt{h$_3$ =} $\left\{\begin{array}{l}
        \texttt{alleven(A) :- head(A,B), even(B), tail(A,C), alleven(C).}\\
        \texttt{alleven(A) :- empty(A).}
        \end{array}\right\}$ \\
    \end{tabular}
\end{center}

\noindent is a solution and would entail all positive examples and no negative examples, thus being $S_{ACC}$-optimal. Similar results hold for Propositions~\ref{prop:sound3} and \ref{prop:sound4} where pruning generalizations may remove $S_{ACC}$-optimal recursive programs.
\end{example}

\bigskip \noindent Additionally, we can identify some sound constraints that do not rely on comparing hypotheses:

\begin{proposition}
Given problem input $(B,D,C,E^+,E^-,S_{ACC})$ and hypothesis $H \in \mathcal{H}_{D,C}$ where $tp(H,B,E^+) = |E^+|$, if hypothesis $H' \in \mathcal{H}_{D,C}$ is a generalization of $H$, then $S_{ACC}(H,B,E^+,E^-) \geq S_{ACC}(H',B,E^+,E^-)$.
\label{prop:sound6}
\end{proposition}

\begin{proof}
Since $H'$ is a generalization of $H$, by Proposition~\ref{prop:scoresofgeneralizations},\\  $tp(H,B,E^+) \leq tp(H',B,E^+) \Rightarrow  tp(H',B,E^+) = |E^+|$ and \\ $tn(H,B,E^-) \geq tn(H',B,E^-)$. Thus, $S_{ACC}(H,B,E^+,E^-) \geq S_{ACC}(H',B,E^+,E^-)$.
\end{proof}

\begin{proposition}
Given problem input $(B,D,C,E^+,E^-,S_{ACC})$ and hypothesis $H \in \mathcal{H}_{D,C}$ where $tn(H,B,E^-) = |E^-|$, if hypothesis $H' \in \mathcal{H}_{D,C}$ is a specialization of $H$, then $S_{ACC}(H,B,E^+,E^-) \geq S_{ACC}(H',B,E^+,E^-)$.
\label{prop:sound7}
\end{proposition}

\begin{proof}
Since $H'$ is a specialization of $H$, by Proposition~\ref{prop:scoresofspecializations},\\ $tn(H,B,E^-) \leq tn(H',B,E^-) \Rightarrow tn(H',B,E^-) = |E^-|$ and \\ $tp(H,B,E^+) \geq tp(H',B,E^+)$. Thus, $S_{ACC}(H,B,E^+,E^-) \geq S_{ACC}(H',B,E^+,E^-)$.
\end{proof}

\subsection{Hypothesis Constraints Applications}
These propositions all apply in any ILP setting, however, in our relaxed LFF setting, they determine sets of programs which should be pruned should certain conditions hold:

\begin{itemize}
    \item Proposition~\ref{prop:sound1} implies that if a hypothesis $H$'s accuracy score is at least $fn(H',B,E^+)$ greater than that of some hypothesis $H'$, we may prune all generalizations of $H'$ as they cannot be $S_{ACC}$-optimal.
    \item Proposition~\ref{prop:sound2} implies that if a hypothesis $H$'s accuracy score is at least $fp(H',B,E^-)$ greater than that of some hypothesis $H'$, we may prune all specializations of $H'$ as they cannot be $S_{ACC}$-optimal.
    \item Proposition~\ref{prop:sound3} implies that if a hypothesis $H$ is a generalization of a hypothesis $H'$ and both have equal $tp$ values, we may prune all non-recursive superset of $H$ as they cannot be $S_{ACC}$-optimal. 
    \item Proposition~\ref{prop:sound5} implies that if a hypothesis $H$ is a superset of a hypothesis $H'$ and both have equal $tp$ values, we may prune all non-recursice specializations of $H$ as they cannot be $S_{ACC}$-optimal.
    \item Proposition~\ref{prop:sound4} implies that if a hypothesis $H$ is a specialization of a hypothesis $H'$ and both have equal $tn$ values, we may prune all non-recursive generalizations of $H$ as they cannot be $S_{ACC}$-optimal.
    \item Proposition~\ref{prop:sound6} implies that if a hypothesis $H$ entails all positive examples, we may prune all larger generalizations of $H$ as they cannot be $S_{ACC}$-optimal.
    \item Proposition~\ref{prop:sound7} implies that if a hypothesis $H$ entails no negative examples, we may prune all larger specializations of $H$ as they cannot be $S_{ACC}$-optimal. 
\end{itemize}

\section{Relaxed LFF Hypothesis Constraints with Hypothesis Size}

Since an LFF solution must entail all positive examples and no negative examples, it is likely to significantly overfit the data in the presence of noise. In ILP, overfitting often is seen through overly large programs with extra clauses  specifically covering noisy examples, as we can illustrate:

\begin{example}[Large Overfitting Hypotheses]
\label{exp:overfitting}
Suppose our sets of examples for a east-west trains problem are as follows:

\begin{center}
    \begin{tabular}{l}
        $E^+$ = \{\texttt{eastbound(train$_1$).}, \texttt{eastbound(train$_2$).}, \texttt{eastbound(train$_3$)}\}\\
        $E^-$ = \{\texttt{eastbound(train$_4$.})\}\\
    \end{tabular}
\end{center}

with background knowledge:

\begin{center}
    \begin{tabular}{l}
        \texttt{has$\_$car(train$_1$, car$_1$)., two$\_$wheels(car$_1$)., long(car$_1$).,}\\
        \texttt{has$\_$car(train$_2$, car$_2$)., two$\_$wheels(car$_2$)., roof$\_$closed(car$_2$).,}\\
        \texttt{has$\_$car(train$_3$, car$_3$)., three$\_$wheels(car$_3$)., short(car$_3$).,}\\
        \texttt{has$\_$car(train$_4$, car$_4$)., three$\_$wheels(car$_4$).}\\
    \end{tabular}
\end{center}

\noindent Also suppose that the \texttt{eastbound(train$_3$).} fact is noisy and should truly be a negative example. If it had been correctly classified, we see that an LFF optimal solution to this problem would be:

\[\texttt{h$_1$ =} \left\{\begin{array}{c}
    \texttt{eastbound(A) :- has$\_$car(A,B), two$\_$wheels(B).}
\end{array}\right\}\]

\noindent However, because of the noisy fact, we are required to add an additional clause to entail this fact. So, an LFF optimal solution to this noisy problem could be:

\[\texttt{h$_2$ =} \left\{\begin{array}{l}
    \texttt{eastbound(A) :- has$\_$car(A,B), two$\_$wheels(B).}\\
    \texttt{eastbound(A) :- has$\_$car(A,B), three$\_$wheels(B), short(B).}
\end{array}\right\}\]

\noindent Note that \texttt{h$_2$} has nearly over double the literals as \texttt{h$_1$}. Like in many machine learning methods, the size or complexity of the model can be correlated with overfitting \cite{mitchell1997machine}. In the worst case for an ILP problem, we could generate a program in which there is exactly one clause entailing a single positive fact, meaning that the entire program would contain $|E^+|$ clauses. For instance, a naive program for the example above would be:

\[\texttt{h$_3$ =} \left\{\begin{array}{l}
    \texttt{eastbound(A) :- has$\_$car(A,B), two$\_$wheels(B), long(B).}\\
    \texttt{eastbound(A) :- has$\_$car(A,B), two$\_$wheels(B), roof$\_$closed(B).}\\
    \texttt{eastbound(A) :- has$\_$car(A,B), three$\_$wheels(B), short(B).}
\end{array}\right\}\]

\medskip \noindent or equivalently, we generate a model which simply remembers all positive examples, storing each in its entirety. But in \texttt{h$_3$}, we can clearly see how the first two clauses could simply be combined into \texttt{eastbound(A) :- has$\_$car(A,B), two$\_$wheels(B).} which is a generalization of both. Additionally, a naive hypothesis such as this may generate clauses which subsume each other, making the subsumed clauses redundant. Most importantly, such a hypothesis would not generalize at all to data outside of the training set: unless any inputs given to the program were in the training set, the program will always return false. Though constructing such hypotheses is trivial, it will overfit any noisy data, is not useful, and is impractically large given large $E^+$. 
\end{example}

\paragraph{Application of Minimum Description Length Principle} This provides clear motivation why optimality in ILP is typically tied to the size of programs. With the presence of noise, the relaxed LFF setting has to be particularly cognizant of overfitting and avoid exceptionally large programs or programs where single clauses entail one single positive example. To this end, we look to apply the \textit{minimum description length} (MDL) principle, a common method used for machine learning model selection \cite{quinlan1989inferring, alves2004improving, Gr_nwald_2019}. At a high level, the MDL principle states that the optimal hypothesis or theory is one where the sum of the theory length and the length of the training data when encoded with that theory is a minimum. If we apply this idea to Example~\ref{exp:overfitting} above, hypothesis \texttt{h$_2$} fits all four examples perfectly, but has a size of seven literals. Hypothesis \texttt{h$_1$} however fits three of the examples with a size of only three literals. Taking the sum of correctly fit examples and programs length yields totals of 11 for \texttt{h$_1$} and 6 for \texttt{h$_1$}. Under the MDL setting, we would claim \texttt{h$_1$} is more optimal than \texttt{h$_2$} for this particular encoding. 
An interested reader should consult \cite{rissanen1978modeling} and \cite{Gr_nwald_2019} for more details on MDL and its applications to machine learning. 

\bigskip \noindent In order to apply the MDL principle, we define an alternative hypothesis scoring method which takes into account the hypothesis size, recalling from Definition~\ref{def:hypothesissize} $size(H)$ equals the number of literals in hypothesis $H$.

\begin{definition}[MDL Score]
Given an input tuple $(B,D,C,E^+,E^-)$ and a hypothesis $H \in \mathcal{H}_{D,C}$, the function $S_{MDL} = tp(H,B,E^+) + tn(H,B,E^-) - size(H)$.
\end{definition}

\bigskip \noindent Note that this definition is similar to what the Aleph refers to as its \textit{compression} evaluation function \cite{srinivasan2001aleph}. With this definition, we can define several additional circumstances when hypotheses are known to be suboptimal under $S_{MDL}$ scoring. First, we consider comparing two hypotheses with one another:

\begin{proposition}
Given problem input $(B,D,C,E^+,E^-,S_{MDL})$ and hypotheses $H_1$, $H_2$, and $H_3$ in $\mathcal{H}_{D,C}$ where $H_3$ is a generalization of $H_1$, if \\ $S_{MDL}(H_2,B,E^+,E^-) - S_{MDL}(H_1,B,E^+,E^-) > fn(H_1,B,E^+) - (size(H_3) - size(H_2))$, then $S_{MDL}(H_2,B,E^+,E^-) > S_{MDL}(H_3,B,E^+,E^-)$.
\label{prop:mdl1}
\end{proposition}

\begin{proof}
The improvement in the MDL score from $H_1$ to $H_2$ can be quantified by:
\scriptsize
\begin{align*}
    S_{MDL}(H_2,B,E^+,E^-) - S_{MDL}(H_1,B,E^+,E^-) &> fn(H_1,B,E^+) - (size(H_3) - size(H_1))\\
    S_{MDL}(H_2,B,E^+,E^-) - [tp(H_1,B,E^+) + tn(H_1,B,E^-) - size(H_1)] &> |E^+| - tp(H_1,B,E^+) - (size(H_3) - size(H_1))\\
    S_{MDL}(H_2,B,E^+,E^-) &> |E^+| + tn(H_1,B,E^-) - size(H_3)\\
    &\geq tp(H_3,B,E^+) + tn(H_3,B,E^-) - size(H_3)\\
    &= S_{MDL}(H_3,B,E^+,E^-)
\end{align*}
\normalsize
The second to last line following from Proposition~\ref{prop:scoresofgeneralizations} and that $H_3$ is a generalization of $H_1$ in addition to the fact $|E^+| \geq tp(H_3,B,E^+)$
\end{proof}

\bigskip \noindent Given $S_{MDL}(H_2,B,E^+,E^-)$ and $S_{MDL}(H_1,B,E^+,E^-)$, we can quantify the exact size of $H_3$ when this holds:

\begin{align*}
    S_{MDL}(H_2,B,E^+,E^-) &> |E^+| + tp(H_1,B,E^+) - size(H_3)\\
    S_{MDL}(H_2,B,E^+,E^-) - |E^+| - tn(H_1,B,E^-) &> -size(H_3)\\
    |E^+| + tn(H_1,B,E^-) - S_{MDL}(H_2,B,E^+,E^-) &< size(H_3)
\end{align*}

\begin{proposition} 
Given problem input $(B,D,C,E^+,E^-,S_{MDL})$ and hypotheses $H_1$, $H_2$, and $H_3$ in $\mathcal{H}_{D,C}$ where $H_3$ is a specialization of $H_1$, if \\ $S_{MDL}(H_2,B,E^+,E^-) - S_{MDL}(H_1,B,E^+,E^-) > fp(H_1,B,E^-) - (size(H_3) - size(H_2))$, then $S_{MDL}(H_2,B,E^+,E^-) > S_{MDL}(H_3,B,E^+,E^-)$.
\label{prop:mdl2}
\end{proposition}

\begin{proof}
The improvement in the MDL score from $H_1$ to $H_2$ can be quantified by:
\scriptsize
\begin{align*}
    S_{MDL}(H_2,B,E^+,E^-) - S_{MDL}(H_1,B,E^+,E^-) &> fp(H_1,B,E^-) - (size(H_3) - size(H_1))\\
    S_{MDL}(H_2,B,E^+,E^-) - [tp(H_1,B,E^+) + tn(H_1,B,E^-) - size(H_1)] &> |E^-| - tn(H_1,B,E^-) - (size(H_3) - size(H_1))\\
    S_{MDL}(H_2,B,E^+,E^-) &> |E^-| + tp(H_1,B,E^+) - size(H_3)\\
    &\geq tn(H_1',B,E^-) + tp(H_3,B,^+) - size(H_3)\\
    &= S_{MDL}(H_3,B,E^+,E^-)
\end{align*}
\normalsize
The second to last line following from Proposition~\ref{prop:scoresofspecializations} and that $H_3$ is a specialization of $H_1$ in addition to the fact $|E^-| \geq tn(H_3,B,E^-)$.
\end{proof}

\bigskip \noindent Given $S_{MDL}(H_2,B,E^+,E^-)$ and  $S_{MDL}(H_1,B,E^+,E^-)$, we can quantify the exact size of $H_1'$ when this holds:

\begin{align*}
    S_{MDL}(H_2,B,E^+,E^-) &> |E^-| + tp(H_1,B,E^+) - size(H_3)\\
    S_{MDL}(H_2,B,E^+,E^-) - |E^-| - tp(H_1,B,E^+) &> -size(H_3)\\
    |E^-| + tp(H_1,B,E^+) - S_{MDL}(H_2,B,E^+,E^-) &< size(H_3)
\end{align*}

\bigskip \noindent Additionally, we can identify some situations that do not rely on comparing hypotheses. Recall that for a hypothesis $H$, we can write $fn(H,B,E^+) = |E^+| - tp(H,B,E^+)$ and similarly $fp(H,B,E^-) = |E^-| - tn(H,B,E^-)$.

\begin{proposition}
Given problem input $(B,D,C,E^+,E^-,S_{MDL})$ and hypotheses $H$ and $H'$ in $\mathcal{H}_{D,C}$ where $H'$ is a generalization of $H$, if $size(H') > fn(H,B,E^+) + size(H)$, then $S_{MDL}(H,B,E^+,E^-) > S_{MDL}(H',B,E^+,E^-)$.
\label{prop:mdl3}
\end{proposition}

\begin{proof}
By Propsition~\ref{prop:scoresofgeneralizations} the maximum value for $S_{MDL}(H',B,E^+,E^-)$ is: 
\begin{align*}
    S_{MDL}(H',B,E^+,E^-) &= |E^+| + tn(H,B,E^-) - size(H')\\
    &< |E^+| + tn(H,B,E^-) - [|E^+| - tp(H,B,E^+) + size(H)]\\
    &= tp(H,B,E^+) + tn(H,B,E^-) - size(H)\\
    &= S_{MDL}(H,B,E^+,E^-)
\end{align*}
\end{proof}

\begin{proposition}
Given problem input $(B,D,C,E^+,E^-,S_{MDL})$ and hypotheses $H$ and $H'$ in $\mathcal{H}_{D,C}$ where $H'$ is a specialization of $H$, if $size(H') > fp(H,B,E^-) + size(H)$, then $S_{MDL}(H,B,E^+,E^-) > S_{MDL}(H',B,E^+,E^-)$.
\label{prop:mdl4}
\end{proposition}

\begin{proof}
By Proposition~\ref{prop:scoresofspecializations} the maximum value for $S_{MDL}(H',B,E^+,E^-)$ is: 
\begin{align*}
    S_{MDL}(H',B,E^+,E^-) &= |E^-| + tp(H,B,E^+) - size(H')\\
    &< |E^-| + tp(H,B,E^+) - [|E^-| - tn(H,B,E^-) + size(H)]\\
    &= tp(H,B,E^+) + tn(H,B,E^-) - size(H)\\
    &= S_{MDL}(H,B,E^+,E^-)
\end{align*}
\end{proof}

\subsection{Hypothesis Constraints with Hypothesis Size Applications}
These propositions all apply in any ILP setting, however, in our relaxed LFF setting, they determine sets of programs of specific lengths which should be pruned should certain conditions hold:

\begin{itemize}
    \item Proposition~\ref{prop:mdl1} implies that given hypotheses $H_1$ and $H_2$, we may prune any hypothesis of size greater than $|E^+| + tn(H_1,B,E^-) - S_{MDL}(H_2,B,E^+,E^-)$ which is also a generalization of $H_1$ as they cannot be $S_{MDL}$-optimal.
    \item Proposition~\ref{prop:mdl2} implies that given hypotheses $H_1$ and $H_2$, we may prune any hypothesis of size greater than $|E^-| + tp(H_1,B,E^+) - S_{MDL}(H_2,B,E^+,E^-)$ which is also a specialization of $H_1$ as they cannot be $S_{MDL}$-optimal.
    \item Proposition~\ref{prop:mdl3} implies that given hypothesis $H$, we may prune all generalizations of $H$ with size greater than $fn(H,B,E^+) + size(H)$ as they cannot be $S_{MDL}$-optimal.
    \item Proposition~\ref{prop:mdl4} implies that given hypothesis $H$, we may prune all specializations of $H$ with size greater than $fp(H,B,E^-) + size(H)$ as they cannot be $S_{MDL}$-optimal.
\end{itemize}

\section{Summary}
In this chapter, we introduced the novel contribution of the relaxed LFF setting and its problem definitions. We explained how in order to avoid the overpruning, we relax how hypothesis constraints should be used. We next introduced and proved the original and sound hypothesis constraints in Propositions~\ref{prop:sound1}-\ref{prop:sound7}. We next demonstrated how the MDL principle can be applied in this setting to avoid overfitting by taking into account program length. We concluded by introducing and Propositions~\ref{prop:mdl1}-\ref{prop:mdl4} which describe sound hypothesis constraints which take into account hypothesis size under this MDL scoring. In the next chapter, we discuss Noisy Popper's implementation of the relaxed LFF framework, describing preliminaries of Normal Popper's implementation as needed. 

\end{chapter}
\begin{chapter}
{Noisy Popper Implementation}
In this chapter, we will discuss the implementation of Noisy Popper using the relaxed LFF approach. We start by discussing the implementation of Normal Popper as is necessary to understand Noisy Popper. After this, we will explain the specific implementation differences used by Noisy Popper including its \textit{anytime algorithm} approach, its use of \textit{minimal constraints} to efficiently prune the search space, and finally the implementation of the \textit{sound hypothesis constraints} under the $S_{ACC}$ scoring and \textit{sound hypothesis constraints with hypothesis size} under the $S_{MDL}$ scoring discussed in Chapter 4.

\section{Normal Popper Implementation}
It is necessary to discuss the Normal Popper implementation as Noisy Popper is an extension of it and uses the same structure and functions with slight modifications. In implementation, Normal Popper combines use of the Prolog logic programming language with ASP in a three stage generate-test-constrain loop. Algorithm 1 \cite{popper} below illustrates these three stages of the Normal Popper algorithm. We will discuss the generate, test, and constrain loop implementation here and provide illustrative examples, however an interested reader should consult \cite{popper} for full details.

\begin{algorithm}
\label{alg:normalpopper}
\caption{Normal Popper \cite{popper}}\label{alg:cap}
\begin{algorithmic}[1]
    \Require $E^+$, $E^-$, BK, $D$, $C$, max$\_$vars, max$\_$literals, max$\_$clauses (where $D$ is a declaration bias and $C$ is a set of constraints)
    \Ensure LFF Solution or Empty Set
    \State num$\_$literals $\gets 1$
    \While{num$\_$literals $\leq$ max$\_$literals}
        \State program $\gets$ generate($D$, $C$, max$\_$vars, num$\_$literals, max$\_$clauses)
        \If{program = 'space$\_$exhausted'}
            \State num$\_$literals $\gets$ num$\_$literals + 1
            \State continue
        \EndIf
        \State (tp, tn) $\gets$ test($E^+$, $E^-$, BK, program)
        \If{tp = $|E^+|$ \textbf{and} tn = $|E^-|$}
            \Return program
        \EndIf
        \State $C$ $\gets$ $C$ + learn$\_$constraints(program, tp, tn)
    \EndWhile\\
    \Return $\{\}$
\end{algorithmic}
\end{algorithm}

\paragraph{Generate} 
The $generate$ function in the Popper algorithm takes the declaration bias, current set of hypothesis constraints as inputs along with upper bounds on the number of literals allowed within clauses and number of clauses allowed within a hypothesis. The hypothesis constraints determine the syntax of each valid hypothesis, e.g., the number of clauses allowed, which predicates can appear together, which clauses are allowed, etc. Hypothesis constraints are constructed as ASP constraints. A defined meta-language is used to encode programs from Prolog to ASP which are then used to construct constraints. Collectively, these constraints form an ASP problem whose answer sets consist of programs which satisfy the given constraints, i.e., are consistent with all declaration and hypothesis constraints. Such an approach was discussed with other ILP systems \cite{corapi2011inductive, law2018inductive, hexmil, ilasp, ilasp3, schuller2018best} in Chapter 2. The $generate$ function returns an answer set to the current ASP problem. This answer set represents a candidate definite program which the system considers as a potential solution. Normal Popper additionally removes invalid hypotheses such as recursive programs without a base case as well as redundant hypothesis in which one clause subsumes another. If there is no answer set to the ASP problem with the specified number of body literals $num$\_$literals$, then 'space$\_$exhausted' is returned rather than a candidate program and the number of body literals allowed is incremented by one. In this way, Normal Popper searches the hypothesis space in order of increasing program size.

\paragraph{Test}
After the generate stage of the Normal Popper loop, in the $test$ function, Normal Popper converts the provided answer set from the generate stage into a Prolog program. The system tests this candidate program against the positive and negative examples with the BK provided to determine which examples are entailed and which are not. Table \ref{table:10} \cite{popper} below illustrates the possible outcomes of the program as well as the constraints which will later be generated from them. Note that we are using short hard where $tp = tp($program$,B,E^+)$ and $tn = tn($program$,B,E^-)$. Outcomes are tuples consisting of a true positive and true negative score. For example, an outcome of (5, 0) indicates that the given hypothesis entails only five positive examples and does not entail any negative examples (i.e. is consistent).

\begin{table}[h!]
    \centering
    \resizebox{\columnwidth}{!}{%
    \begin{tabular}{c|c|c} 
    \toprule
    \textbf{Outcome} & \textbf{tn = $|E^-|$} & \textbf{tn $<$ $|E^-|$} \\
    \midrule
    \textbf{tp = $|E^+|$} & n/a & Generalization \\
    \textbf{0 $<$ tp $<$ $|E^+|$} & Specialization & Specialization, Generalization \\
    \textbf{tp = 0} & Specialization, Elimination & Specialization, Generalization, Elimination \\
    \midrule
    \end{tabular}}
    \caption{The possible outcomes of testing a hypothesis with the constraints learned by normal popper. Note that an outcome where tp = $|E^+|$ and tn = $|E^-|$ indicates the hypothesis is a solution.}
    \label{table:10}
\end{table}

\bigskip \noindent If the hypothesis is found to be a solution, that is if tp = $|E^+|$ and tn = $|E^-|$, the program is returned by the system. Otherwise, Normal Popper continues onto the constrain stage which uses the failed hypothesis outcome in order to generate additional hypothesis constraints and further prune the hypothesis space.

\paragraph{Constrain}
In the case of a failed hypothesis, Normal Popper uses the hypothesis outcome to determine which hypothesis constraints to apply. Table \ref{table:10} depicts the constraints associated with each possible outcome. We will describe how a hypothesis is encoded as an ASP constraint as Noisy Popper uses modified versions of these encodings though an interested reader should consult \cite{popper} for a more detailed explanation. We will explain the general form these ASP constraints take and give examples for of generalization, specialization, and elimination constraints based on their definitions from Chapter 3 as well as banish constraints which, while less essential to the base version of Normal Popper, are critical for Noisy Popper's implementation.

\subsection{Hypothesis to ASP Encoding}
As mentioned in Example \ref{exp:hypothesisconstraints} in Section~\ref{sec:hypothesisconstraints}, the meta-language which encodes atoms to ASP programs using either \texttt{head$\_$literal(Clause,Pred,Arity,Vars)} for head literals or \texttt{body$\_$literal(Clause,Pred,Arity,Vars)} for body literals where \texttt{Clause} denotes the clause containing the literal, \texttt{Pred} defines the predicate symbol of the literal, \texttt{Arity} defines the arity of the predicate and \texttt{Vars} is a tuple containing input variables to the predicate. To do this in Normal Popper, two functions are called: \texttt{encodeHead} and \texttt{encodeBody} as defined here as they are in \cite{popper}:

\begin{center}
    \begin{tabular}{l}
         \texttt{encodeHead(Clause,Pred(Var0,...,Vark)) :=} \\
         \quad \texttt{head\_literal(Clause,Pred,k + 1,(encodeVar(Var0),...,encodeVar(Vark)))} \\
         \\
         \texttt{encodeBody(Clause,Pred(Var0,...,Vark)) :=} \\
         \quad \texttt{body\_literal(Clause,Pred,k + 1,(encodeVar(Var0),...,encodeVar(Vark)))}
    \end{tabular}
\end{center}

\noindent where \texttt{encodeVar} converts variables to an ASP encoding.

\begin{example}[Atom Encoding]
If we are trying to encode \texttt{eastbound(A)} as a head literal and \texttt{has$\_$car(A,B)} as a body literal, both in a clause \texttt{C1}, we would call \texttt{encodeHead(C1,eastbound(A))} and \texttt{encodeBody(C1,has$\-$car(A,B))} which would return ASP programs \texttt{head$\_$literal(C1,eastbound,1,(V0))} and \\ \texttt{body$\_$literal(C1,has$\_$car,2,(V0,V1))} respectively.
\end{example}

\medskip \noindent Encoding entire clauses naturally builds from encoding literals using the function \texttt{encodeClause} defined as it is in \cite{popper}:
\medskip
\begin{center}
    \begin{tabular}{l}
        \texttt{encodeClause(Clause,(head:-body$_1$,...,body$_n$)) :=} \\
        \quad \texttt{encodeHead(Clause,head)},\\
        \quad \texttt{encodeBody(Clause,body1),...,encodeBody(Clause,body$_n$)},\\
        \quad \texttt{assertDistinct(vars(head) $\cup$ vars(body$_1$) $\cup$...$\cup$ vars(body$_m$))}
    \end{tabular}
\end{center}

\medskip \noindent where the \texttt{assertDistinct} function simply imposes a pairwise inequality constraint on all variables. For instance, if the three encoded variables are \texttt{V0, V1} and \texttt{V2}, this function simply returns the constraints: \texttt{V0!=V1, V0!=V2, V1!=V2}. Since clauses can appear in multiple hypotheses, Normal Popper uses the \texttt{clauseIdent} function which maps clauses to ASP constraints using a unique identifier. This identifier is used in the ASP literal \texttt{included$\_$clause(Clause,Id)} which indicates that a clause \texttt{Clause} includes all literals of the clause identified by \texttt{Id}. This leads to the \texttt{inclusionRule} function as defined in \cite{popper}:
\medskip
\begin{center}
    \begin{tabular}{l}
         \texttt{inclusionRule(head:-body$_1$,...,body$_n$) :=}\\
         \quad \texttt{included$\_$clause(Cl,clauseIdent(head:-body$_1$,...,body$_n$)):-}\\
         \quad \quad \texttt{encodeClause(Cl,(head:-body$_1$,...,body$_n$)).}
    \end{tabular}
\end{center}

\medskip \noindent This function's head is true if all of the literals of the provided clause appear simultaneously in a clause. Note that this may hold true even if additional literals not in the provided clause are present. To ensure that a provided clause appears exactly in a program, we define the \texttt{exactClause} as in \cite{popper}:
\medskip
\begin{center}
    \begin{tabular}{l}
        \texttt{exactClause(Clause,(head:-body$_1$,...,body$_m$)) :=}\\
        \quad \texttt{included$\_$clause(Clause,clauseIdent(head:-body$_1$,...,body$_n$)),}\\
        \quad \texttt{clause$\_$size(Clause,n)}
    \end{tabular}
\end{center}

\medskip \noindent where the function \texttt{clause$\_$size(Clause,n)} asserts true only if the given clause contains exactly $n$ body literals.

\begin{example}[Clause Encoding]
If we are trying to encode the clause \\ \texttt{eastbound(A) :- has$\_$car(A,B), two$\_$wheels(B).} and we suppose that \\ \texttt{clauseIdent(eastbound(A) :- has$\_$car(A,B), two$\_$wheels(B).) = id$_1$}, then we define an inclusion rule with the function \\  \texttt{inclusionRule(eastbound(A):-has$\_$car(A,B),two$\_$wheels(B))} which returns:

\begin{center}
    \begin{tabular}{l}
        \texttt{included$\_$clause(Cl,id$_1$) :-}\\
        \quad \texttt{head$\_$literal(Cl,eastbound,1,(V0))},\\
        \quad \texttt{body$\_$literal(Cl,has$\_$car,2,(V0,V1))},\\
        \quad \texttt{body$\_$literal(Cl,two$\_$wheels,1,(V1))},\\
        \quad \texttt{V0!=V1.}
    \end{tabular}
\end{center}

\medskip \noindent and if we wish to ensure this clause appears exactly in a hypothesis, we use the function \texttt{exactClause(C1,eastbound(A):-has$\_$car(A,B),two$\_$wheels(B))} which returns:

\begin{center}
    \begin{tabular}{l}
        \texttt{included$\_$clause(C1,id$_1$)},\\
        \texttt{clause$\_$size(Cl,2)}.
    \end{tabular}
\end{center}
\end{example}

\medskip \noindent Now that we have defined the ASP encoding used by Normal and Noisy Popper, we can define the exact forms which constraints take.

\subsection{Generalization Constraints}
By Definition~\ref{def:generalization}, a generalization of a hypothesis $H$ is a program which contains exactly all of $H$'s exact clauses \cite{popper}. Thus, the generalization constraints of Definition~\ref{def:generalizationconstraint} can be defined as in \cite{popper}:

\begin{center}
    \begin{tabular}{l}
         \texttt{generalizationConstraint($\{$Clause$_1$, Clause$_2$,...,Clause$_n\}$) :=}\\
         \quad \texttt{inclusionRule(Clause$_1$),...,}\\
         \quad \texttt{inclusionRule(Clause$_n$)}.\\
         \quad \texttt{:-exactClause(C1$_1$,Clause$_1$),...,}\\
         \quad \quad \texttt{exactClause(C1$_n$,Clause$_n$)}.
    \end{tabular}
\end{center}

\begin{example}[Generalization Constraint Encoding]
The ASP encoding for the inclusion rule and generalization constraint of the hypothesis \\ \texttt{h = $\{$eastbound(A) :- has$\_$car(A,B), two$\_$wheels(B).$\}$} would be:

\begin{center}
    \begin{tabular}{l}
        \texttt{included$\_$clause(Cl,id$_1$) :-}\\
        \quad \texttt{head$\_$literal(Cl,eastbound,1,(V0))},\\
        \quad \texttt{body$\_$literal(Cl,has$\_$car,2,(V0,V1))},\\
        \quad \texttt{body$\_$literal(Cl,two$\_$wheels,1,(V1))},\\
        \quad \texttt{V0!=V1.}\\
        \texttt{:-}\\
        \quad \texttt{included$\_$clause(C10,id$_1$)},\\
        \quad \texttt{clause$\_$size(C10,2).}
    \end{tabular}
\end{center}
\end{example}

\subsection{Specialization Constraints}
By Definition~\ref{def:specialization}, a specialization of a hypothesis $H$ is a program which contains all of $H$'s clauses, which may be specialized, and no additional clauses \cite{popper}. Thus, the specialization constraints of Definition~\ref{def:specializationconstraint} can be defined as in \cite{popper}:

\begin{center}
    \begin{tabular}{l}
        \texttt{specializationConstraint($\{$Clause$_1$, Clause$_2$,...,Clause$_n\}$) :=}\\
        \quad \texttt{inclusionRule(Clause$_1$),...,}\\
        \quad \texttt{inclusionRule(Clause$_n$)}.\\
        \quad \texttt{:-included$\_$clause(C1$_1$,clauseIdent(Clause$_1$)),...,}\\
        \quad \quad\texttt{included$\_$clause(C1$_n$,clauseIdent(Clause$_n$)),}\\
        \quad \quad \texttt{assertDistinct($\{$Cl$_1$,...,Cl$_n\}$), not clause(n).}
    \end{tabular}
\end{center}

\medskip \noindent The \texttt{not clause(n)} literal is only satisfied if there are no more than the \texttt{n} distinct clauses given in the constraint. Note, since all clauses may be specialized, we do not require clauses to be exact as we do in generalization constraints.

\begin{example}[Specialization Constraint Encoding]
The ASP encoding for the inclusion rule and specialization constraint of the hypothesis:

\begin{center}
    \begin{tabular}{l}
        \texttt{h =} $\left\{\begin{array}{l}
        \texttt{eastbound(A,B) :- has$\_$car(A,B), two$\_$wheels(B).}\\
        \texttt{eastbound(A,B) :- has$\_$car(A,B), short(B).}
        \end{array}\right\}$
    \end{tabular}
\end{center}

\noindent would be:

\begin{center}
    \begin{tabular}{l}
        \texttt{included$\_$clause(Cl,id$_2$) :-}\\
        \quad \texttt{head$\_$literal(Cl,eastbound,1,(V0))},\\
        \quad \texttt{body$\_$literal(Cl,has$\_$car,2,(V0,V1))},\\
        \quad \texttt{body$\_$literal(Cl,two$\_$wheels,1,(V1))},\\
        \quad \texttt{V0!=V1.}\\
        \texttt{included$\_$clause(Cl,id$_3$) :-}\\
        \quad \texttt{head$\_$literal(Cl,eastbound,1,(V0))},\\
        \quad \texttt{body$\_$literal(Cl,has$\_$car,2,(V0,V1))},\\
        \quad \texttt{body$\_$literal(Cl,short,1,(V1))},\\
        \quad \texttt{V0!=V1.}\\
        \texttt{:-}\\
        \quad \texttt{included$\_$clause(C10,id$_2$)},\\
        \quad \texttt{included$\_$clause(C11,id$_3$)},\\
        \quad \texttt{C10!=C11, not clause(2).}
    \end{tabular}
\end{center}
\end{example}

\subsection{Elimination Constraints}
As outlined in Section~\ref{sec:eliminationconstraints}, in the case where a hypothesis $H$ is totally incomplete, we wish to prune all separable hypotheses which contain all clauses of $H$ where any of them may be specialized. Before we can describe the ASP encoding for an such a constraint, we require the following logic programs to determine separability \cite{popper}:

\begin{center}
    \begin{tabular}{l}
        \texttt{non$\_$separable :- head$\_$literal($\_$,P,A,$\_$), body$\_$literal($\_$,P,A,$\_$).}\\
        \texttt{separable :- not non$\_$separable.}
    \end{tabular}    
\end{center}

\medskip \noindent With this, we can define the encoding for an elimination constraint as is done in \cite{popper}:

\begin{center}
    \begin{tabular}{l}
         \texttt{eliminationConstraint($\{$Clause$_1$, Clause$_2$,...,Clause$_n\}$) :=}\\
         \quad \texttt{inclusionRule(Clause$_1$),...,}\\
         \quad \texttt{inclusionRule(Clause$_n$)}.\\
         \quad \texttt{:-included$\_$clause(C1$_1$,clauseIdent(Clause$_1$)),...,}\\
         \quad \quad \texttt{included$\_$clause(C1$_n$,clauseIdent(Clause$_n$)),}\\
         \quad \quad \texttt{separable.}
    \end{tabular}
\end{center}

\begin{example}[Elimination Constraint Encoding]
The ASP encoding for the elimination constraint for the hypothesis\\ \texttt{h = $\{$eastbound(A) :- has$\_$car(A,B), two$\_$wheels(B).$\}$} would be:

\begin{center}
    \begin{tabular}{l}
        \texttt{included$\_$clause(Cl,id$_4$) :-}\\
        \quad \texttt{head$\_$literal(Cl,eastbound,1,(V0))},\\
        \quad \texttt{body$\_$literal(Cl,has$\_$car,2,(V0,V1))},\\
        \quad \texttt{body$\_$literal(Cl,two$\_$wheels,1,(V1))},\\
        \quad \texttt{V0!=V1.}\\
        \texttt{:-}\\
        \quad \texttt{included$\_$clause(C10,id$_4$)},\\
        \quad \texttt{separable.}
    \end{tabular}
\end{center}
\end{example}

\subsection{Banish Constraints}
The final type of hypothesis constraint Normal Popper implements is to remove a single hypothesis and is known as a \textit{banish constraint}. While Normal Popper only used this constraint for testing purposes and not in its full implementation, Noisy Popper makes extensive use of it since in the relaxed setting it is common that a failed hypothesis does not generate a hypothesis constraints which prunes itself from the hypothesis space. A banish constraint simply provides that the each clause of a given hypothesis appears in the program, non-specialized, and with no additional clauses. The exact encoding as seen in \cite{popper} is as follows:

\begin{center}
    \begin{tabular}{l}
         \texttt{banishConstraint($\{$Clause$_1$, Clause$_2$,...,Clause$_n\}$) :=}\\
         \quad \texttt{inclusionRule(Clause$_1$),...,}\\
         \quad \texttt{inclusionRule(Clause$_n$)}.\\
         \quad \texttt{:-exact$\_$clause(C1$_1$,Clause$_1$),...,}\\
         \quad \quad \texttt{exact$\_$clause(C1$_n$,(Clause$_n$),}\\
         \quad \quad \texttt{not clause(n).}
    \end{tabular}
\end{center}

\begin{example}[Banish Constraint Encoding]
The ASP encoding for the banish constraint for the hypothesis\\ \texttt{h = $\{$eastbound(A) :- has$\_$car(A,B), two$\_$wheels(B).$\}$} would be:

\begin{center}
    \begin{tabular}{l}
        \texttt{included$\_$clause(Cl,id$_5$) :-}\\
        \quad \texttt{head$\_$literal(Cl,eastbound,1,(V0))},\\
        \quad \texttt{body$\_$literal(Cl,has$\_$car,2,(V0,V1))},\\
        \quad \texttt{body$\_$literal(Cl,two$\_$wheels,1,(V1))},\\
        \quad \texttt{V0!=V1.}\\
        \texttt{:-}\\
        \quad \texttt{included$\_$clause(C10,id$_5$)},\\
        \quad \texttt{clause$\_$size(C10,2),}\\
        \quad \texttt{not clause(1).}
    \end{tabular}
\end{center}
\end{example}

\subsection{Normal Popper Worked Example}
\label{sec:normalworkedexample}
To illustrate clearly how Normal Popper works, we will consider a modified example from \cite{popper} using the east-west trains problem. Assume we are trying to find the hypothesis \texttt{eastbound(A) :- has$\_$car(A,B), short(B), two$\_$wheels(B).} and will only consider a small initial hypothesis space, $\mathcal{H}_1$:

\begin{center}
    \begin{tabular}{l}
        $\mathcal{H}_1$ =  $\left\{\begin{array}{l}
            \texttt{h$_1$ =} \left\{\begin{array}{l}
            \texttt{eastbound(A) :- has$\_$car(A,B),long(B).}
            \end{array}\right\} \\
            \texttt{h$_2$ =} \left\{\begin{array}{l}
            \texttt{eastbound(A) :- has$\_$car(A,B),long(A),two$\_$wheels(B).}
            \end{array}\right\} \\
            \texttt{h$_3$ =} \left\{\begin{array}{l}
            \texttt{eastbound(A) :- has$\_$car(A,B),roof$\_$closed(B).}
            \end{array}\right\} \\
            \texttt{h$_4$ =} \left\{\begin{array}{l}
            \texttt{eastbound(A) :- has$\_$car(A,B),short(B),two$\_$wheels(B).}
            \end{array}\right\} \\
            \texttt{h$_5$ =} \left\{\begin{array}{l}
            \texttt{eastbound(A) :- has$\_$car(A,B),long(B),roof$\_$closed(B).}
            \end{array}\right\} \\
            \texttt{h$_6$ =} \left\{\begin{array}{l}
            \texttt{eastbound(A) :- has$\_$car(A,B),roof$\_$closed(B).}\\
            \texttt{eastbound(A) :- has$\_$car(A,B),short(B).}
            \end{array}\right\} \\
            \texttt{h$_7$ =} \left\{\begin{array}{l}
            \texttt{eastbound(A) :- has$\_$car(A,B),roof$\_$closed(B).}\\
            \texttt{eastbound(A) :- has$\_$car(A,B),long(C,D),three$\_$wheels(B).}
            \end{array}\right\} \\
            \texttt{h$_8$ =} \left\{\begin{array}{l}
            \texttt{eastbound(B) :- has$\_$car(A,B),short(B),three$\_$wheels(B).} \\ 
            \texttt{eastbound(A) :- has$\_$car(A,B),has$\_$load(B,C),triangle(C).}
            \end{array}\right\} \\
            \texttt{h$_9$ =} \left\{\begin{array}{l}
            \texttt{eastbound(A) :- has$\_$car(A,B),long(B).} \\ 
            \texttt{eastbound(A) :- has$\_$car(A,B),short(B),three$\_$wheels(D,B).} \\ 
            \texttt{eastbound(A) :- has$\_$car(A,B),short(B),two$\_$wheels(B).}
            \end{array}\right\} \\
        \end{array}\right\}$
    \end{tabular}
\end{center}

\bigskip \noindent We will also assume we have the following set of positive examples:

\begin{center}
    \begin{tabular}{l}
        $E^+ = \{$\texttt{eastbound(train$\_$1).}, \texttt{eastbound(train$\_$2).}$\}$\\
    \end{tabular}
\end{center}

\noindent And we will assume the following set of negative examples:

\begin{center}
    \begin{tabular}{l}
        $E^- = \{$\texttt{eastbound(train$\_$3).}, \texttt{eastbound(train$\_$4).}$\}$\\
    \end{tabular}
\end{center}

\noindent with the BK containing the following facts about the trains:

\begin{center}
    \begin{tabular}{l}
        \texttt{has$\_$car(train$_1$, car$_1$)., short(car$_1$)., two$\_$wheels(car$_1$), roof$\_$closed(car$_1$).}\\
        \texttt{has$\_$car(train$_2$, car$_2$)., short(car$_2$)., two$\_$wheels(car$_2$), jagged$\_$roof(car$_2$).}\\
        \texttt{has$\_$car(train$_2$, car$_3$)., three$\_$wheels(B)., roof$\_$closed(car$_3$).}\\
        \texttt{has$\_$car(train$_3$, car$_4$)., roof$\_$closed(car$_4$)., three$\_$wheels(car$_4$)., short(B)}\\
        \texttt{has$\_$car(train$_4$, car$_5$)., has$\_$load(car$_5$,load$_1$)., circle(load$_1$)., two$\_$wheels(B).}\\
    \end{tabular}
\end{center}

\noindent Note that \texttt{train$_2$} has 2 cars. Normal Popper will first generate the simplest hypothesis from the search space:

\begin{center}
    \begin{tabular}{l}
        \texttt{h$_1$ =} $\left\{\begin{array}{l}
        \texttt{eastbound(A,B) :- has$\_$car(A,B),long(B).}
        \end{array}\right\}$
    \end{tabular}
\end{center}

\noindent We can see that since neither \texttt{train$_1$} nor \texttt{train$_2$} contain a long car, they both will return false when input to this hypothesis. This makes \texttt{h$_1$} a failed hypothesis as it is totally incomplete which implies \texttt{h$_1$} is too specific. Normal Popper will generate the following specialization constraint: 

\begin{center}
    \begin{tabular}{l}
        \texttt{included$\_$clause(Cl,id$_1$) :-}\\
        \quad \texttt{head$\_$literal(Cl,eastbound,1,(V0))},\\
        \quad \texttt{body$\_$literal(Cl,has$\_$car,2,(V0,V1))},\\
        \quad \texttt{body$\_$literal(Cl,long,1,(V1))},\\
        \quad \texttt{V0!=V1.}\\
        \texttt{:-}\\
        \quad \texttt{included$\_$clause(C10,id$_1$)},\\
        \quad \texttt{not clause(1).}
    \end{tabular}
\end{center}

\noindent which prunes all specializations of \texttt{h$_1$} from $\mathcal{H}_1$, namely \texttt{h$_2$} and \texttt{h$_5$}. Since \texttt{h$_1$} is totally incomplete, Normal Popper will also generate the following elimination constraint: 

\begin{center}
    \begin{tabular}{l}
        \texttt{:-}\\
        \quad \texttt{included$\_$clause(C10,id$_1$)},\\
        \quad \texttt{separable.}
    \end{tabular}
\end{center}

\noindent which prunes all separable hypotheses which contain all of the clauses of \texttt{h$_1$} where each clause may be specialized. That means that \texttt{h$_9$} is pruned from the hypothesis space. After this pruning, our hypothesis space is left as:

\begin{center}
    \begin{tabular}{l}
        $\mathcal{H}_1$ =  $\left\{\begin{array}{l}
            \texttt{h$_3$ =} \left\{\begin{array}{l}
            \texttt{eastbound(A) :- has$\_$car(A,B),roof$\_$closed(B).}
            \end{array}\right\} \\
            \texttt{h$_4$ =} \left\{\begin{array}{l}
            \texttt{eastbound(A) :- has$\_$car(A,B),short(B),two$\_$wheels(B).}
            \end{array}\right\} \\
            \texttt{h$_6$ =} \left\{\begin{array}{l}
            \texttt{eastbound(A) :- has$\_$car(A,B),roof$\_$closed(B).}\\
            \texttt{eastbound(A) :- has$\_$car(A,B),short(B).}
            \end{array}\right\} \\
            \texttt{h$_7$ =} \left\{\begin{array}{l}
            \texttt{eastbound(A) :- has$\_$car(A,B),roof$\_$closed(B).}\\
            \texttt{eastbound(A) :- has$\_$car(A,B),long(C,D),three$\_$wheels(B).}
            \end{array}\right\} \\
            \texttt{h$_8$ =} \left\{\begin{array}{l}
            \texttt{eastbound(B) :- has$\_$car(A,B),short(B),three$\_$wheels(B).} \\ 
            \texttt{eastbound(A) :- has$\_$car(A,B),has$\_$load(B,C),triangle(C).}
            \end{array}\right\} \\
        \end{array}\right\}$
    \end{tabular}
\end{center}

\bigskip \noindent The next hypothesis Normal Popper will generate is:

\begin{center}
    \begin{tabular}{l}
        \texttt{h$_3$ =} $\left\{\begin{array}{l}
        \texttt{eastbound(A) :- has$\_$car(A,B),roof$\_$closed(B).}
        \end{array}\right\}$
    \end{tabular}
\end{center}

\noindent When we test this hypothesis, we find that it does entail both positive examples as both trains contain a car with a closed roof. However, \texttt{h$_3$} also entails the negative example \texttt{train$_3$}. This implies that the hypothesis is too general and thus Normal Popper will generate the following generalization constraint: 

\begin{center}
    \begin{tabular}{l}
        \texttt{included$\_$clause(Cl,id$_2$) :-}\\
        \quad \texttt{head$\_$literal(Cl,eastbound,1,(V0))},\\
        \quad \texttt{body$\_$literal(Cl,has$\_$car,2,(V0,V1))},\\
        \quad \texttt{body$\_$literal(Cl,roof$\_$closed,1,(V1))},\\
        \quad \texttt{V0!=V1.}\\
        \texttt{:-}\\
        \quad \texttt{included$\_$clause(C10,id$_2$)},\\
        \quad \texttt{clause$\_$size(C10,2).}
    \end{tabular}
\end{center}

\noindent which prunes all generalizations of \texttt{h$_3$}, namely \texttt{h$_6$} and \texttt{h$_7$}. Now, our hypothesis space is left as:

\begin{center}
    \begin{tabular}{l}
        $\mathcal{H}_1$ =  $\left\{\begin{array}{l}
            \texttt{h$_4$ =} \left\{\begin{array}{l}
            \texttt{eastbound(A) :- has$\_$car(A,B),short(B),two$\_$wheels(B).}
            \end{array}\right\} \\
            \texttt{h$_8$ =} \left\{\begin{array}{l}
            \texttt{eastbound(B) :- has$\_$car(A,B),short(B),three$\_$wheels(B).} \\ 
            \texttt{eastbound(A) :- has$\_$car(A,B),has$\_$load(B,C),triangle(C).}
            \end{array}\right\} \\
        \end{array}\right\}$ 
    \end{tabular}
\end{center}

\bigskip \noindent Finally, Normal Popper will generate hypothesis \texttt{h$_4$} which successfully entails all positive examples and no negative examples, making it a solution to the problem and is thus returned.

\bigskip \noindent The following sections will discuss how these constraints are adapted into Noisy Popper and the modifications Noisy Popper makes to better handle noise.

\section{Anytime Algorithm}

The first large obstacle Normal Popper presents when trying to handle noisy can be observed in Theorem 1 from \cite{popper}:

\paragraph{Theorem 1 (Optimality)}: [Normal] Popper returns an optimal solution if one exists \cite{popper}.

\bigskip \noindent Thus, given any set of examples, Normal Popper will either return a solution which entails all examples in $E^+$ and no examples in $E^-$, overfitting if noise is present, or no hypothesis at all which equates to an empty hypothesis, i.e., a program which always returns true. To avoid returning no hypothesis, Noisy Popper is constructed as an \textit{anytime algorithm} in which a hypothesis can be returned by the system at any point in its runtime, regardless of whether or not that hypothesis is an optimal solution. 

\bigskip \noindent The approach taken in Noisy Popper consists of maintaining the \textit{best hypothesis seen so far}. That is, each hypothesis generated by Popper is scored by the $S_{ACC}$ function from Definition~\ref{def:accuracyscore} and the hypothesis of highest score is maintained by the system. In the case that the Popper algorithm is halted early or the entirety of the hypothesis space is exhausted without finding a solution, the best hypothesis so far is returned. Otherwise, if an LFF solution is found which necessarily has maximum $S_{ACC}$ score, that solution is returned as it would be in Normal Popper. This change to the Normal Popper algorithm can be see in Algorithm 2 below.

\begin{algorithm}
\caption{Noisy Popper}\label{alg:noisy}
\begin{algorithmic}[1]
    \Require $E^+$, $E^-$, $B$, $D$, $C$, $t$, max$\_$vars, max$\_$literals, max$\_$programs, max$\_$clauses (where $B$ is a set of background knowledge, $D$ is a declaration bias, $C$ is a set of constraints, and $t$ is the minimal constraint threshold)
    \Ensure Hypothesis Constraint Consistent Logic Program or Empty Set
    \State num$\_$literals $\gets 1$
    \State num$\_$programs $\gets 1$
    \State best$\_$hypothesis $\gets null$
    \State program$\_$list $\gets$ []
    \While{num$\_$literals $\leq$ max$\_$literals \textbf{and} num$\_$programs $\leq$ max$\_$programs}
        \State program $\gets$ generate($D$, $C$, max$\_$vars, num$\_$literals, max$\_$clauses)
        \If{program = 'space$\_$exhausted'}
            \State num$\_$literals $\gets$ num$\_$literals + 1
            \State continue
        \EndIf
        \State (tp, tn) $\gets$ test($E^+$, $E^-$, $B$, program)
        \If{tp = $|E^+|$ \textbf{and} tn = $|E^-|$}
            \Return program
        \EndIf
        \If{$S_{ACC}$(program, $B$, $E^+$, $E^-$) $>$ $S_{ACC}$(best$\_$hypothesis, $B,E^+,E^-$)}
            \State best$\_$program $\gets$ program
        \EndIf
        \If{tp $>$ $t|E^+|$}
            \State tp $\gets$ $|E^+|$
        \EndIf
        \If{tn $>$ $t|E^-|$}
            \State tn $\gets$ $|E^-|$
        \EndIf
        \State $C$ $\gets$ $C$ + learn$\_$constraints(program, tp, tn)
        \State $C$ $\gets$ $C$ + learn$\_$sound$\_$constraints(program,program$\_$list, $B$, $E^+$, $E^-$, tp, tn)
        \State $C$ $\gets$ $C$ + learn$\_$size$\_$constraints(program, program$\_$list, $B$, $E^+$, $E^-$, tp, tn)
        \State append(program$\_$list, program)
    \EndWhile\\
    \Return best$\_$program
\end{algorithmic}
\end{algorithm}

\bigskip \noindent Note, Noisy Popper takes a \textit{max$\_$programs} parameter which essentially gives a timeout to the algorithm, forcing it to return whatever the current best hypothesis is after that many programs have been considered.

\section{Minimal Constraints} 
An effective strategy to apply constraints is to do so minimally and only in cases where we intuitively know that the hypotheses considered are poor. For instance, if a hypothesis entails no positive examples, we can confidently conclude that the hypothesis is too specific, even if a portion of those examples are noisy. Likewise, if a hypothesis entails all negative examples, we can conclude that it is much too general regardless of noise. To this end, given a hypothesis $H$, we consider applying the typical hypothesis constraints of Normal Popper as follows:
\begin{itemize}
    \item if $H$ is totally incomplete (i.e., tp = $0$) prune all specializations of $H$ and all separable hypotheses which contain a specialization of $H$.
    \item if $H$ is totally inconsistent (i.e., tn = $0$) prune all generalizations of $H$.
    \item Otherwise, only prune $H$ from the hypothesis space
\end{itemize}

\bigskip \noindent Though these constraints are not sound as they may prune $S_{ACC}$-optimal hypotheses they have proven very effective in practice greatly improve the efficiency of the system by pruning significant chunks of the hypothesis space.

\paragraph{Minimal Constraint Threshold} 
These minimal constraints arbitrarily, though reasonably, choose a threshold at which to prune at 0, i.e., Noisy Popper should prune normally if the $tp$ or $tn$ scores equal zero. However, Noisy Popper implements this threshold as an optional hyperparameter $0 \leq t \leq 1$ representing the percentage of positive (resp. negative) examples which if entailed (resp. not entailed) by a hypothesis, no pruning will occur. More specifically, typical constraints of Normal Popper are applied for a hypothesis $H$ as follows:  
\begin{itemize}
    \item if tp $\leq t|E^+|$ prune all specializations of $H$. If tp = $0$ prune all separable hypotheses which contain a specialization of $H$.
    \item if tn $\leq t|E^-|$ prune all generalizations of $H$.
    \item Otherwise, only prune $H$ from the hypothesis space
\end{itemize}

\noindent This implementation can be seen in lines 17-22 of the Noisy Popper algorithm which alters the values of $tp$ and $tn$ to $|E^+|$ and $|E^-|$ respectively should they exceed the threshold amounts. These modified $tp$ and $tn$ values are used as arguments for the standard $learn\_constraints$ function on line 23 which produces hypothesis constraints as it did in Algorithm 1. It is common that a program may have $tp$ and $tn$ altered to $|E^+|$ and $|E^-|$ respectively, i.e., complete relaxation of Normal Popper's constraints. In these situations, the $learn$\_$constraints$ function only generates a \emph{banish constraint} to ensure that that hypothesis is removed from the search space. Without this, the algorithm may consider that same hypothesis infinitely. Noisy Popper was designed to limit the use of hyperparameters as they make many ILP systems cumbersome to use effectively. By default, $t=0$ which is often the most effective setting.

\section{Sound Hypothesis Constraints}
The $learn\_sound\_constraints$ function in line 24 of the Noisy Popper algorithm is implemented in Algorithm 3 seen below. The algorithm compares previously generated programs with the newly generated one to build up a set of hypothesis constraints which are ultimately added to the ASP constraint set $C$. For simplicity, we will often refer to these specific constraints as \textit{sound constraints}. Note that $|E^+|-1$ is used to essentially convey "some but not all positive examples" and $|E^-|-1$ likewise conveys "some but not all negative examples". Lines 3-5 correspond to Proposition~\ref{prop:sound1} and prunes generalizations of a previously seen hypothesis if they cannot have an $S_{ACC}$ score higher than the new hypothesis. Likewise, lines 6-8 correspond to Proposition~\ref{prop:sound2} and prune specializations of a previously seen hypothesis if they cannot have an $S_{ACC}$ score higher than the new hypothesis.

\bigskip \noindent Lines 9-11 correspond to Propositions~\ref{prop:sound3} and \ref{prop:sound5} and prune non-recursive supersets and non-recursive specializations of the new hypothesis which has a true positive value equal to a previous hypothesis if the new hypothesis is a superset of the previously seen one. Lines 14-16 pertain just to Proposition~\ref{prop:sound3} as we cannot prune specializations of the new hypothesis if it is not a superset of the previously hypothesis. Lines 18-21 correspond to Proposition~\ref{prop:sound4} and prune generalizations of the new hypothesis which has a true negative value equal to a previously seen hypothesis if the new hypothesis is a specialization of the previously seen one.

\bigskip \noindent Lastly, lines 23-28 correspond to Propositions~\ref{prop:sound6} and \ref{prop:sound7} and prune all generalizations of the new hypothesis if it entails all positive examples and all specializations of the new hypothesis if it entails no negative examples. 

\begin{algorithm}
\caption{Learn Sound Hypothesis Constraints}\label{alg:sound}
\begin{algorithmic}[1]
    \Require program, program$\_$list, $B$, $E^+$, $E^-$, tp, tn (where $B$ is a set of background knowledge, tp = $tp$(program,$B,E^+$) and tn = $tn$(program,$N,E^-$) as calculated in Algorithm 1)
    \Ensure Set of hypothesis constraints (may be empty)
    \State constraints $\gets$ $\{\}$
    \For{p in program$\_$list}
        \If{$S_{ACC}($program,$B,E^+,E^-) - S_{ACC}($p,$B,E^+,E^-) > |E^+| - tp($p,$B,E^+$)}
            \State constraints $\gets$ constraints + learn$\_$constraints(p, $|E^+|$, $|E^-|-1$)
        \EndIf
        \If{$S_{ACC}($program,$B,E^+,E^-)- S_{ACC}($p,$B,E^+,E^- > |E^-| - tn($p,$B,E^-$)}
            \State constraints $\gets$ constraints + learn$\_$constraints(p, $|E^+|-1$, $E^-$)
        \EndIf
            
        \If{is$\_$generalization(program, p) \textbf{and} $tp($p$,B,E^+)$ = tp}
            \If{p $\subseteq$ program}
                \State constraints $\gets$ constraints + 
                \State \quad \quad learn$\_$constraints$\_$non$\_$rec(program, $|E^+|-1$, $|E^-|-1$)
            \Else
                \State constraints $\gets$ constraints + 
                \State \quad \quad learn$\_$constraints$\_$non$\_$rec(program, $|E^+|-1$, $|E^-|$)
            \EndIf
        \EndIf
        
        \If{is$\_$specialization(program, p) \textbf{and} $tn($p,$B,E^-)$ = tn}
            \State constraints $\gets$ constraints + 
            \State \quad \quad learn$\_$constraints$\_$non$\_$rec(program, $|E^+|$, $|E^-|-1$))
        \EndIf
    \EndFor
    \If{tp = $|E^+|$}
        \State constraints $\gets$ constraints + learn$\_$constraints(program, $|E^+|$, $|E^-|-1$)
    \EndIf
    \If{tn = $|E^-|$}
        \State constraints $\gets$ constraints + learn$\_$constraints(program, $|E^+|-1$, $|E^-|$)
    \EndIf\\
    \Return constraints
\end{algorithmic}
\end{algorithm}

\bigskip \noindent In Algorithm 3, while the standard $learn\_constraints$ function is used to generate ASP constraints as normal, a specific variant function $learn\_constraints\_non\_rec$ is also used to learn constraints which specifically do not prune recursive hypotheses. In the ASP encoding, this simply requires adding \texttt{not recursive} to the constraint. So, a generalization constraint which does not prune recursive hypotheses would be defined as:

\begin{center}
    \begin{tabular}{l}
         \texttt{generalizationConstraintNonRec($\{$Clause$_1$, Clause$_2$,...,Clause$_n\}$) :=}\\
         \quad \texttt{inclusionRule(Clause$_1$),...,}\\
         \quad \texttt{inclusionRule(Clause$_n$)}.\\
         \quad \texttt{:-exactClause(C1$_1$,Clause$_1$),...,}\\
         \quad \quad \texttt{exactClause(C1$_n$,Clause$_n$)},\\
         \quad \quad \texttt{not recursive.}
    \end{tabular}
\end{center}

\bigskip \noindent and likewise, an analogous specialization constraint would be defined as:

\begin{center}
    \begin{tabular}{l}
        \texttt{specializationConstraintNonRec($\{$Clause$_1$, Clause$_2$,...,Clause$_n\}$) :=}\\
        \quad \texttt{inclusionRule(Clause$_1$),...,}\\
        \quad \texttt{inclusionRule(Clause$_n$)}.\\
        \quad \texttt{:-included$\_$clause(C1$_1$,clauseIdent(Clause$_1$)),...,}\\
        \quad \quad\texttt{included$\_$clause(C1$_n$,clauseIdent(Clause$_n$)),}\\
        \quad \quad \texttt{assertDistinct($\{$Cl$_1$,...,Cl$_n\}$),} \\
        \quad \quad \texttt{not clause(n), not recursive.}
    \end{tabular}
\end{center}

\bigskip \noindent Additionally, the $is\_generalization$ and $is\_specialization$ functions are specially implemented for Noisy Popper and check for a version of subsumption as Definitions~\ref{def:generalization} and \ref{def:specialization} outline, i.e., if every clause in $H_1$ is subsumed by some clause in $H_2$, then $H_2$ is a generalization of $H_1$ and $H_1$ is a specialization of $H_2$. Notably, checking subsumption is NP-complete \cite{kapur1986np} and these functions do not check for variable substitutions.

\section{Sound Constraints with Hypothesis Size}
The $learn\_size\_constraints$ function in line 25 of the Noisy Popper algorithm is implemented below in Algorithm 4. Like with Algorithm 3, this algorithm builds a list of hypothesis constraints using hypothesis size by comparing previously generated programs with the newly generated one. We will often refer to these constraints as simply \textit{size constraints}. These constraints are ultimately added to the ASP constraint set $C$ in the Noisy Popper algorithm and used to prune the hypothesis space of particularly large hypotheses.

\begin{algorithm}
\caption{Learn Sound Hypothesis Constraints with Hypothesis Size}\label{alg:size}
\begin{algorithmic}[1]
    \Require program, program$\_$list, $B$, $E^+$, $E^-$, tp, tn (where $B$ is a set of background knowledge, tp = $tp$(program,$B,E^+$) and tn = $tn$(program,$N,E^-$) as calculated in Algorithm 1)
    \Ensure Set of hypothesis constraints (may be empty)
    \State constraints $\gets$ $\{\}$
    \For{p in program$\_$list}
        \State gen$\_$size$_1$ $\gets$ $|E^+| + tn($p,$B,E^-)$ $-$ $S_{MDL}($program,$B,E^+,E^-)$
        \State constraints $\gets$ constraints +
        \State \quad \quad learn$\_$constraints$\_$with$\_$size(p, $|E^+|$, $|E^-|-1$ gen$\_$size$_1$)
        \State spec$\_$size$_1$ $\gets$ $|E^-| + tp$(p,$B,E^+)$ $-$ $S_{MDL}($program,$B,E^+,E^-)$
        \State constraints $\gets$ constraints +
        \State \quad \quad learn$\_$constraints$\_$with$\_$size(p, $|E^+|-1$, $|E^-|$, spec$\_$size$_1$)
    \EndFor
        
    \State gen$\_$size$_2$ $\gets$ $|E^+| - $ tp + $size($program)
    \State constraints $\gets$ constraints +
    \State \quad \quad learn$\_$constraints$\_$with$\_$size(program, $|E^+|$, $|E^-|-1$, gen$\_$size$_2$)
    \State spec$\_$size$_2$ $\gets$ $|E^-| -$ tn + $size($program)
    \State constraints $\gets$ constraints +
    \State \quad \quad learn$\_$constraints$\_$with$\_$size(program, $|E^+|-1$, $|E^-|$, spec$\_$size$_2$)\\
    \Return constraints
\end{algorithmic}
\end{algorithm}

\bigskip \noindent Lines 3-5 correspond to Proposition~\ref{prop:mdl1} and prune all generalizations of a previously seen hypotheses of particular size as they cannot have an $S_{MDL}$ score greater than that of the newly generated program and are thus not $S_{MDL}$-optimal. Likewise, lines 6-8 correspond to Proposition~\ref{prop:mdl2} and prune all specializations of a previously seen hypothesis of particular size as they cannot be $S_{MDL}$-optimal. Note that even in the case where the new hypothesis performs exceptionally poorly, these will still generate hypothesis constraints and remove exceptionally large generalizations and specializations from the hypothesis space. Lines 9-11 correspond to Proposition~\ref{prop:mdl3} and prunes all generalizations of the new hypothesis $H$ with size greater than $fn(H,B,E^+) + size(H)$ as these can never have a greater $S_{MDL}$ score than $H$ and are thus not $S_{MDL}$-optimal. Likewise, lines 12-14 correspond to Proposition~\ref{prop:mdl4} and prunes all specializations of the new hypothesis $H$ with size greater than $fp(H,B,E^-) + size(H)$ as these cannot be $S_{MDL}$-optimal. 

\bigskip \noindent Like with Algorithm 3, Algorithm 4 also introduces a new modified version of the $learn\_constraints$ called $learn\_constraints\_with\_size$ which generates variants of the typical ASP constraints, taking an additional size argument. These ASP constraints only prune hypothesis which have a size greater than the given size argument. A generalization constraint which only prunes hypotheses of particular size would be defined as:

\begin{center}
    \begin{tabular}{l}
         \texttt{generalizationConstraintWithSize($\{$Clause$_1$, Clause$_2$,...,Clause$_n\}$, size) :=}\\
         \quad \texttt{inclusionRule(Clause$_1$),...,}\\
         \quad \texttt{inclusionRule(Clause$_n$)}.\\
         \quad \texttt{:-exactClause(C1$_1$,Clause$_1$),...,}\\
         \quad \quad \texttt{exactClause(C1$_n$,Clause$_n$)},\\
         \quad \quad \texttt{program$\_$size(N), size < N.}
    \end{tabular}
\end{center}

\noindent where \texttt{program$\_$size(N)} holds true only if the number of body literals in the given program equals \texttt{N}. Likewise, an analogous specialization constraint would be defined as:

\begin{center}
    \begin{tabular}{l}
        \texttt{specializationConstraintWithSize($\{$Clause$_1$, Clause$_2$,...,Clause$_n\}$, size) :=}\\
        \quad \texttt{inclusionRule(Clause$_1$),...,}\\
        \quad \texttt{inclusionRule(Clause$_n$)}.\\
        \quad \texttt{:-included$\_$clause(C1$_1$,clauseIdent(Clause$_1$)),...,}\\
        \quad \quad\texttt{included$\_$clause(C1$_n$,clauseIdent(Clause$_n$)),}\\
        \quad \quad \texttt{assertDistinct($\{$Cl$_1$,...,Cl$_n\}$),} \\
        \quad \quad \texttt{not clause(n), program$\_$size(N), size < N.}
    \end{tabular}
\end{center}

\paragraph{Learning from All Previous Hypotheses}
The $learn\_sound\_constraints$ and $learn\_size\_constraints$ algorithms takes a list of programs with which to compare the most recently generated hypothesis with. From line 26 in the Noisy Popper algorithm, we can see that this list is composed of all previously seen hypotheses that the system has generated to that point. While the original motivation behind the sound constraints was to learn from the best hypothesis as it was continuously being maintained, Propositions~\ref{prop:sound1}-\ref{prop:mdl4} all hold when comparing any two hypotheses. Thus, we can generate additional constraints by comparing any new hypothesis with a running list of all previously encountered and scored hypotheses. We motivate this through an example:

\begin{example}[Comparing to All Previous Hypotheses]
    Consider an east-west trains problem with BK $B$ and $|E^+| = |E^-| = 5$. Assume we have already observed the following hypotheses:
    
    \begin{center}
        \begin{tabular}{l}
            \texttt{h$_1$ =} $\left\{\begin{array}{l}
            \texttt{eastbound(A) :- has$\_$car(A,B),short(B).}
            \end{array}\right\}$ \\
        \end{tabular}
        \begin{tabular}{l}
            \texttt{h$_2$ =} $\left\{\begin{array}{l}
            \texttt{eastbound(A) :- has$\_$car(A,B),long(B).}
            \end{array}\right\}$ \\
        \end{tabular}
    \end{center}
    
    \noindent where $tp($\texttt{h$_1$},$B,E^+) = 2$, $tn($\texttt{h$_1$},$B,E^-) = 0$,  $tp($\texttt{h$_2$},$B,E^+) = 0$, and $tn($\texttt{h$_2$},$B,E^-) = 2$. Now, consider the next hypothesis generated is \texttt{h$_3$ = \{eastbound(A) :- has$\_$car(A,B), three$\_$wheels(B).\}} which has $S_{ACC}($\texttt{h$_3$},$B,E^+,E^-)$ = 6. 
    
    \noindent Since $S_{ACC}($\texttt{h$_3$},$B,E^+,E^-) - S_{ACC}($\texttt{h$_1$},$B,E^+,E^-) > |E^+| - tp($\texttt{h$_1$},$B,E^+)$, by Proposition~\ref{prop:sound1} we may prune all generalizations of \texttt{h$_1$}. Similarly, since \\ $S_{ACC}($\texttt{h$_3$},$B,E^+,E^-) - S_{ACC}($\texttt{h$_2$},$B,E^+,E^-) > |E^-| - tn($\texttt{h$_2$},$B,E^-)$, by Proposition~\ref{prop:sound2} we may prune all specializations of \texttt{h$_1$}. If we had not maintained both \texttt{h$_1$} and \texttt{h$_2$}, we would have only been able to identify one of these hypothesis constraints. Thus, the number of hypothesis constraints we can generate can increase as more hypotheses are maintained for comparison. 
\end{example}

\bigskip \noindent Comparing to all previous hypotheses provides the system with noticeable improvement in practice though at the cost of significant inefficiencies as repeated subsumption checks are taxing. Steps are taken in the Noisy Popper implementation to avoid redundant constraint generation as much as possible. Lists are maintained to keep track of programs which have had their generalizations and specializations pruned. Should a program have its generalizations pruned, it is removed from the respective list to ensure it is not checked again and similar actions are taken for specializations. Programs removed from the generalizations list additionally cannot generate any generalization constraints with hypothesis size as these would be similarly redundant and likewise for specialization constraints with hypothesis size. Even with these changes, we still may loop over every previously seen hypothesis with each generate-test-constrain loop of Algorithm 2. Thus, given $N$ hypothesis in the hypothesis space, we may make $O(N^2)$ total hypotheses comparisons in both Algorithm 3 and Algorithm 4. Each comparison may additional check for incomplete subsumption as specified previously which, if we assume each program has at most $C$ clauses each with at most $L$ literals takes $O((CL)^2)$. Thus, the additional code used to modify Normal Popper to Noisy Popper has a worst-case runtime of $O((NCL)^2)$. Again, checking subsumption in full is NP-complete \cite{kapur1986np} and we are using incomplete subsumption checks here.

\section{Noisy Popper Worked Example}
To illustrate how Noisy Popper works, we will consider another east-west trains problem. Again assume we are trying to find the hypothesis \texttt{eastbound(A) :- has$\_$car(A,B), short(B), two$\_$wheels(B).} We will also assume that $t = 0$ is the minimal constraint threshold used and consider only a small initial hypothesis space, $\mathcal{H}_2$:

\begin{center}
    \begin{tabular}{l}
        $\mathcal{H}_2$ =  $\left\{\begin{array}{l}
            \texttt{h$_1$ =} \left\{\begin{array}{l}
            \texttt{eastbound(A) :- has$\_$car(A,B),long(B).}
            \end{array}\right\} \\
            \texttt{h$_2$ =} \left\{\begin{array}{l}
            \texttt{eastbound(A) :- has$\_$car(A,B),long(A),two$\_$wheels(B).}
            \end{array}\right\} \\
            \texttt{h$_3$ =} \left\{\begin{array}{l}
            \texttt{eastbound(A) :- has$\_$car(A,B),short(B).}
            \end{array}\right\} \\
            \texttt{h$_4$ =} \left\{\begin{array}{l}
            \texttt{eastbound(A) :- has$\_$car(A,B),short(B),two$\_$wheels(B).}
            \end{array}\right\} \\
            \texttt{h$_5$ =} \left\{\begin{array}{l}
            \texttt{eastbound(A) :- has$\_$car(A,B),long(B),roof$\_$closed(B).}
            \end{array}\right\} \\
            \texttt{h$_6$ =} \left\{\begin{array}{l}
            \texttt{eastbound(A) :- has$\_$car(A,B),roof$\_$closed(B),three$\_$wheels(B).}\\
            \texttt{eastbound(A) :- has$\_$car(A,B),short(B).}
            \end{array}\right\} \\
            \texttt{h$_7$ =} \left\{\begin{array}{l}
            \texttt{eastbound(A) :- has$\_$car(A,B),short(B),two$\_$wheels(B).}\\
            \texttt{eastbound(A) :- has$\_$car(A,B),roof$\_$closed(B),two$\_$wheels(B).}
            \end{array}\right\} \\
            \texttt{h$_8$ =} \left\{\begin{array}{l}
            \texttt{eastbound(A) :- has$\_$car(A,B),long(B).} \\ 
            \texttt{eastbound(A) :- has$\_$car(A,B),short(B),three$\_$wheels(D,B).} \\ 
            \texttt{eastbound(A) :- has$\_$car(A,B),short(B),two$\_$wheels(B).}
            \end{array}\right\} \\
        \end{array}\right\}$
    \end{tabular}
\end{center}

\bigskip \noindent We will also assume similar sets of examples as in the worked example in Section~\ref{sec:normalworkedexample}, but with the addition of a noisy positive examples \texttt{eastbound(train$_5$).}:

\begin{center}
    \begin{tabular}{l}
        $E^+ = \{$\texttt{eastbound(train$\_$1).}, \texttt{eastbound(train$\_$2).,eastbound(train$_5$).}$\}$\\
        $E^- = \{$\texttt{eastbound(train$\_$3).}, \texttt{eastbound(train$\_$4).}$\}$\\
    \end{tabular}
\end{center}

\noindent where the BK is again the same as in Section~\ref{sec:normalworkedexample} but with the added facts for \texttt{train$_5$}

\begin{center}
    \texttt{has$\_$car(train$_5$, car$_6$)., two$\_$wheels(car$_6$), roof$\_$closed(car$_6$).}\\
\end{center}

\noindent Noisy Popper will first generate hypothesis \texttt{h$_1$ = \{eastbound(A):-has$\_$car(A,B),long(B).\}} from the hypothesis space. Since no train contains a long car, no example is entailed, positive or negative, giving $S_{ACC} = 2$. Being the first hypothesis considered, this is saved as the best hypothesis. Because \texttt{h$_1$} entails no positive examples and no negative examples, the minimal constraints will not alter the outcome from (tp = 0, tn = 2). This means that a specialization constraint and elimination constraint are generated as they are in Section~\ref{sec:normalworkedexample}, pruning all specializations of \texttt{h$_1$}, namely \texttt{h$_2$} and \texttt{h$_5$} as well as all separable hypotheses which contain specializations of \texttt{h$_1$}, namely \texttt{h$_8$} (line 23 of Algorithm 2). 

\bigskip \noindent \noindent Note that the $learn\_sound\_constraints$ function would normally produce an identical specialization constraint as above since $tn($\texttt{h$_1$},$B,E^-)$ $= |E^-|$ (lines 26-27 of Algorithm 3), but this is avoided in implementation as it would be redundant. Similarly, the $learn\_size\_constraints$ function would produce constraints pruning all generalizations of \texttt{h$_1$} with size greater than $|E^+| - tp($\texttt{h$_1$},$B,E^+)$ $+$ $size$(\texttt{h$_1$}) = 5 (lines 10-12 of Algorithm 4) and pruning all specializations of \texttt{h$_1$} with size greater than $|E^-| - tn($\texttt{h$_1$},$B,E^-)$ $+$ $size$(\texttt{h$_1$}) = 2 (lines 13-15 of Algorithm 4), but this last constraint is also avoided in implementation as it is redundant. These functions produce no other constraints as there are no previous programs with which to compare \texttt{h$_1$} to. This leaves the hypothesis space as:

\begin{center}
    \begin{tabular}{l}
        $\mathcal{H}_2$ =  $\left\{\begin{array}{l}
            \texttt{h$_3$ =} \left\{\begin{array}{l}
            \texttt{eastbound(A) :- has$\_$car(A,B),short(B).}
            \end{array}\right\} \\
            \texttt{h$_4$ =} \left\{\begin{array}{l}
            \texttt{eastbound(A) :- has$\_$car(A,B),short(B),two$\_$wheels(B).}
            \end{array}\right\} \\
            \texttt{h$_6$ =} \left\{\begin{array}{l}
            \texttt{eastbound(A) :- has$\_$car(A,B),roof$\_$closed(B),three$\_$wheels(B).}\\
            \texttt{eastbound(A) :- has$\_$car(A,B),short(B).}
            \end{array}\right\} \\
             \texttt{h$_7$ =} \left\{\begin{array}{l}
            \texttt{eastbound(A) :- has$\_$car(A,B),short(B),two$\_$wheels(B).}\\
            \texttt{eastbound(A) :- has$\_$car(A,B),roof$\_$closed(B),two$\_$wheels(B).}
            \end{array}\right\} \\
        \end{array}\right\}$
    \end{tabular}
\end{center}

\bigskip \noindent Noisy Popper will next generate hypothesis \texttt{h$_2$ = \{estbound(A):-has$\_$car(A,B),short(B).\}}. Since \texttt{train$_1$}, \texttt{train$_2$}, and \texttt{train$_3$} contain short cars, two positive and one negative example will be entailed, giving $S_{ACC}$(\texttt{h$_2$},$B,E^+,E^-)$ = 3. Noisy Popper will replace \texttt{h$_1$} with \texttt{h$_2$} as the new best hypothesis. Since the $tp($\texttt{h$_2$},$B,E^+) > 0$ and $tn($\texttt{h$_2$},$B,E^-) > 0$, in Normal Popper the generalizations and specializations of \texttt{h$_2$} would be pruned, including the true solution \texttt{h$_4$}. However, due to the constraint relaxation, the values of $tp$ and $tn$ are changed to $|E^+|$ and $|E^-|$ respectively (lines 17-22 of Algorithm 2), so only a banish constraint is generated by $learn\_constraints$ in line 23 of Algorithm 2. No constraints are generated by $learn\_sound\_constraints$ and $learn\_size\_constraints$ as these are all again redundant or do not effect our hypothesis space. The only program removed is \texttt{h$_2$} through the banish constraint leaving the hypothesis space as:

\begin{center}
    \begin{tabular}{l}
        $\mathcal{H}_2$ =  $\left\{\begin{array}{l}
            \texttt{h$_4$ =} \left\{\begin{array}{l}
            \texttt{eastbound(A) :- has$\_$car(A,B),short(B),two$\_$wheels(B).}
            \end{array}\right\} \\
            \texttt{h$_6$ =} \left\{\begin{array}{l}
            \texttt{eastbound(A) :- has$\_$car(A,B),roof$\_$closed(B),three$\_$wheels(B).}\\
            \texttt{eastbound(A) :- has$\_$car(A,B),short(B).}
            \end{array}\right\} \\
             \texttt{h$_7$ =} \left\{\begin{array}{l}
            \texttt{eastbound(A) :- has$\_$car(A,B),short(B),two$\_$wheels(B).}\\
            \texttt{eastbound(A) :- has$\_$car(A,B),roof$\_$closed(B),two$\_$wheels(B).}
            \end{array}\right\} \\
        \end{array}\right\}$ 
    \end{tabular}
\end{center}

\bigskip \noindent The next hypothesis generated is \texttt{h$_4$ = \{eastbound(A):-has$\_$car(A,B),\\short(B),two$\_$wheels(B).\}} which entails no negative examples and all but the single noisy positive example. This gives $S_{ACC}$(\texttt{h$_4$)} = 4 and thus it will be maintained as the new best hypothesis, though since it does not entail all positive and no negative examples, it is not immediately returned. Again, due to the constraint relaxation, the outcome is values for $tp$ and $tn$ are changed to $|E^+|$ and $|E^-|$ respectively, avoiding pruning all specializations of \texttt{h$_4$} as would be done in Normal Popper (lines 17-22 of Algorithm 2). However, all specializations of \texttt{h$_4$} are pruned regardless as it entails no negative examples (lines 30-31 of Algorithm 3). $learn\_size\_constraints$ will generate a constraint which prunes all generalizations of previously seen hypothesis \texttt{h$_3$} with size greater than $|E^+| + tn$(\texttt{h$_3$},$B,E^-)$ $-$ $S_{MDL}$(\texttt{h$_4$},$B,E^+,E^-) = 4$ (lines 3-5 of Algorithm 4). The generalization with size constraint created would be:

\begin{center}
    \begin{tabular}{l}
        \texttt{included$\_$clause(Cl,id$_1$) :-}\\
        \quad \texttt{head$\_$literal(Cl,eastbound,1,(V0))},\\
        \quad \texttt{body$\_$literal(Cl,has$\_$car,2,(V0,V1))},\\
        \quad \texttt{body$\_$literal(Cl,short,1,(V1))},\\
        \quad \texttt{V0!=V1, V0!=V2, V1!=V2.}\\
        \texttt{:-}\\
        \quad \texttt{included$\_$clause(C10,id$_2$)},\\
        \quad \texttt{clause$\_$size(C10,2),}\\
        \quad \texttt{program$\_$size(N), 4 < N.}
    \end{tabular}
\end{center}

\noindent which would prune \texttt{h$_6$} from the hypothesis space. $learn\_size\_constraints$ would also generate a constraint which prunes all generalizations of \texttt{h$_4$} of size greater than $|E^+| - tp$(\texttt{h$_4$},$B,E^+$) $+$ $size$(\texttt{h$_4$}) = 4 (lines 10-12 of Algorithm 4). The generalization with size constraint generated would be similar to the one above but with one additional literal encoded. This prunes \texttt{h$_7$} from the hypothesis space which is notable as this hypothesis perfectly fits the data, but would overfit as it would entail the noisy positive example. Lastly, \texttt{h$_4$} itself is pruned from the hypothesis space via a banish constraint leaving $\mathcal{H}_2$ empty. Since no solution was found, the best maintained hypothesis is returned instead, meaning that \texttt{h$_4$} is correctly returned.

\section{Summary}
In this chapter, we discussed the implementation details of the Normal Popper system including its generate-test-constraint loop and how it encodes programs and hypothesis constraints into an ASP problem. We then discussed how Noisy Popper modifies Normal Popper to be an anytime algorithm and demonstrated how it implements unsound minimal constraints to prune typically suboptimal hypothesis. Next, we described the algorithm which generates sound hypothesis constraints or \textit{sound constraints} based on Propositions~\ref{prop:sound1}-\ref{prop:sound7} in Chapter 4 including necessary changes to the ASP constraint encodings. We likewise described the algorithm which generates sound hypothesis constraints which take into account hypothesis size or \textit{size constraints} under the MDL scoring based on Propositions~\ref{prop:mdl1}-\ref{prop:mdl4}. Finally, we demonstrated Noisy Popper through a worked example. In the next chapter, we will describe the experiments and results used to compare Noisy Popper to Normal Popper and experiments which compare the effectiveness of the individual components of Noisy Popper. 

\end{chapter}
\begin{chapter}
{Experimental Results}
In this chapter we will empirically explore the capabilities of Noisy Popper. Namely, we aim to determine the validity of the claims made in Chapter 1 which stated that Noisy Popper is more capable of generalizing to noisy data than Normal Popper and that without noise, Noisy Popper still generalizes well, though less efficiently than Normal Popper. To this end, the following experimental questions are formulated:

\bigskip \noindent \textbf{Q1.} How well does Noisy Popper generalize to datasets with varying levels of noise in comparison to Normal Popper?

\bigskip \noindent To answer this question, we compare the two systems directly on several problems commonly found in the literature. We will compare purely their accuracies over various amounts of noise including without noise. Though we will briefly compare the two systems as they are, this is a slightly unfair comparison as Normal Popper will typically return an empty solution in the presence of noise. Thus, for most of the experiments we will enhance Normal Popper as an anytime algorithm as we did for Noisy Popper.

\bigskip \noindent \textbf{Q2.} How inefficient is Noisy Popper in comparison to Normal Popper?

\bigskip \noindent To answer this question, we will again test the two systems against several datasets, this time measuring the time it takes for the systems to complete. 

\bigskip \noindent A natural question regarding the effectiveness of Noisy Popper is to what degree do each enhancement impacts learning, both in speed and accuracy. Thus, the following question should be posed:

\bigskip \noindent \textbf{Q3.} How significantly does each enhancement within Noisy Popper impact its learning capabilities and efficiency. 

\bigskip \noindent To answer this question, we will evaluate Noisy Popper as a whole with versions of Noisy Popper without (i) minimal constraints (ii) sound constraints and (iii) size constraints in addition to a completely relaxed brute force version of Normal Popper, \textit{Enumerate}. The accuracies and completion times of each system will be compared to determine the impact of each in various settings.

\section{Noisy Popper vs. Normal Popper}
The purpose of this first set of experiments is to evaluate how well Noisy Popper generalizes to noisy and noiseless datasets in comparison to Normal Popper, measuring both the predictive accuracy of the systems as well as the time it take both to run. We will evaluate both systems over several diverse problem sets commonly found in the literature: two Michalski's east-west trains problem variants, several program synthesis list transformation problems, and two inductive general game playing (IGGP \cite{cropper2019inductive}) problems. These datsets will also be used for experiments in the following sections.

\subsection{Experiment 1: East-West Trains}
This series of problems consists of learning \texttt{eastbound} target predicates for two variations on Michalski's east-west trains problem as described in Chapter 1 and used throughout this paper. Such problems are easy for a system to overfit and will help determine how effective Noisy Popper is at generalizing to noisy data.

\paragraph{Materials}
Noisy Popper and Normal Popper will be evaluated on two different east-west trains problems with data generated from the following ground truth hypotheses:

\begin{center}
    \begin{tabular}{l}
        \texttt{h$_1$ =} $\left\{\begin{array}{l}
        \texttt{eastbound(A) :-}\\ 
        \texttt{has$\_$car(A,C),long(C),roof$\_$closed(C),has$\_$car(A,B),three$\_$wheels(B).}
        \end{array}\right\}$ \\
        \texttt{h$_2$ =} $\left\{\begin{array}{l}
        \texttt{eastbound(A) :-}\\ \texttt{has$\_$car(A,C),roof$\_$open(C),has$\_$car(A,B),roof$\_$closed(B).}
        \end{array}\right\}$
    \end{tabular}
\end{center}

\noindent Both systems are given identical BK containing the descriptions of all 999 trains via relations \texttt{has$\_$car/2, has$\_$load/2, short/1, long/1, two$\_$wheels/1, three$\_$wheels,} etc.

\bigskip \noindent The language biases for Normal and Noisy Popper restrict hypotheses to at most six unique variables, at most size body literals per clause, and at most three clauses. The systems are given type and directions (i.e., input or output) for arguments of each predicate. The minimal constraint threshold for Noisy Poppper is set to its default $t = 0$. For one experiment, Normal Popper will be run as is with no modification. This will adequately demonstrate Normal Popper's inability to generalize at all to noisy data as it will be unable to find an LFF solution before the given system timeout. Out of fairness, the rest of the experiments here and moving forward will run Normal Popper enhanced as an anytime algorithm which, like Noisy Popper, maintains its best seen hypothesis and returns it if no LFF solution is found. Normal Popper will still generate constraints as normal with this enhancement.

\paragraph{Methods}
For each hypothesis above, 50 positive and 50 negative randomly selected examples will be generated for training while 200 positive and 200 negative randomly selected examples will be generated for testing. Example trains are selected randomly using the two hypotheses from the pool of 999 defined in the BK. Should a system fail to return any hypothesis, we will assume that all examples are entailed giving a default predictive accuracy of 50$\%$ in this instance. A timeout of ten minutes and a limit of 200 generated programs is enforced per task which ensures both systems can learn from the same number of programs. The predictive accuracy and learning times are recorded for each task and each experiment will repeat the task ten times with the means and standard errors plotted. Each experiment will be repeated for training noise levels from 0$\%$ to 40$\%$ in increments of 10$\%$ with 5$\%$ additionally being tested. The test sets will remain noiseless.

\paragraph{Results and Analysis}
Table~\ref{table:1} below shows that when compared to Normal Popper unenhanced by an anytime algorithm, Noisy Popper far exceeds its predictive accuracy. Normal Popper achieving 50$\%$ accuracy indicates that the was unable to find an LFF solution to the task in the 200 program limit allotted and thus is given the default predictive accuracy. This is an expected and uninteresting result as Normal Popper's inability to return non-LFF solutions has already been discussed.

\begin{table}[!ht]
    \centering
    \resizebox{0.6\columnwidth}{!}{%
    \begin{tabular}{c|c|c|c} 
    \toprule
    \textbf{Target Program} & \textbf{Training Noise ($\%$)} & \textbf{Normal Popper} & \textbf{Noisy Popper} \\
    \midrule
    \multirow{4}{*}{\texttt{h$_1$}} & 0 & \textbf{100}$\pm{0}$ & \textbf{100}$\pm{0}$ \\
    & 5 & 50$\pm{0}$ & \textbf{100}$\pm{0}$ \\
    & 10 & 50$\pm{0}$ & \textbf{100}$\pm{0}$ \\
    & 20 & 50$\pm{0}$ & \textbf{100}$\pm{0}$ \\
    \midrule
    \multirow{4}{*}{\texttt{h$_2$}} & 0 & \textbf{100}$\pm{0}$ & \textbf{100}$\pm{0}$ \\
    & 5 & 50$\pm{0}$ & \textbf{100}$\pm{0}$ \\
    & 10 & 50$\pm{0}$ & \textbf{100}$\pm{0}$ \\
    & 20 & 50$\pm{0}$ & \textbf{100}$\pm{0}$ \\
    \midrule
    \end{tabular}}
    \caption{East-West Trains predictive accuracy for programs \texttt{h$_1$} and \texttt{h$_2$} with Normal Popper not enhanced as an anytime algorithm. The error is standard}
    \label{table:1}
\end{table}

\bigskip \noindent Figures~\ref{fig:2} and \ref{fig:3} below compare Normal Popper enhanced as an anytime algorithm with Noisy Popper. In terms of predictive accuracy, Noisy Popper outperforms Normal Popper at all noise levels greater than 0$\%$ for both problems and ties as expected at 0$\%$ noise with perfect predictive accuracy. A McNemar's \cite{lachenbruch2014mcnemar} test on the Noisy and Normal Popper predictive accuracy additionally confirmed the significance at the $p < 0.001$ level for both problems. For most noise levels, Normal Popper typically either overfits the data, returning a hypothesis with several extra clauses, or underfits the data having pruned the correct solution from the hypothesis space early on. Noisy Popper's relaxed setting helps avoid this over pruning. Over 30$\%$ noise, Noisy Popper also begins overfitting the data though still produces higher predictive accuracy than Normal Popper. Normal Popper consistently ran much quicker than Noisy Popper, running in under two seconds regardless of noise level. Noisy Popper's learning time is much more volatile, dependent on the number of constraints Noisy Popper generates. Precise reasons for this inefficiency are discussed in the second section of experiments. However, with no noise, both systems take roughly the same amount of time to learn the correct solution. Overall, these results suggest that the answer to \textbf{Q1} is that Noisy Popper generalizes better than Normal Popper under and generalizes as well as Normal Popper when no noise is present. It also suggests that the answer is \textbf{Q2} is that Noisy Popper is significantly less efficient than Normal Popper and can be at most around 20 times slower than Normal Popper with these datasets.

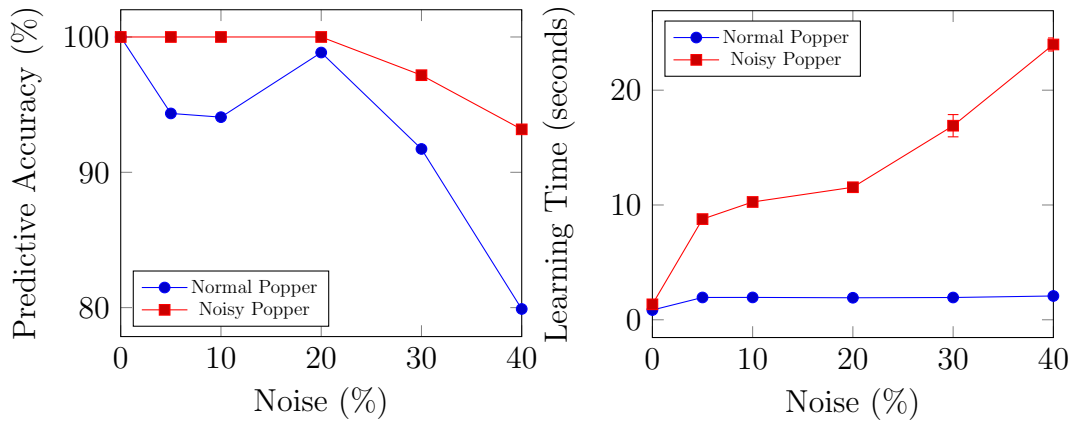
\begin{figure}[ht]
    \begin{subfigure}{.45\linewidth}
    \centering
        \begin{tikzpicture}
        \begin{axis}[
        width=\linewidth,
        legend pos=south west,
        legend style={nodes={scale=0.6, transform shape}},
        xmin=0,
        xmax=40,
        xlabel={Noise ($\%$)},
        ylabel={Predictive Accuracy ($\%$)}]
        
        \addplot+[error bars/.cd,y dir=both,y explicit]
        table [
        x expr=\thisrow{x_data} * 100,
        y=y_data,
        col sep=comma,
        y error plus expr=\thisrow{error},y error minus expr=\thisrow{error},
        ] {data/trains1-acc-vs-noise-normal.csv};
        
        \addplot+[color=red,mark=square*,error bars/.cd,y dir=both,y explicit]
        table [
        x expr=\thisrow{x_data} * 100,
        y=y_data,
        col sep=comma,
        y error plus expr=\thisrow{error},y error minus expr=\thisrow{error},
        ] {data/trains1-acc-vs-noise-noisy.csv};
        
        \legend{{Normal Popper}, {Noisy Popper}}
        \end{axis}
        \end{tikzpicture}
    \end{subfigure}
    \begin{subfigure}{.45\linewidth}
    \centering
        \begin{tikzpicture}
        \begin{axis}[
        width=\linewidth,
        legend pos=north west,
        legend style={nodes={scale=0.6, transform shape}},
        xmin=0,
        xmax=40,
        xlabel={Noise ($\%$)},
        ylabel={Learning Time (seconds)}]
        
        \addplot+[error bars/.cd,y dir=both,y explicit]
        table [
        x expr=\thisrow{x_data} * 100,        y=y_data,
        col sep=comma,
        y error plus expr=\thisrow{error},y error minus expr=\thisrow{error},
        ] {data/trains1-time-vs-noise-normal.csv};
        
        \addplot+[color=red,mark=square*,error bars/.cd,y dir=both,y explicit]
        table [
        x expr=\thisrow{x_data} * 100,
        y=y_data,
        col sep=comma,
        y error plus expr=\thisrow{error},y error minus expr=\thisrow{error},
        ] {data/trains1-time-vs-noise-noisy.csv};
        
        \legend{{Normal Popper}, {Noisy Popper}}
        \end{axis}
        \end{tikzpicture}
    \end{subfigure}
\caption{East-West Trains predictive accuracy and learning time (in seconds) for program \texttt{h$_1$} when varying percentage of noisy training data. Standard error is depicted by bars.}
\label{fig:2}
\end{figure}

\begin{figure}[ht]
    \begin{subfigure}{.45\linewidth}
    \centering
        \begin{tikzpicture}
        \begin{axis}[
        width=\linewidth,
        legend pos=south west,
        legend style={nodes={scale=0.6, transform shape}},
        xmin=0,
        xmax=40,
        xlabel={Noise ($\%$)},
        ylabel={Predictive Accuracy ($\%$)}]
        
        \addplot+[error bars/.cd,y dir=both,y explicit]
        table [
        x expr=\thisrow{x_data} * 100,
        y=y_data,
        col sep=comma,
        y error plus expr=\thisrow{error},y error minus expr=\thisrow{error},
        ] {data/trains2-acc-vs-noise-normal.csv};
        
        \addplot+[color=red,mark=square*,error bars/.cd,y dir=both,y explicit]
        table [
        x expr=\thisrow{x_data} * 100,
        y=y_data,
        col sep=comma,
        y error plus expr=\thisrow{error},y error minus expr=\thisrow{error},
        ] {data/trains2-acc-vs-noise-noisy.csv};
        
        \legend{{Normal Popper}, {Noisy Popper}}
        \end{axis}
        \end{tikzpicture}
    \end{subfigure}
    \begin{subfigure}{.45\linewidth}
    \centering
        \begin{tikzpicture}
        \begin{axis}[
        width=\linewidth,
        legend pos=north west,
        legend style={nodes={scale=0.6, transform shape}},
        xmin=0,
        xmax=40,
        xlabel={Noise ($\%$)},
        ylabel={Learning Time (seconds)}]
        
        \addplot+[error bars/.cd,y dir=both,y explicit]
        table [
        x expr=\thisrow{x_data} * 100,
        y=y_data,
        col sep=comma,
        y error plus expr=\thisrow{error},y error minus expr=\thisrow{error},
        ] {data/trains2-time-vs-noise-normal.csv};
        
        \addplot+[color=red,mark=square*,error bars/.cd,y dir=both,y explicit]
        table [
        x expr=\thisrow{x_data} * 100,
        y=y_data,
        col sep=comma,
        y error plus expr=\thisrow{error},y error minus expr=\thisrow{error},
        ] {data/trains2-time-vs-noise-noisy.csv};
        
        \legend{{Normal Popper}, {Noisy Popper}}
        \end{axis}
        \end{tikzpicture}
    \end{subfigure}
    \caption{East-West Trains predictive accuracy and learning time (in seconds) for program \texttt{h$_2$} when varying percentage of noisy training data. Standard error is depicted by bars.}
\label{fig:3}
\end{figure}
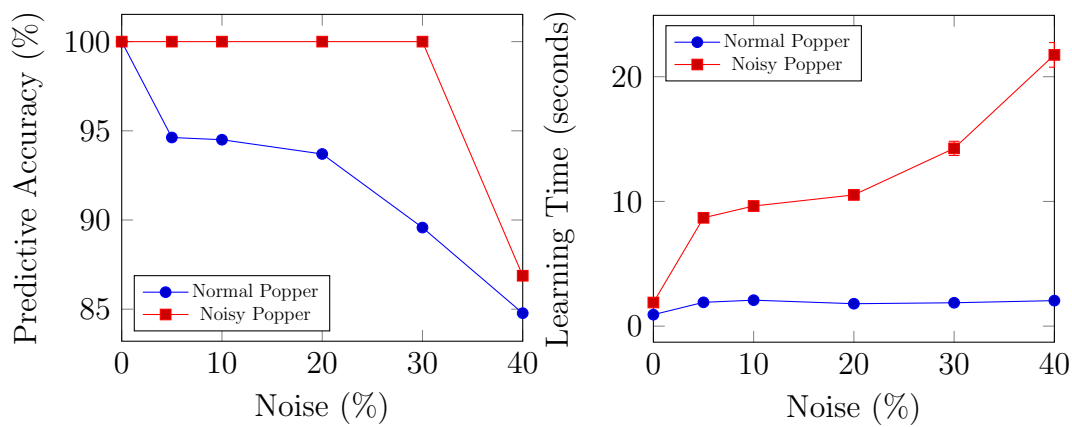

\subsection{Experiment 2: List Manipulations}
This series of problems consists of learning target predicates which manipulate or check certain properties of number lists and serves as an example of program synthesis. These problems are difficult typically requiring recursive solutions. Learning recursive programs has often been considered a difficult though important task for ILP systems \cite{cropper2020turning}. 

\paragraph{Materials}
Noisy Popper and Normal Popper will be evaluated on the nine list manipulation problems seen below in Table~\ref{table:2} from \cite{popper}. These have been shown to be challenging for many ILP systems unless strong inductive biases are provided with the exception of Normal Popper which has demonstrated near perfect accuracy on each task in a noiseless setting \cite{popper}.

\begin{table}[ht!]
    \centering
    \resizebox{\columnwidth}{!}{%
    \begin{tabular}{l|l|l} 
    \toprule
    \textbf{Name} & \textbf{Description} & \textbf{Example Solution}
    \\
    \midrule
    \texttt{addhead} & Prepend head of list three times & \texttt{addhead(A,B):-head(A,C),cons(C,A,D),cons(C,D,E),cons(C,E,B).}
    \\
    \midrule
    \texttt{droplast} & Drop the last element of the list & 
    \begin{tabular}{@{}l@{}} \texttt{droplast(A,B):-tail(A,B),empty(B).} \\ \texttt{droplast(A,B):-tail(A,C),droplast(C,D),head(A,E),cons(E,D,B).} \end{tabular}
    \\
    \midrule
    \texttt{evens} & Check all elements are even &
    \begin{tabular}{@{}l@{}} \texttt{evens(A):-empty(A).} \\ \texttt{evens(A):-head(A,B),even(B),tail(A,C),even(C)} \end{tabular}
    \\
    \midrule
    \texttt{finddup} & Find duplicate elements &
    \begin{tabular}{@{}l@{}} \texttt{finddup(A,B):-head(A,B),tail(A,C),member(B,C).} \\ \texttt{finddup(A,B):-tail(A,C),finddup(C,B).} \end{tabular}
    \\
    \midrule
    \texttt{last} & Last element of list & 
    \begin{tabular}{@{}l@{}} \texttt{last (A,B):-tail(A,C),empty(C),head(A,B).} \\ \texttt{last (A,B):-tail(A,C),last (C,B).} \end{tabular}
    \\
    \midrule
    \texttt{len} & Calculates list length &
    \begin{tabular}{@{}l@{}} \texttt{len(A,B):-empty(A),zero(B).} \\ \texttt{len(A,B):-tail(A,C),len(C,D),increment(D,B).} \end{tabular}
    \\
    \midrule
    \texttt{member} & Member of the list & 
    \begin{tabular}{@{}l@{}} \texttt{member(A,B):-head(A,B).} \\ \texttt{member(A,B):-tail(A,C),member(C,B).} \end{tabular}
    \\
    \midrule
    \texttt{sorted} & Checks if list is sorted &
    \begin{tabular}{@{}l@{}} \texttt{sorted(A):-tail(A,B),empty(B).} \\ \texttt{sorted(A):-head(A,B),tail(A,C),head(C,D),geq(D,B),sorted(C).} \end{tabular}
    \\
    \midrule
    \texttt{threesame} & First three elements are identical &
    \texttt{threesame(A):-head(A,B),tail(A,C),head(C,B),tail(C,D),head(D,B).}
    \\
    \midrule
    \end{tabular}}
    \caption{List manipulation problems with descriptions and example solutions \cite{popper}.}
    \label{table:2}
\end{table}

\bigskip \noindent Both systems are given identical BK containing some of the monadic (i.e. one argument) relations \texttt{empty, even, odd, one,} and \texttt{zero}, dyadic (i.e., two argument) relations \texttt{decrement, head, geq, increment, member} and \texttt{tail}, and triadic (i.e., three argument) relations \texttt{append} and \texttt{prepend} in order to construct solutions.

\bigskip \noindent The language biases for Normal and Noisy Popper restrict hypotheses to at most five unique variables, at most five body literals per clause, and at most two clauses. The systems are again given type and directions for arguments of each predicate as well as a timeout to prevent non-terminating programs from running infinitely. The minimal constraint threshold for Noisy Popper is set to its default $t = 0$. Normal Popper will also be enhanced with an anytime algorithm approach as is done in Experiment 1 above.

\paragraph{Methods}
For each task, 20 positive and 20 negative randomly generated examples are used for training while 1000 positive and 1000 negative randomly generated examples are used for testing. List elements are sampled uniformly from the set $\{1, 2, ..., 100\}$. A timeout of ten minutes and a limit of 500 generated programs is enforced per task, even if this prevents a system from finding a more accurate solution. The predictive accuracy and learning times are recorded for each task and each experiment will repeat the task ten times with the means and standard errors plotted. Each experiment will be repeated for training noise levels of 0$\%$, 5$\%$, 10$\%$, and 20$\%$, though the test sets will remain noiseless. 

\paragraph{Results and Analysis}
Tables~\ref{table:3} and \ref{table:4} below depict the predictive accuracy and learning times respectively for each of the list manipulation tasks with training data noise levels of 0$\%$, 5$\%$, 10$\%$, and 20$\%$. For most problems, both systems were able to find correct solutions with the exceptions of \texttt{finddup} in which both systems struggled and \texttt{member} in which Normal Popper struggled with added noise. This indicates that even Normal Popper enhanced with an anytime algorithm is capable of generalizing well to noisy data. This makes cases where Noisy Popper outperforms Normal Popper notable. Both systems however can struggle on particular problems or datasets, though it is possible both systems could have found correct solutions to the \texttt{finddup} task if given additional time. 

\bigskip \noindent Normal Popper consistently ran significantly faster than Normal Popper. This is due to the large number of constraints Noisy Popper generates which the ASP solver must use and the number of programs compared to generate these constraints. As discussed in Chapter 5, this number of comparisons is quadratic in total number of hypotheses generated. For a single hypothesis generated by Normal Popper, the system may generate at most two constraints whereas Noisy Popper can generate a multitude of constraints for every program already seen by the system. While the Noisy Popper implementation attempts to mitigate redundant constraints, this clearly produces a large bottleneck for the system also effected by a large grounding issue discussed in the following experimentation section. This data indicates than an answer to \textbf{Q1} is that Noisy Popper generalizes as well as Normal Popper for many noisy datasets, though typically never performs worse. An answer to \textbf{Q2} is again that Noisy Popper is much more inefficient than Normal Popper and may in fact be unusable in certain cases due to its extreme inefficiencies. 

\begin{table}[p]
    \centering
    \resizebox{0.75\columnwidth}{!}{%
    \begin{tabular}{c|c|c|c} 
    \toprule
    \textbf{Name} & \textbf{Training Noise ($\%$)} & \textbf{Normal Popper} & \textbf{Noisy Popper} \\
    \midrule
    \multirow{4}{*}{\texttt{addhead}} & 0 & \textbf{100}$\pm{0}$ & \textbf{100}$\pm{0}$ \\
    & 5 & \textbf{100}$\pm{0}$ & \textbf{100}$\pm{0}$ \\
    & 10 & \textbf{100}$\pm{0}$ & \textbf{100}$\pm{0}$ \\
    & 20 & \textbf{100}$\pm{0}$ & \textbf{100}$\pm{0}$ \\
    \midrule
    \multirow{4}{*}{\texttt{droplast}} & 0 & \textbf{100}$\pm{0}$ & \textbf{100}$\pm{0}$ \\
    & 5 & \textbf{100}$\pm{0}$ & \textbf{100}$\pm{0}$ \\
    & 10 & \textbf{100}$\pm{0}$ & \textbf{100}$\pm{0}$ \\
    & 20 & \textbf{100}$\pm{0}$ & \textbf{100}$\pm{0}$ \\
    \midrule
    \multirow{4}{*}{\texttt{evens}} & 0 & \textbf{100}$\pm{0}$ & \textbf{100}$\pm{0}$ \\
    & 5 & \textbf{100}$\pm{0}$ & \textbf{100}$\pm{0}$ \\
    & 10 & \textbf{100}$\pm{0}$ & \textbf{100}$\pm{0}$ \\
    & 20 & 95$\pm{0}$ & \textbf{99}$\pm{0}$ \\
    \midrule
    \multirow{4}{*}{\texttt{finddup}} & 0 & \textbf{55}$\pm{0}$ & 54$\pm{0}$ \\
    & 5 & \textbf{54}$\pm{0}$ & 52$\pm{0}$ \\
    & 10 & 52$\pm{0}$ & \textbf{53}$\pm{0}$ \\
    & 20 & 51$\pm{0}$ & \textbf{53}$\pm{0}$ \\
    \midrule
    \multirow{4}{*}{\texttt{last}} & 0 & \textbf{100}$\pm{0}$ & \textbf{100}$\pm{0}$ \\
    & 5 & \textbf{100}$\pm{0}$ & \textbf{100}$\pm{0}$ \\
    & 10 & \textbf{100}$\pm{0}$ & \textbf{100}$\pm{0}$ \\
    & 20 & \textbf{100}$\pm{0}$ & \textbf{100}$\pm{0}$ \\
    \midrule
    \multirow{4}{*}{\texttt{len}} & 0 & \textbf{100}$\pm{0}$ & \textbf{100}$\pm{0}$ \\
    & 5 & \textbf{100}$\pm{0}$ & \textbf{100}$\pm{0}$ \\
    & 10 & \textbf{100}$\pm{0}$ & \textbf{100}$\pm{0}$ \\
    & 20 & \textbf{100}$\pm{0}$ & \textbf{100}$\pm{0}$ \\
    \midrule
    \multirow{4}{*}{\texttt{member}} & 0 & \textbf{100}$\pm{0}$ & \textbf{100}$\pm{0}$ \\
    & 5 & 97$\pm{0}$ & \textbf{100}$\pm{0}$ \\
    & 10 & 96$\pm{0}$ & \textbf{100}$\pm{0}$ \\
    & 20 & 86$\pm{0}$ & \textbf{100}$\pm{0}$ \\
    \midrule
    \multirow{4}{*}{\texttt{sorted}} & 0 & \textbf{100}$\pm{0}$ & \textbf{100}$\pm{0}$ \\
    & 5 & \textbf{100}$\pm{0}$ & \textbf{100}$\pm{0}$ \\
    & 10 & \textbf{100}$\pm{0}$ & \textbf{100}$\pm{0}$ \\
    & 20 & \textbf{100}$\pm{0}$ & \textbf{100}$\pm{0}$ \\
    \midrule
    \multirow{4}{*}{\texttt{threesame}} & 0 & \textbf{100}$\pm{0}$ & \textbf{100}$\pm{0}$ \\
    & 5 & \textbf{100}$\pm{0}$ & \textbf{100}$\pm{0}$ \\
    & 10 & \textbf{99}$\pm{0}$ & \textbf{99}$\pm{0}$ \\
    & 20 & \textbf{99}$\pm{0}$ & \textbf{99}$\pm{0}$ \\
    \midrule
    \end{tabular}}
    \caption{Predictive accuracy for Normal and Noisy Popper on list manipulation problems. Accuracies are rounded to the nearest integer and errors to the nearest tenth. Errors are standard.}
    \label{table:3}
\end{table}

\begin{table}[p]
    \centering
    \resizebox{0.75\columnwidth}{!}{%
    \begin{tabular}{c|c|c|c}
    \toprule
    \textbf{Name} & \textbf{Training Noise ($\%$)} & \textbf{Normal Popper} & \textbf{Noisy Popper} \\
    \midrule
    \multirow{4}{*}{\texttt{addhead}} & 0 & \textbf{0.6}$\pm{0}$ & 2$\pm{0.1}$ \\
    & 5 & \textbf{12}$\pm{3}$ & 79$\pm{4}$ \\
    & 10 & \textbf{9}$\pm{2}$ & 74$\pm{2}$ \\
    & 20 & \textbf{8}$\pm{3}$ & 74$\pm{3}$ \\
    \midrule
    \multirow{4}{*}{\texttt{droplast}} & 0 & \textbf{34}$\pm{8}$ & 81$\pm{39}$ \\
    & 5 & \textbf{78}$\pm{9}$ & 135$\pm{45}$ \\
    & 10 & \textbf{80}$\pm{10}$ & 142$\pm{33}$ \\
    & 20 & \textbf{79}$\pm{8}$ & 137$\pm{36}$ \\
    \midrule
    \multirow{4}{*}{\texttt{evens}} & 0 & \textbf{2}$\pm{0}$ & 7$\pm{0.5}$ \\
    & 5 & \textbf{13}$\pm{3}$ & 38$\pm{1}$ \\
    & 10 & \textbf{14}$\pm{1}$ & 45$\pm{0.5}$ \\
    & 20 & \textbf{13}$\pm{1}$ & 40$\pm{3}$ \\
    \midrule
    \multirow{4}{*}{\texttt{finddup}} & 0 & \textbf{8}$\pm{1}$ & 39$\pm{2}$ \\
    & 5 & \textbf{7}$\pm{0.9}$ & 36$\pm{1}$ \\
    & 10 & \textbf{8}$\pm{0.9}$ & 39$\pm{2}$ \\
    & 20 & \textbf{9}$\pm{1}$ & 40$\pm{2}$ \\
    \midrule
    \multirow{4}{*}{\texttt{last}} & 0 & \textbf{1}$\pm{0.4}$ & 15$\pm{4}$ \\
    & 5 & \textbf{3}$\pm{0.4}$ & 19$\pm{0.5}$ \\
    & 10 & \textbf{3}$\pm{0.4}$ & 21$\pm{0.3}$ \\
    & 20 & \textbf{3}$\pm{0.4}$ & 20$\pm{0.3}$ \\
    \midrule
    \multirow{4}{*}{\texttt{len}} & 0 & \textbf{0.5}$\pm{0}$ & 2$\pm{0.1}$ \\
    & 5 & \textbf{3}$\pm{0.4}$ & 59$\pm{6}$ \\
    & 10 & \textbf{2}$\pm{0.2}$ & 56$\pm{2}$ \\
    & 20 & \textbf{2}$\pm{0.1}$ & 56$\pm{2}$ \\
    \midrule
    \multirow{4}{*}{\texttt{member}} & 0 & \textbf{0.4}$\pm{0}$ & 0.7$\pm{0}$ \\
    & 5 & \textbf{10}$\pm{3}$ & 23$\pm{3}$ \\
    & 10 & \textbf{18}$\pm{6}$ & 23$\pm{3}$ \\
    & 20 & \textbf{20}$\pm{5}$ & 23$\pm{3}$ \\
    \midrule
    \multirow{4}{*}{\texttt{sorted}} & 0 & \textbf{4}$\pm{0.4}$ & 26$\pm{3}$ \\
    & 5 & \textbf{8}$\pm{0.4}$ & 42$\pm{0.1}$ \\
    & 10 & \textbf{8}$\pm{0.4}$ & 43$\pm{1}$ \\
    & 20 & \textbf{8}$\pm{0.4}$ & 43$\pm{0.2}$ \\
    \midrule
    \multirow{4}{*}{\texttt{threesame}} & 0 & \textbf{0.3}$\pm{0.1}$ & 0.5$\pm{0.1}$ \\
    & 5 & \textbf{0.8}$\pm{0.1}$ & 2$\pm{0.1}$ \\
    & 10 & \textbf{1}$\pm{0.1}$ & 3$\pm{0.3}$ \\
    & 20 & \textbf{1}$\pm{0.1}$ & 4$\pm{0.2}$ \\
    \midrule
    \end{tabular}}
    \caption{Learning times for Normal and Noisy Popper on list manipulation problems. Times are rounded to the nearest second if they are greater than 1 second and to the tenth otherwise. Errors are standard.}
    \label{table:4}
\end{table}

\subsection{Experiment 3: IGGP Problems}
The general game playing (GGP) competition \cite{genesereth2013international} measures a system's general intelligence by having giving the agent the rules to several new games described as logic programs before having the agent play each game. The competition winner is the agent which scores the best total over all games. The inductive general game playing (IGGP) \cite{cropper2019inductive} task inverts the GGP task, providing a system with logical traces of a game in order for the system to try and learn the rules of the game. These experiments focus on the \textit{minimal decay} and \textit{rock, paper, scissors} (rps) tasks, aiming to learn the target predicate \texttt{next$\_$score} which determines the score a player will have given their action and the action of the other player on a given turn.

\paragraph{Materials}
Noisy Popper and Normal Popper will be evaluated on the IGGP minimal decay and rps tasks. The BK for each system will contain facts about particular gameplay traces, i.e., specific actions players took on each turn, the actual score of a player after a turn has completed, etc. Some gameplay rules are also provided in the rps BK such as which action beats which other, i.e., rock beats scissors. 

\bigskip \noindent The language biases for both systems restrict hypotheses to at most five unique variables, at most five body literals per clause, and at most two clauses for the minimal decay task and at most seven unique variable, at most six body literals per clause, and at most six clauses for the rps task. Again, types and directions are given for predicate arguments. The minimal constraint threshold for Noisy Popper is set to its default $t = 0$. Normal Popper will be enhanced with an anytime algorithm approach as is done in Experiment 1 above.

\paragraph{Methods}
For the minimal decay task, 5 positive and 20 negative randomly generated examples are used for training while 5 positive and 30 negative randomly generated examples are used of testing. This allows us to observe how well each system generalizes when the number of positive and negative training examples are not equal. A timeout of ten minutes and a limit of 150 generated programs is enforced per task. For the RPS task, 20 positive and 50 negative examples are randomly generated to train on and 100 positive and 200 negative are generated for testing. A timeout of ten minutes and a limit of 200 generated programs is enforced per task. The predictive accuracy and learning times are recorded for each task and each experiment will be repeat the task ten times with the means and standard errors plotted. Each experiment will be repeated for training noise levels from 0$\%$ to 40$\%$ in increments of 10$\%$ with 5$\%$ noise additionally being tested. Testing sets will remain noiseless.

\paragraph{Results and Analysis}
Figures~\ref{fig:4} and \ref{fig:5} show the predictive accuracies and runtimes of Normal and Noisy Popper on the IGGP minimal decay and rps tasks respectively. These are notably difficult problems and even with no noise, neither system could generate a correct solution for either problem. For minimal decay, this is largely attributed to the exceptionally small number of examples used to train. Both systems typically achieved 100$\%$ training accuracy on the noiseless minimal decay data but did not achieve 100$\%$ testing accuracy. For RPS, the small program limit of 200 was the contributing factor for suboptimal noiseless accuracies though larger program limits often led to blowup in the Noisy Popper runtime. Despite the low number of programs both systems could learn from, Noisy Popper typically achieved equal or greater predictive accuracy than Normal Popper for all noise levels, though by only a slim margin. A McNemar's \cite{lachenbruch2014mcnemar} test on the Noisy and Normal Popper predictive accuracy additionally confirmed the significance at the $p < 0.001$ level for both problems indicating that the performances differences were not random. The results for both tasks suggest that the answer to \textbf{Q1} is that Noisy Popper can generalize better to noisy data than Normal Popper and as well to noiseless data, though the difference is often marginal and Normal Popper enhanced with an anytime algorithm can often generalize well on its own. Additionally, they suggest the answer to \textbf{Q2} is that Noisy Popper is significantly less efficient than Normal Popper even with extremely small program limits, an issue for problems which require learning large programs.

\begin{figure}[ht]
    \begin{subfigure}{.45\linewidth}
    \centering
        \begin{tikzpicture}
        \begin{axis}[
        width=\linewidth,
        legend pos=north east,
        legend style={nodes={scale=0.5, transform shape}},
        xmin=0,
        xmax=40,
        xlabel={Noise ($\%$)},
        ylabel={Predictive Accuracy ($\%$)}]
        
        \addplot+[error bars/.cd,y dir=both,y explicit]
        table [
        x expr=\thisrow{x_data} * 100,
        y=y_data,
        col sep=comma,
        y error plus expr=\thisrow{error},y error minus expr=\thisrow{error},
        ] {data/minimal-decay-acc-vs-noise-normal.csv};
        
        \addplot+[color=red,mark=square*,error bars/.cd,y dir=both,y explicit]
        table [
        x expr=\thisrow{x_data} * 100,
        y=y_data,
        col sep=comma,
        y error plus expr=\thisrow{error},y error minus expr=\thisrow{error},
        ] {data/minimal-decay-acc-vs-noise-noisy.csv};
        
        \legend{{Normal Popper}, {Noisy Popper}}
        \end{axis}
        \end{tikzpicture}
    \end{subfigure}
    \begin{subfigure}{.45\linewidth}
    \centering
        \begin{tikzpicture}
        \begin{axis}[
        width=\linewidth,
        legend style={nodes={scale=0.5, transform shape},at={(0.975,0.5)},anchor=east},
        xmin=0,
        xmax=40,
        xlabel={Noise ($\%$)},
        ylabel={Time (seconds)}]
        
        \addplot+[error bars/.cd,y dir=both,y explicit]
        table [
        x expr=\thisrow{x_data} * 100,
        y=y_data,
        col sep=comma,
        y error plus expr=\thisrow{error},y error minus expr=\thisrow{error},
        ] {data/minimal-decay-time-vs-noise-normal.csv};
        
        \addplot+[color=red,mark=square*,error bars/.cd,y dir=both,y explicit]
        table [
        x expr=\thisrow{x_data} * 100,
        y=y_data,
        col sep=comma,
        y error plus expr=\thisrow{error},y error minus expr=\thisrow{error},
        ] {data/minimal-decay-time-vs-noise-noisy.csv};
        
        \legend{{Normal Popper}, {Noisy Popper}}
        \end{axis}
        \end{tikzpicture}
    \end{subfigure}
\caption{IGGP Minimal Decay task predictive accuracy and time when varying percentage of noisy training data. Standard error is depicted by bars.}
\label{fig:4}
\end{figure}
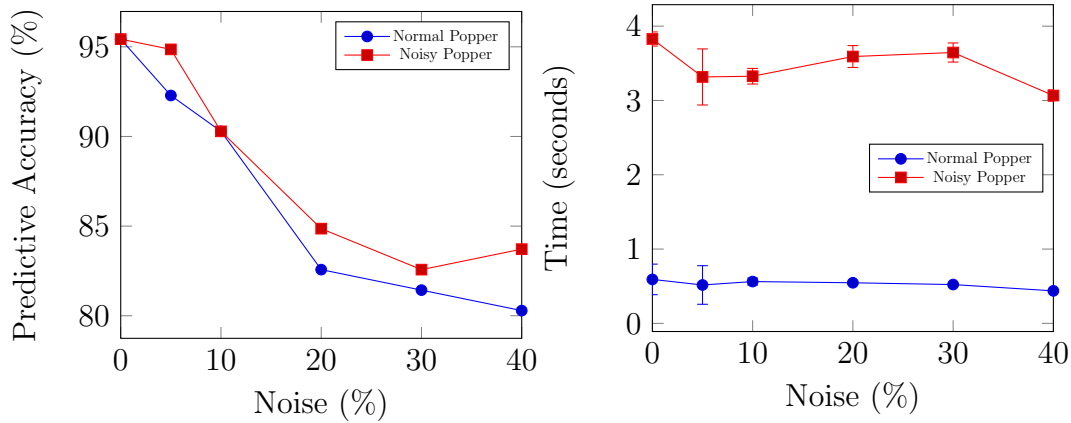

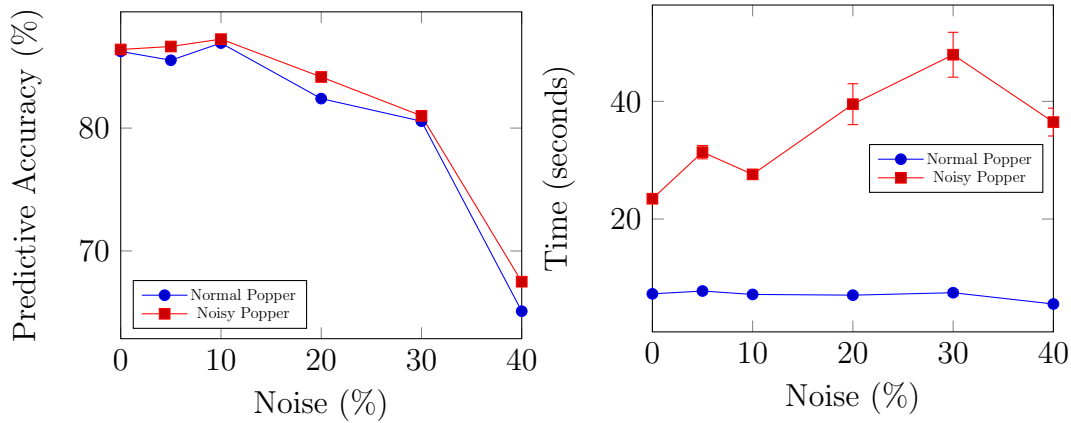
\begin{figure}[ht]
    \begin{subfigure}{.45\linewidth}
    \centering
        \begin{tikzpicture}
        \begin{axis}[
        width=\linewidth,
        legend pos=south west,
        legend style={nodes={scale=0.5, transform shape}},
        xmin=0,
        xmax=40,
        xlabel={Noise ($\%$)},
        ylabel={Predictive Accuracy ($\%$)}]
        
        \addplot+[error bars/.cd,y dir=both,y explicit]
        table [
        x expr=\thisrow{x_data} * 100,
        y=y_data,
        col sep=comma,
        y error plus expr=\thisrow{error},y error minus expr=\thisrow{error},
        ] {data/rps-acc-vs-noise-normal.csv};
        
        \addplot+[color=red,mark=square*,error bars/.cd,y dir=both,y explicit]
        table [
        x expr=\thisrow{x_data} * 100,
        y=y_data,
        col sep=comma,
        y error plus expr=\thisrow{error},y error minus expr=\thisrow{error},
        ] {data/rps-acc-vs-noise-noisy.csv};
        
        \legend{{Normal Popper}, {Noisy Popper}}
        \end{axis}
        \end{tikzpicture}
    \end{subfigure}
    \begin{subfigure}{.45\linewidth}
    \centering
        \begin{tikzpicture}
        \begin{axis}[
        width=\linewidth,
        legend style={nodes={scale=0.5, transform shape},at={(0.975,0.5)},anchor=east},
        xmin=0,
        xmax=40,
        xlabel={Noise ($\%$)},
        ylabel={Time (seconds)}]
        
        \addplot+[error bars/.cd,y dir=both,y explicit]
        table [
        x expr=\thisrow{x_data} * 100,
        y=y_data,
        col sep=comma,
        y error plus expr=\thisrow{error},y error minus expr=\thisrow{error},
        ] {data/rps-time-vs-noise-normal.csv};
        
        \addplot+[color=red,mark=square*,error bars/.cd,y dir=both,y explicit]
        table [
        x expr=\thisrow{x_data} * 100,
        y=y_data,
        col sep=comma,
        y error plus expr=\thisrow{error},y error minus expr=\thisrow{error},
        ] {data/rps-time-vs-noise-noisy.csv};
        
        \legend{{Normal Popper}, {Noisy Popper}}
        \end{axis}
        \end{tikzpicture}
    \end{subfigure}
\caption{IGGP RPS task predictive accuracy and time when varying percentage of noisy training data. Standard error is depicted by bars.}
\label{fig:5}
\end{figure}

\section{Noisy Popper Enhancements}
The purpose of this set of experiments is to determine how effective each enhancement of Noisy Popper is in aiding the overall system. In each of the following experiments, the following variants of Noisy Popper will be run against one another: (i) brute force method described in Section 4.2 which creates no hypothesis constraints (we will refer to this variant as \textit{Enumerate}, (ii) Noisy Popper in its entirety, (iii) Noisy Popper without minimal constraints (which will be labelled as \textit{w/o minimal}), (iv) Noisy Popper without sound constraints (which will be labelled as \textit{w/o sound}), (v) Noisy Popper without size constraints (which will be labelled \textit{w/o size}).

\subsection{Experiment 1: East-West Trains}
This set of experiments is identical to those in Section 6.1.1  using the same east-west trains problems.

\paragraph{Materials}
Each version of Noisy Popper will be evaluated using the same hypotheses as in Section 6.1.1. The language biases and BKs remain the same.

\paragraph{Methods}
The methods are the same as in Section 6.1.1.

\paragraph{Results and Analysis}
Figure~\ref{fig:6} shows that all variants of Noisy Popper achieve similar predictive accuracy for all noise levels which are higher than Enumerate's predictive accuracy. Noisy Popper without minimal constraints performs slightly better with 40$\%$ noise indicating that the minimal constraints used by the other systems typically pruned a highly accurate hypothesis. However, Enumerate did not perform as well despite pruning no hypotheses indicating that Noisy Popper without minimal constraints was only able to find its best hypothesis due the remaining constraints it generated from the other enhancements, i.e. without the additional pruning, it would not have run long enough to find the best solution it did. Figure~\ref{fig:7} shows that the predictive accuracies for all systems were roughly the same for all noise levels, but this may be attributed to \texttt{h$_2$} being an easier hypothesis to find. Here at 40$\%$ noise, Noisy Popper without minimal constraints and without sound hypothesis constraints both perform slightly better, most likely due to the other Popper variants overfitting the data and pruning too much.

\bigskip \noindent The learning times from Figures~\ref{fig:6} and \ref{fig:7} demonstrate that Noisy Popper actually gains some speedup from its minimal and sound constraints. This is likely because when any hypotheses have generalizations or specializations pruned, those programs are essentially forgotten by the system and generate no further constraints. Notably, they no longer generate size constraints which is where the greatest bottleneck is as evidenced by the fast runtime of Noisy Popper without sound size hypothesis constraints. These size constraints create a blowup in the grounding as all possible program sizes defined in the ASP constraints must be grounded and the size constraints specify a large range of sizes in the ASP constraints. This can mean thousands of individual programs are required to be grounded by just a single ASP constraint. In practice, the vast majority of constraints generated by Noisy Popper are these size constraints leading to the inefficiency of the system. This data suggests an answer to \textbf{Q3} is that none of the enhancements individually provide significant benefits to the predictive accuracy of the system, but in conjunction can make the system better than brute force enumeration. Minimal constraints however provide significant speedup to the system and the sound hypothesis constraints additionally contribute to this improvement. The size constraints are the biggest bottleneck however providing little benefit in return when run for such short durations. It is possible that these size constraints would prevent the system from overfitting when run for extended periods, but their inefficiencies make running the system for too long infeasible. Further improvements to the system and testing would be needed to draw conclusions from this hypothesis.

\begin{figure}[ht]
    \begin{subfigure}{.5\linewidth}
    \centering
        \begin{tikzpicture}
        \begin{axis}[
        width=\linewidth,
        xmin=0,
        xmax=40,
        legend pos=south west,
        legend style={nodes={scale=0.5, transform shape}},
        xlabel={Noise ($\%$)},
        ylabel={Predictive Accuracy ($\%$)}]
        
        \addplot+[error bars/.cd,y dir=both,y explicit]
        table [
        x expr=\thisrow{x_data} * 100,
        y=y_data,
        col sep=comma,
        y error plus expr=\thisrow{error},y error minus expr=\thisrow{error},
        ] {data2/trains1-acc-vs-noise-enum.csv};
        
        \addplot+[color=red,mark=square*,error bars/.cd,y dir=both,y explicit]
        table [
        x expr=\thisrow{x_data} * 100,
        y=y_data,
        col sep=comma,
        y error plus expr=\thisrow{error},y error minus expr=\thisrow{error},
        ] {data2/trains1-acc-vs-noise-noisy.csv};
        
        \addplot+[error bars/.cd,y dir=both,y explicit]
        table [
        x expr=\thisrow{x_data} * 100,
        y=y_data,
        col sep=comma,
        y error plus expr=\thisrow{error},y error minus expr=\thisrow{error},
        ] {data2/trains1-acc-vs-noise-no-min.csv};
        
        \addplot+[error bars/.cd,y dir=both,y explicit]
        table [
        x expr=\thisrow{x_data} * 100,
        y=y_data,
        col sep=comma,
        y error plus expr=\thisrow{error},y error minus expr=\thisrow{error},
        ] {data2/trains1-acc-vs-noise-no-learn.csv};
        
        \addplot+[error bars/.cd,y dir=both,y explicit]
        table [
        x expr=\thisrow{x_data} * 100,
        y=y_data,
        col sep=comma,
        y error plus expr=\thisrow{error},y error minus expr=\thisrow{error},
        ] {data2/trains1-acc-vs-noise-no-size.csv};
        
        \legend{{Enumerate}, {Noisy Popper}, {w/o Min. Cons.}, {w/o Sound Cons.}, {w/o Size Cons.}}
        \end{axis}
        \end{tikzpicture}
    \end{subfigure}
    \begin{subfigure}{.5\linewidth}
    \centering
        \begin{tikzpicture}
        \begin{axis}[
        xmin=0,
        xmax=40,
        width=\linewidth,
        legend pos=north west,
        legend style={nodes={scale=0.3, transform shape}},
        xlabel={Noise ($\%$)},
        ylabel={Learning Time (seconds)}]
        
        \addplot+[error bars/.cd,y dir=both,y explicit]
        table [
        x expr=\thisrow{x_data} * 100,
        y=y_data,
        col sep=comma,
        y error plus expr=\thisrow{error},y error minus expr=\thisrow{error},
        ] {data2/trains1-time-vs-noise-enum.csv};
        
        \addplot+[color=red,mark=square*,error bars/.cd,y dir=both,y explicit]
        table [
        x expr=\thisrow{x_data} * 100,
        y=y_data,
        col sep=comma,
        y error plus expr=\thisrow{error},y error minus expr=\thisrow{error},
        ] {data2/trains1-time-vs-noise-noisy.csv};
        
        \addplot+[error bars/.cd,y dir=both,y explicit]
        table [
        x expr=\thisrow{x_data} * 100,
        y=y_data,
        col sep=comma,
        y error plus expr=\thisrow{error},y error minus expr=\thisrow{error},
        ] {data2/trains1-time-vs-noise-no-min.csv};
        
        \addplot+[error bars/.cd,y dir=both,y explicit]
        table [
        x expr=\thisrow{x_data} * 100,
        y=y_data,
        col sep=comma,
        y error plus expr=\thisrow{error},y error minus expr=\thisrow{error},
        ] {data2/trains1-time-vs-noise-no-learn.csv};
        
        \addplot+[error bars/.cd,y dir=both,y explicit]
        table [
        x expr=\thisrow{x_data} * 100,
        y=y_data,
        col sep=comma,
        y error plus expr=\thisrow{error},y error minus expr=\thisrow{error},
        ] {data2/trains1-time-vs-noise-no-size.csv};
        
        \end{axis}
        \end{tikzpicture}
    \end{subfigure}
\caption{East-West Trains predictive accuracies and learning times of Noisy Popper variants (in seconds) for program \texttt{h$_1$} when varying percentage of noisy training data. Standard error is depicted by bars.}
\label{fig:6}
\end{figure}
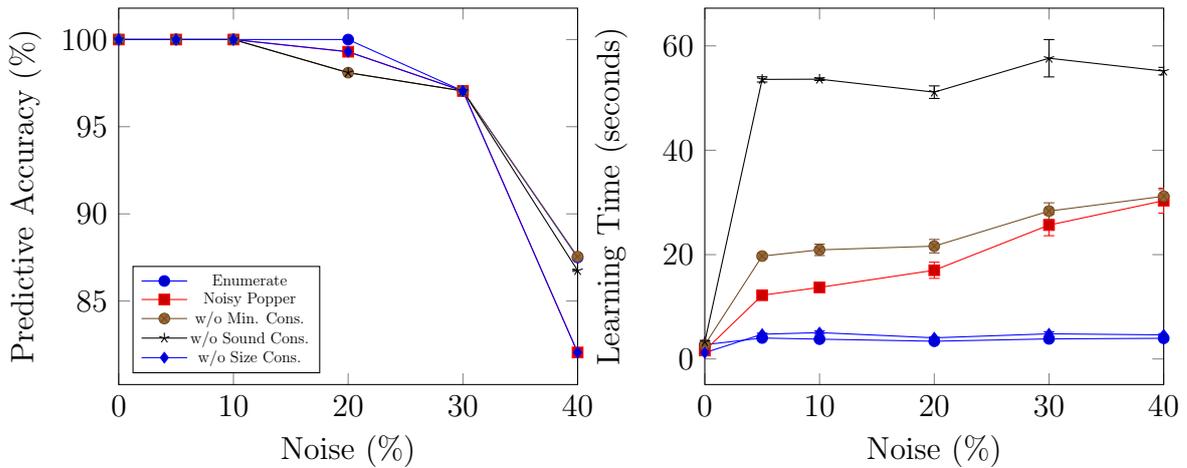

\begin{figure}[ht]
    \begin{subfigure}{.5\linewidth}
    \centering
        \begin{tikzpicture}
        \begin{axis}[
        width=\linewidth,
        legend pos=south west,
        legend style={nodes={scale=0.5, transform shape}},
        xmin=0,
        xmax=40,
        xlabel={Noise ($\%$)},
        ylabel={Predictive Accuracy ($\%$)}]
        
        \addplot+[error bars/.cd,y dir=both,y explicit]
        table [
        x expr=\thisrow{x_data} * 100,
        y=y_data,
        col sep=comma,
        y error plus expr=\thisrow{error},y error minus expr=\thisrow{error},
        ] {data2/trains2-acc-vs-noise-enum.csv};
        
        \addplot+[color=red,mark=square*,error bars/.cd,y dir=both,y explicit]
        table [
        x expr=\thisrow{x_data} * 100,
        y=y_data,
        col sep=comma,
        y error plus expr=\thisrow{error},y error minus expr=\thisrow{error},
        ] {data2/trains2-acc-vs-noise-noisy.csv};
        
        \addplot+[error bars/.cd,y dir=both,y explicit]
        table [
        x expr=\thisrow{x_data} * 100,
        y=y_data,
        col sep=comma,
        y error plus expr=\thisrow{error},y error minus expr=\thisrow{error},
        ] {data2/trains2-acc-vs-noise-no-min.csv};
        
        \addplot+[error bars/.cd,y dir=both,y explicit]
        table [
        x expr=\thisrow{x_data} * 100,
        y=y_data,
        col sep=comma,
        y error plus expr=\thisrow{error},y error minus expr=\thisrow{error},
        ] {data2/trains2-acc-vs-noise-no-learn.csv};
        
        \addplot+[error bars/.cd,y dir=both,y explicit]
        table [
        x expr=\thisrow{x_data} * 100,
        y=y_data,
        col sep=comma,
        y error plus expr=\thisrow{error},y error minus expr=\thisrow{error},
        ] {data2/trains2-acc-vs-noise-no-size.csv};
        
        \legend{{Enumerate}, {Noisy Popper}, {w/o Min. Cons.}, {w/o Sound Cons.}, {w/o Size Cons.}}
        \end{axis}
        \end{tikzpicture}
    \end{subfigure}
    \begin{subfigure}{.5\linewidth}
    \centering
        \begin{tikzpicture}
        \begin{axis}[
        width=\linewidth,
        legend pos=north west,
        legend style={nodes={scale=0.3, transform shape}},
        xmin=0,
        xmax=40,
        xlabel={Noise ($\%$)},
        ylabel={Learning Time (seconds)}]
        
        \addplot+[error bars/.cd,y dir=both,y explicit]
        table [
        x expr=\thisrow{x_data} * 100,
        y=y_data,
        col sep=comma,
        y error plus expr=\thisrow{error},y error minus expr=\thisrow{error},
        ] {data2/trains2-time-vs-noise-enum.csv};
        
        \addplot+[color=red,mark=square*,error bars/.cd,y dir=both,y explicit]
        table [
        x expr=\thisrow{x_data} * 100,
        y=y_data,
        col sep=comma,
        y error plus expr=\thisrow{error},y error minus expr=\thisrow{error},
        ] {data2/trains2-time-vs-noise-noisy.csv};
        
        \addplot+[error bars/.cd,y dir=both,y explicit]
        table [
        x expr=\thisrow{x_data} * 100,
        y=y_data,
        col sep=comma,
        y error plus expr=\thisrow{error},y error minus expr=\thisrow{error},
        ] {data2/trains2-time-vs-noise-no-min.csv};
        
        \addplot+[error bars/.cd,y dir=both,y explicit]
        table [
        x expr=\thisrow{x_data} * 100,
        y=y_data,
        col sep=comma,
        y error plus expr=\thisrow{error},y error minus expr=\thisrow{error},
        ] {data2/trains2-time-vs-noise-no-learn.csv};
        
        \addplot+[error bars/.cd,y dir=both,y explicit]
        table [
        x expr=\thisrow{x_data} * 100,
        y=y_data,
        col sep=comma,
        y error plus expr=\thisrow{error},y error minus expr=\thisrow{error},
        ] {data2/trains2-time-vs-noise-no-size.csv};
        
        \end{axis}
        \end{tikzpicture}
    \end{subfigure}
\caption{East-West Trains predictive accuracies and learning times of Noisy Popper variants (in seconds) for program \texttt{h$_2$} when varying percentage of noisy training data. Standard error is depicted by bars.}
\label{fig:7}
\end{figure}
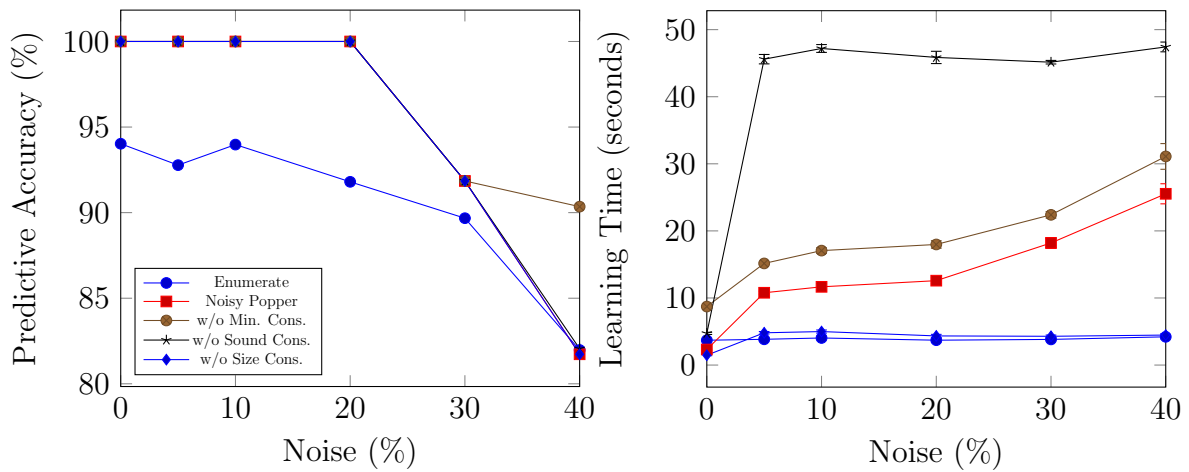

\subsection{Experiment 2: List Manipulations}
This set of experiments is identical to those in Section 6.1.2.

\paragraph{Materials}
The materials are identical to those in Section 6.1.2.

\paragraph{Methods}
The methods are identical to those in Section 6.1.2.

\paragraph{Results and Analysis}
Table~\ref{table:5} depicts the predictive accuracy for Noisy Popper and its variants on each of the list manipulation tasks. Each performs well on most tasks except for \texttt{finddup} where again no system can find an accurate solution, \texttt{sorted} where only Noisy Popper performs well, and \texttt{threesame} where each variant performs slightly worse than Noisy Popper as noise increases. Enumerate notably performs poorly on several datasets, indicating that it could not find a correct solution in the time allotted while Noisy Popper and its variants could. The \texttt{sorted} task gives the best indication that the minimal constraints provide the biggest impact to predictive accuracy with sound hypothesis constraints contributing to a smaller degree and size constraints being ineffective except at higher noise levels where it prevents Noisy Popper from overfitting.  

\bigskip \noindent Table~\ref{table:6} again demonstrates that typically without the size constraints, Noisy Popper runs much more efficiently, though not as quickly as Enumerate on average. However, in some  tasks such as \texttt{member} and \texttt{threesame}, without noise Noisy Popper without size constraints finds the correct solution faster than any other system. We also can again see that typically, the minimal and sound constraints provide Noisy Popper with considerable speedup. Figure~\ref{fig:10} below depicts the predictive test accuracy of the best hypothesis being maintained by each system versus the total number of programs each program has generated and learned from for the evens task with 5$\%$ training noise. This plot is exemplary of many of the list manipulation tasks over all noise levels and demonstrates how Noisy Popper with and without size constraints typically requires the fewest number of programs to learn from to find an optimal solution. Noisy Popper without sound constraints typically requires generating several more hypotheses to accomplish the same and without sound hypothesis constraints requires yet more programs. If Enumerate finds an optimal solution, it typically requires generating the greatest number of hypotheses. Overall, this data suggests the answer to \textbf{Q3} is again that the enhancements of Noisy Popper do not individually provide large benefits to the accuracy of the system except for specific datasets. However, minimal constraints and sound constraints do provide speedup to the system and help the system find correct solutions quicker than without. They additionally reduce the total number of programs which must be generated to find an optimal solution. This indicates that both sets of constraints contribute are beneficial to the overall system. Size constraints can provide deterrence against overfitting, but only in specific cases and at the cost of great inefficiency. 
\begin{figure}[ht!]
    \centering
    \begin{tikzpicture}[thick,scale=0.8, every node/.style={scale=0.8}]
    \begin{axis}[
    width=\linewidth,
    legend pos=south east,
    legend style={nodes={scale=1, transform shape}},
    xlabel={Number of Programs Generated},
    ylabel={Predictive Accuracy ($\%$)}]
    
    \addplot+[error bars/.cd,y dir=both,y explicit]
    table [
    x=x_data,
    y expr=\thisrow{y_data} * 100,
    col sep=comma,
    y error plus expr=\thisrow{error} * 100,y error minus expr=\thisrow{error} * 100,
    ] {data3/alleven_test_5_exp1_Enumerate.csv};
    
    \addplot+[color=red,mark=square*,error bars/.cd,y dir=both,y explicit]
    table [
    x=x_data,
    y expr=\thisrow{y_data} * 100,
    col sep=comma,
    y error plus expr=\thisrow{error} * 100,y error minus expr=\thisrow{error} * 100,
    ] {data3/alleven_test_5_exp1_Noisy.csv};
    
    \addplot+[error bars/.cd,y dir=both,y explicit]
    table [
    x=x_data,
    y expr=\thisrow{y_data} * 100,
    col sep=comma,
    y error plus expr=\thisrow{error} * 100,y error minus expr=\thisrow{error} * 100,
    ] {data3/alleven_test_5_exp1_No_Minimal.csv};
    
    \addplot+[error bars/.cd,y dir=both,y explicit]
    table [
    x=x_data,
    y expr=\thisrow{y_data} * 100,
    col sep=comma,
    y error plus expr=\thisrow{error} * 100,y error minus expr=\thisrow{error} * 100,
    ] {data3/alleven_test_5_exp1_No_Learn.csv};
    
    \addplot+[error bars/.cd,y dir=both,y explicit]
    table [
    x=x_data,
    y expr=\thisrow{y_data} * 100,
    col sep=comma,
    y error plus expr=\thisrow{error} * 100,y error minus expr=\thisrow{error} * 100,
    ] {data3/alleven_test_5_exp1_No_Size.csv};

    \legend{{Enumerate}, {Noisy Popper}, {w/o Min. Cons.}, {w/o Sound Cons.}, {w/o Size Cons.}}
    \end{axis}
    \end{tikzpicture}
\caption{Predictive accuracies of maintained best programs for Noisy Popper variants versus the number of programs generated by each system on evens dataset with 5$\%$ training noise. Standard error is depicted by bars.}
\label{fig:10}
\end{figure}
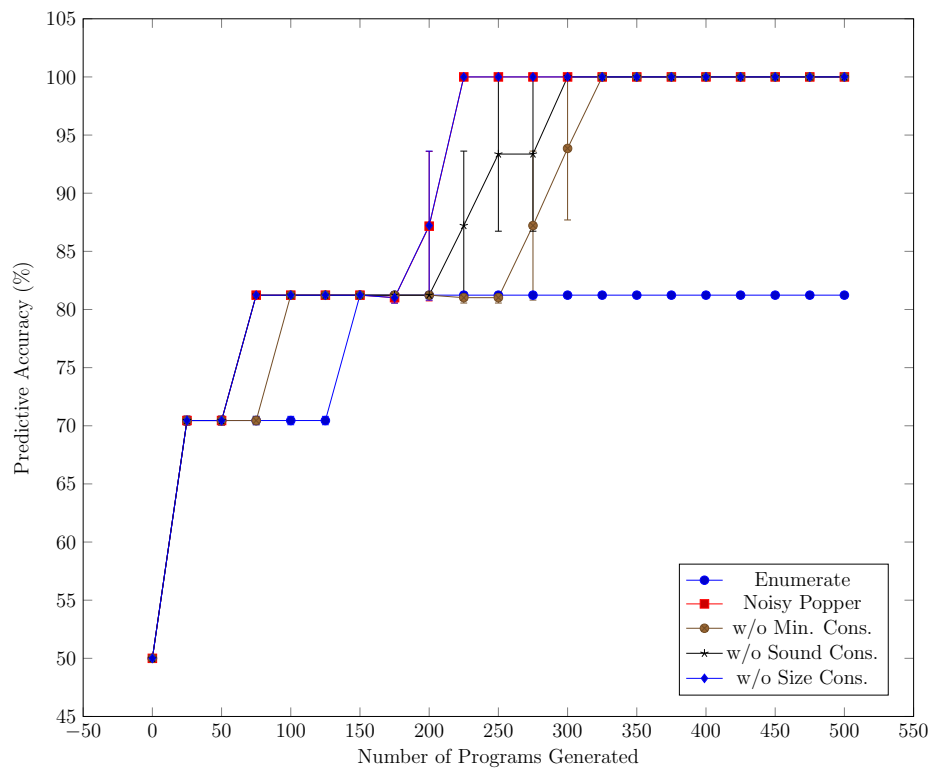

\begin{table}[p]
    \centering
    \resizebox{1\columnwidth}{!}{%
    \begin{tabular}{c|c|c|c|c|c|c} 
    \toprule
    \textbf{Name} & \textbf{Noise ($\%$)} & \textbf{Enumerate} & \textbf{Noisy Popper} & \textbf{w/o Minimal Constraints} & \textbf{w/o Sound Constraints} & \textbf{w/o Sound Size Constraints}\\
    \midrule
    \multirow{4}{*}{\texttt{addhead}} & 0 & \textbf{100}$\pm{0}$ & \textbf{100}$\pm{0}$ & \textbf{100}$\pm{0}$ & \textbf{100}$\pm{0}$ & \textbf{100}$\pm{0}$\\
    & 5 & \textbf{100}$\pm{0}$ & \textbf{100}$\pm{0}$ & \textbf{100}$\pm{0}$ & \textbf{100}$\pm{0}$ & \textbf{100}$\pm{0}$\\
    & 10 & \textbf{100}$\pm{0}$ & \textbf{100}$\pm{0}$ & \textbf{100}$\pm{0}$ & \textbf{100}$\pm{0}$ & \textbf{100}$\pm{0}$\\
    & 20 & \textbf{100}$\pm{0}$ & \textbf{100}$\pm{0}$ & \textbf{100}$\pm{0}$ & \textbf{100}$\pm{0}$ & \textbf{100}$\pm{0}$\\
    \midrule
    \multirow{4}{*}{\texttt{droplast}} & 0 & 50$\pm{0}$ & \textbf{100}$\pm{0}$ & \textbf{100}$\pm{0}$ & \textbf{100}$\pm{0}$ & \textbf{100}$\pm{0}$\\ 
    & 5 & 50$\pm{0}$ & \textbf{100}$\pm{0}$ & \textbf{100}$\pm{0}$ & \textbf{100}$\pm{0}$ & \textbf{100}$\pm{0}$\\
    & 10 & 50$\pm{0}$ & \textbf{100}$\pm{0}$ & \textbf{100}$\pm{0}$ & \textbf{100}$\pm{0}$ & \textbf{100}$\pm{0}$\\
    & 20 & 50$\pm{0}$ & \textbf{100}$\pm{0}$ & \textbf{100}$\pm{0}$ & \textbf{100}$\pm{0}$ & \textbf{100}$\pm{0}$\\
    \midrule
    \multirow{4}{*}{\texttt{evens}} & 0 & 82$\pm{0}$ & \textbf{100}$\pm{0}$ & \textbf{100}$\pm{0}$ & \textbf{100}$\pm{0}$ & \textbf{100}$\pm{0}$\\
    & 5 & 82$\pm{0}$ & \textbf{100}$\pm{0}$ & \textbf{100}$\pm{0}$ & \textbf{100}$\pm{0}$ & \textbf{100}$\pm{0}$\\
    & 10 & 82$\pm{0}$ & \textbf{100}$\pm{0}$ & \textbf{100}$\pm{0}$ & \textbf{100}$\pm{0}$ & \textbf{100}$\pm{0}$\\
    & 20 & 81$\pm{0}$ & \textbf{99}$\pm{0}$ & \textbf{100}$\pm{0}$ & \textbf{100}$\pm{0}$ & \textbf{100}$\pm{0}$\\
    \midrule
    \multirow{4}{*}{\texttt{finddup}} & 0 & 53$\pm{0}$ & 54$\pm{0}$ & 53$\pm{0}$ & 53$\pm{0}$ & 53$\pm{0}$\\
    & 5 & \textbf{52}$\pm{0}$ & \textbf{52}$\pm{0}$ & \textbf{52}$\pm{0}$ & \textbf{52}$\pm{0}$ & \textbf{52}$\pm{0}$\\
    & 10 & \textbf{52}$\pm{0}$ & \textbf{52}$\pm{0}$ & \textbf{52}$\pm{0}$ & \textbf{52}$\pm{0}$ & \textbf{52}$\pm{0}$\\
    & 20 & 50$\pm{0}$ & \textbf{53}$\pm{0}$ & 50$\pm{0}$ & 49$\pm{0}$ & 50$\pm{0}$\\
    \midrule
    \multirow{4}{*}{\texttt{last}} & 0 & \textbf{100}$\pm{0}$ & \textbf{100}$\pm{0}$ & \textbf{100}$\pm{0}$ & \textbf{100}$\pm{0}$ & \textbf{100}$\pm{0}$\\
    & 5 & \textbf{100}$\pm{0}$ & \textbf{100}$\pm{0}$ & \textbf{100}$\pm{0}$ & \textbf{100}$\pm{0}$ & \textbf{100}$\pm{0}$\\
    & 10 & \textbf{100}$\pm{0}$ & \textbf{100}$\pm{0}$ & \textbf{100}$\pm{0}$ & \textbf{100}$\pm{0}$ & \textbf{100}$\pm{0}$\\
    & 20 & \textbf{100}$\pm{0}$ & \textbf{100}$\pm{0}$ & \textbf{100}$\pm{0}$ & \textbf{100}$\pm{0}$ & \textbf{100}$\pm{0}$\\
    \midrule
    \multirow{4}{*}{\texttt{len}} & 0 & \textbf{100}$\pm{0}$ & \textbf{100}$\pm{0}$ & \textbf{100}$\pm{0}$ & \textbf{100}$\pm{0}$ & \textbf{100}$\pm{0}$\\
    & 5 & \textbf{100}$\pm{0}$ & \textbf{100}$\pm{0}$ & \textbf{100}$\pm{0}$ & \textbf{100}$\pm{0}$ & \textbf{100}$\pm{0}$\\
    & 10 & \textbf{100}$\pm{0}$ & \textbf{100}$\pm{0}$ & \textbf{100}$\pm{0}$ & \textbf{100}$\pm{0}$ & \textbf{100}$\pm{0}$\\
    & 20 & \textbf{100}$\pm{0}$ & \textbf{100}$\pm{0}$ & \textbf{100}$\pm{0}$ & \textbf{100}$\pm{0}$ & \textbf{100}$\pm{0}$\\
    \midrule
    \multirow{4}{*}{\texttt{member}} & 0 & \textbf{100}$\pm{0}$ & \textbf{100}$\pm{0}$ & \textbf{100}$\pm{0}$ & \textbf{100}$\pm{0}$ & \textbf{100}$\pm{0}$\\
    & 5 & \textbf{100}$\pm{0}$ & \textbf{100}$\pm{0}$ & \textbf{100}$\pm{0}$ & \textbf{100}$\pm{0}$ & \textbf{100}$\pm{0}$\\
    & 10 & \textbf{100}$\pm{0}$ & \textbf{100}$\pm{0}$ & \textbf{100}$\pm{0}$ & \textbf{100}$\pm{0}$ & \textbf{100}$\pm{0}$\\
    & 20 & \textbf{100}$\pm{0}$ & \textbf{100}$\pm{0}$ & \textbf{100}$\pm{0}$ & \textbf{100}$\pm{0}$ & \textbf{100}$\pm{0}$\\
    \midrule
    \multirow{4}{*}{\texttt{sorted}} & 0 & 76$\pm{0}$ & \textbf{100}$\pm{0}$ & 92$\pm{0.1}$ & 76$\pm{0}$ & \textbf{100}$\pm{0}$\\
    & 5 & 78$\pm{0.5}$ & \textbf{100}$\pm{0}$ & 79$\pm{0.2}$ & 81$\pm{1.5}$ & \textbf{100}$\pm{0}$\\
    & 10 & 76$\pm{0}$ & \textbf{100}$\pm{0}$ & 76$\pm{0}$ & 83$\pm{0.1}$ & 93$\pm{0.1}$\\
    & 20 & 75$\pm{0}$ & \textbf{100}$\pm{0}$ & 75$\pm{0}$ & 83$\pm{0.1}$ & 91$\pm{0.3}$\\
    \midrule
    \multirow{4}{*}{\texttt{threesame}} & 0 & \textbf{100}$\pm{0}$ & \textbf{100}$\pm{0}$ & \textbf{100}$\pm{0}$ & \textbf{100}$\pm{0}$ & \textbf{100}$\pm{0}$\\
    & 5 & 99$\pm{0}$ & \textbf{100}$\pm{0}$ & \textbf{100}$\pm{0}$ & \textbf{100}$\pm{0}$ & \textbf{100}$\pm{0}$\\
    & 10 & \textbf{99}$\pm{0}$ & \textbf{99}$\pm{0}$ & 98$\pm{0}$ & 98$\pm{0}$ & 98$\pm{0}$\\
    & 20 & \textbf{99}$\pm{0}$ & \textbf{99}$\pm{0}$ & \textbf{99}$\pm{0}$ & \textbf{99}$\pm{0}$ & \textbf{99}$\pm{0}$\\
    \midrule
    \end{tabular}}
    \caption{Predictive accuracy for Noisy Popper variants on list manipulation problems. Accuracies are rounded to the nearest integer and errors to the nearest tenth. Errors are standard.}
    \label{table:5}
\end{table}

\begin{table}[p]
    \centering
    \resizebox{1\columnwidth}{!}{%
    \begin{tabular}{c|c|c|c|c|c|c} 
    \toprule
    \textbf{Name} & \textbf{Noise ($\%$)} & \textbf{Enumerate} & \textbf{Noisy Popper} & \textbf{w/o Minimal Constraints} & \textbf{w/o Sound Constraints} & \textbf{w/o Sound Size Constraints}\\
    \midrule
    \multirow{4}{*}{\texttt{addhead}} & 0 & 2$\pm{0.2}$ & 2$\pm{0.1}$ & 3$\pm{0.2}$ & 3$\pm{0.2}$ & \textbf{1}$\pm{0.1}$\\
    & 5 & \textbf{4}$\pm{0.1}$ & 79$\pm{4}$ & 68$\pm{2}$ & 183$\pm{7}$ & 80$\pm{0.2}$\\
    & 10 & \textbf{4}$\pm{0.2}$ & 74$\pm{2}$ & 65$\pm{5}$ & 195$\pm{5}$ & 91$\pm{1}$\\
    & 20 & \textbf{5}$\pm{1}$ & 74$\pm{3}$ & 81$\pm{3}$ & 211$\pm{7}$ & 102$\pm{2}$\\
    \midrule
    \multirow{4}{*}{\texttt{droplast}} & 0 & \textbf{5}$\pm{0}$ & 81$\pm{39}$ & 71$\pm{5}$ & 205$\pm{8}$ & 95$\pm{3}$\\
    & 5 & \textbf{5}$\pm{0}$ & 135$\pm{45}$ & 89$\pm{3}$ & 215$\pm{8}$ & 112$\pm{3}$\\
    & 10 & \textbf{6}$\pm{0.1}$ & 142$\pm{33}$ & 92$\pm{2}$ & 217$\pm{8}$ & 110$\pm{2}$\\
    & 20 & \textbf{6}$\pm{0.2}$ & 137$\pm{6}$ & 92$\pm{2}$ & 217$\pm{8}$ & 138$\pm{5}$\\
    \midrule
    \multirow{4}{*}{\texttt{evens}} & 0 & 6$\pm{0.4}$ & 7$\pm{0.5}$ & 14$\pm{2}$ & 17$\pm{3}$ & \textbf{2}$\pm{0.1}$\\
    & 5 & \textbf{5}$\pm{0.4}$ & 38$\pm{1}$ & 47$\pm{2}$ & 77$\pm{3}$ & 14$\pm{0.4}$\\
    & 10 & \textbf{6}$\pm{0.3}$ & 45$\pm{0.5}$ & 58$\pm{0}$ & 102$\pm{3}$ & 15$\pm{0.2}$\\
    & 20 & \textbf{5}$\pm{0.3}$ & 40$\pm{3}$ & 45$\pm{5}$ & 85$\pm{9}$ & 13$\pm{1}$\\
    \midrule
    \multirow{4}{*}{\texttt{finddup}} & 0 & \textbf{7}$\pm{1}$ & 39$\pm{2}$ & 51$\pm{3}$ & 95$\pm{3}$ & 13$\pm{2}$\\
    & 5 & \textbf{9}$\pm{1}$ & 36$\pm{1}$ & 49$\pm{2}$ & 96$\pm{2}$ & 18$\pm{3}$\\
    & 10 & \textbf{5}$\pm{0.6}$ & 39$\pm{2}$ & 40$\pm{0.8}$ & 83$\pm{0.6}$ & 10$\pm{2}$\\
    & 20 & \textbf{6}$\pm{1}$ & 40$\pm{2}$ & 45$\pm{0.3}$ & 90$\pm{2}$ & 12$\pm{1}$\\
    \midrule
    \multirow{4}{*}{\texttt{last}} & 0 & \textbf{3}$\pm{0.4}$ & 15$\pm{4}$ & 13$\pm{0.6}$ & 9$\pm{2}$ & \textbf{2}$\pm{0.4}$\\
    & 5 & \textbf{4}$\pm{0.2}$ & 19$\pm{0.5}$ & 52$\pm{2}$ & 61$\pm{3}$ & 18$\pm{0.4}$\\
    & 10 & \textbf{4}$\pm{0.1}$ & 21$\pm{0.3}$ & 51$\pm{0.8}$ & 60$\pm{2}$ & 18$\pm{0.2}$\\
    & 20 & \textbf{3}$\pm{0.2}$ & 20$\pm{0.3}$ & 52$\pm{1}$ & 62$\pm{2}$ & 19$\pm{0.1}$\\
    \midrule
    \multirow{4}{*}{\texttt{len}} & 0 & 2$\pm{0.3}$ & 2$\pm{0.1}$ & 3$\pm{0.2}$ & 2$\pm{0.2}$ & \textbf{1}$\pm{0.1}$\\
    & 5 & \textbf{3}$\pm{0.2}$ & 59$\pm{6}$ & 60$\pm{4}$ & 67$\pm{2}$ & 15$\pm{2}$\\
    & 10 & \textbf{3}$\pm{0.2}$ & 56$\pm{2}$ & 62$\pm{2}$ & 66$\pm{3}$ & 18$\pm{1}$\\
    & 20 & \textbf{3}$\pm{0.1}$ & 56$\pm{1}$ & 58$\pm{0.8}$ & 67$\pm{1}$ & 16$\pm{1}$\\
    \midrule
    \multirow{4}{*}{\texttt{member}} & 0 & 0.7$\pm{0.1}$ & 0.7$\pm{0}$ & 0.6$\pm{0.2}$ & 1$\pm{0.1}$ & \textbf{0.5}$\pm{0}$\\
    & 5 & \textbf{2}$\pm{0.1}$ & 23$\pm{3}$ & 33$\pm{0.5}$ & 35$\pm{0.2}$ & 22$\pm{0.1}$\\
    & 10 & \textbf{2}$\pm{0.2}$ & 23$\pm{3}$ & 30$\pm{0}$ & 26$\pm{0.2}$ & 20$\pm{0.1}$\\
    & 20 & \textbf{2}$\pm{0.2}$ & 23$\pm{3}$ & 25$\pm{1}$ & 28$\pm{0.5}$ & 16$\pm{0.5}$\\
    \midrule
    \multirow{4}{*}{\texttt{sorted}} & 0 & \textbf{5}$\pm{0.3}$ & 26$\pm{4}$ & 43$\pm{7}$ & 71$\pm{2}$ & 11$\pm{0.6}$\\
    & 5 & \textbf{4}$\pm{0.8}$ & 46$\pm{0.1}$ & 55$\pm{2}$ & 69$\pm{0.5}$ & 15$\pm{0.3}$\\
    & 10 & \textbf{5}$\pm{0.2}$ & 43$\pm{0.9}$ & 50$\pm{2}$ & 69$\pm{0.4}$ & 15$\pm{0.7}$\\
    & 20 & \textbf{5}$\pm{0.2}$ & 43$\pm{0.2}$ & 47$\pm{0.7}$ & 72$\pm{0.5}$ & 14$\pm{0.4}$\\
    \midrule
    \multirow{4}{*}{\texttt{threesame}} & 0 & 0.5$\pm{0.1}$ & 0.5$\pm{0.1}$ & 0.6$\pm{0.1}$ & 0.6$\pm{0.1}$ & \textbf{0.4}$\pm{0}$\\
    & 5 & 5$\pm{0.2}$ & \textbf{2}$\pm{0.1}$ & 82$\pm{5}$ & 5$\pm{0}$ & \textbf{2}$\pm{0.1}$\\
    & 10 & 5$\pm{0.1}$ & 3$\pm{0.3}$ & 78$\pm{2}$ & 2$\pm{0.1}$ & \textbf{0.8}$\pm{0}$\\
    & 20 & 5$\pm{0.2}$ & 4$\pm{0.2}$ & 76$\pm{2}$ & 3$\pm{0.3}$ & \textbf{1}$\pm{0.1}$\\
    \midrule
    \end{tabular}}
    \caption{Learning times for Noisy Popper variants on list manipulation problems. Times are rounded to the nearest second if they are greater than 1 second and to the tenth otherwise. Errors are standard.}
    \label{table:6}
\end{table}

\newpage

\subsection{Experiment 3: IGGP Problems}
This set of experiments is identical to those in Section 6.1.3 using the two IGGP problems minimal decay and rps. 

\paragraph{Materials}
The materials are identical to those in Section 6.1.3.

\paragraph{Methods}
The methods are identical to those in Section 6.1.3.

\paragraph{Results and Analysis}
Figures~\ref{fig:8} and \ref{fig:9} both demonstrate that on both IGGP tasks, each Noisy Popper variant obtains roughly equal predictive accuracy for these levels of noise. For both minimal decay and rps tasks, Noisy Popper without size constraints again performs significantly more efficiently than the other variants, again due to the grounding blowup previously discussed. This data again suggests that the answer to \textbf{Q3} is that the enhancements of Noisy Popper do not necessarily improve its predictive accuracy over a brute force approach and that due to the grounding blowup of the size constraints, Noisy Popper performs much less efficiently than this brute force approach. 

\begin{figure}[ht]
    \begin{subfigure}{.5\linewidth}
    \centering
        \begin{tikzpicture}
        \begin{axis}[
        width=\linewidth,
        legend pos=south west,
        legend style={nodes={scale=0.45, transform shape}},
        xmin=0,
        xmax=40,
        xlabel={Noise ($\%$)},
        ylabel={Predictive Accuracy ($\%$)}]
        
        \addplot+[error bars/.cd,y dir=both,y explicit]
        table [
        x expr=\thisrow{x_data} * 100,
        y=y_data,
        col sep=comma,
        y error plus expr=\thisrow{error},y error minus expr=\thisrow{error},
        ] {data2/minimal-decay-acc-vs-noise-enum.csv};
        
        \addplot+[color=red,mark=square*,error bars/.cd,y dir=both,y explicit]
        table [
        x expr=\thisrow{x_data} * 100,
        y=y_data,
        col sep=comma,
        y error plus expr=\thisrow{error},y error minus expr=\thisrow{error},
        ] {data2/minimal-decay-acc-vs-noise-noisy.csv};
        
        \addplot+[error bars/.cd,y dir=both,y explicit]
        table [
        x expr=\thisrow{x_data} * 100,
        y=y_data,
        col sep=comma,
        y error plus expr=\thisrow{error},y error minus expr=\thisrow{error},
        ] {data2/minimal-decay-acc-vs-noise-no-min.csv};
        
        \addplot+[error bars/.cd,y dir=both,y explicit]
        table [
        x expr=\thisrow{x_data} * 100,
        y=y_data,
        col sep=comma,
        y error plus expr=\thisrow{error},y error minus expr=\thisrow{error},
        ] {data2/minimal-decay-acc-vs-noise-no-learn.csv};
        
        \addplot+[error bars/.cd,y dir=both,y explicit]
        table [
        x expr=\thisrow{x_data} * 100,
        y=y_data,
        col sep=comma,
        y error plus expr=\thisrow{error},y error minus expr=\thisrow{error},
        ] {data2/minimal-decay-acc-vs-noise-no-size.csv};
        
        \legend{{Enumerate}, {Noisy Popper}, {w/o Min. Cons.}, {w/o Sound Cons.}, {w/o Size Cons.}}
        \end{axis}
        \end{tikzpicture}
    \end{subfigure}
    \begin{subfigure}{.5\linewidth}
    \centering
        \begin{tikzpicture}
        \begin{axis}[
        width=\linewidth,
        legend pos=north west,
        legend style={nodes={scale=0.3, transform shape}},
        xmin=0,
        xmax=40,
        xlabel={Noise ($\%$)},
        ylabel={Time (seconds)}]
        
        \addplot+[error bars/.cd,y dir=both,y explicit]
        table [
        x expr=\thisrow{x_data} * 100,
        y=y_data,
        col sep=comma,
        y error plus expr=\thisrow{error},y error minus expr=\thisrow{error},
        ] {data2/minimal-decay-time-vs-noise-enum.csv};
        
        \addplot+[color=red,mark=square*,error bars/.cd,y dir=both,y explicit]
        table [
        x expr=\thisrow{x_data} * 100,
        y=y_data,
        col sep=comma,
        y error plus expr=\thisrow{error},y error minus expr=\thisrow{error},
        ] {data2/minimal-decay-time-vs-noise-noisy.csv};
        
        \addplot+[error bars/.cd,y dir=both,y explicit]
        table [
        x expr=\thisrow{x_data} * 100,
        y=y_data,
        col sep=comma,
        y error plus expr=\thisrow{error},y error minus expr=\thisrow{error},
        ] {data2/minimal-decay-time-vs-noise-no-min.csv};
        
        \addplot+[error bars/.cd,y dir=both,y explicit]
        table [
        x expr=\thisrow{x_data} * 100,
        y=y_data,
        col sep=comma,
        y error plus expr=\thisrow{error},y error minus expr=\thisrow{error},
        ] {data2/minimal-decay-time-vs-noise-no-learn.csv};
        
        \addplot+[error bars/.cd,y dir=both,y explicit]
        table [
        x expr=\thisrow{x_data} * 100,
        y=y_data,
        col sep=comma,
        y error plus expr=\thisrow{error},y error minus expr=\thisrow{error},
        ] {data2/minimal-decay-time-vs-noise-no-size.csv};
        
        \end{axis}
        \end{tikzpicture}
    \end{subfigure}
\caption{IGGP minimal decay task predictive accuracy and time of Noisy Popper variants (in seconds) for when varying percentage of noisy training data. Standard error is depicted by bars.}
\label{fig:8}
\end{figure}
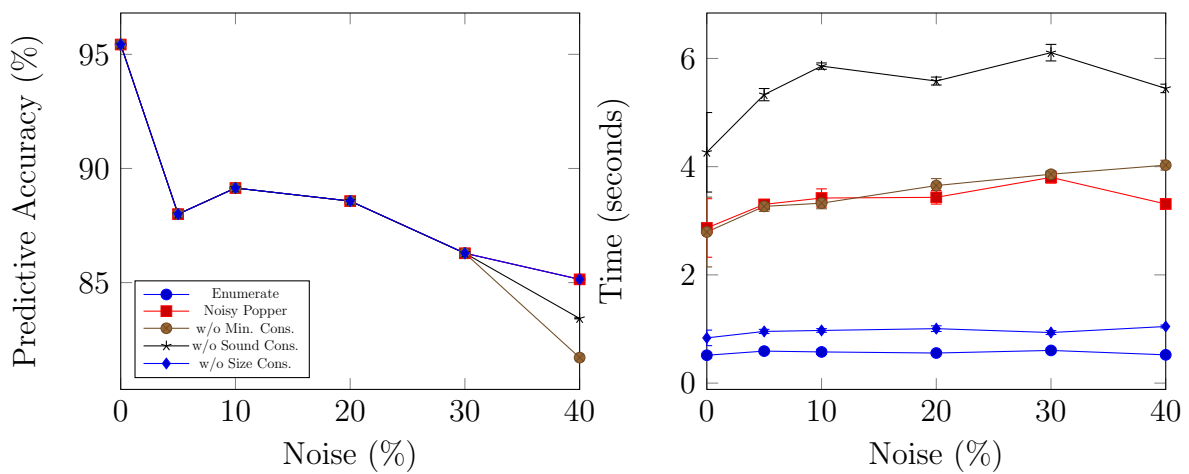

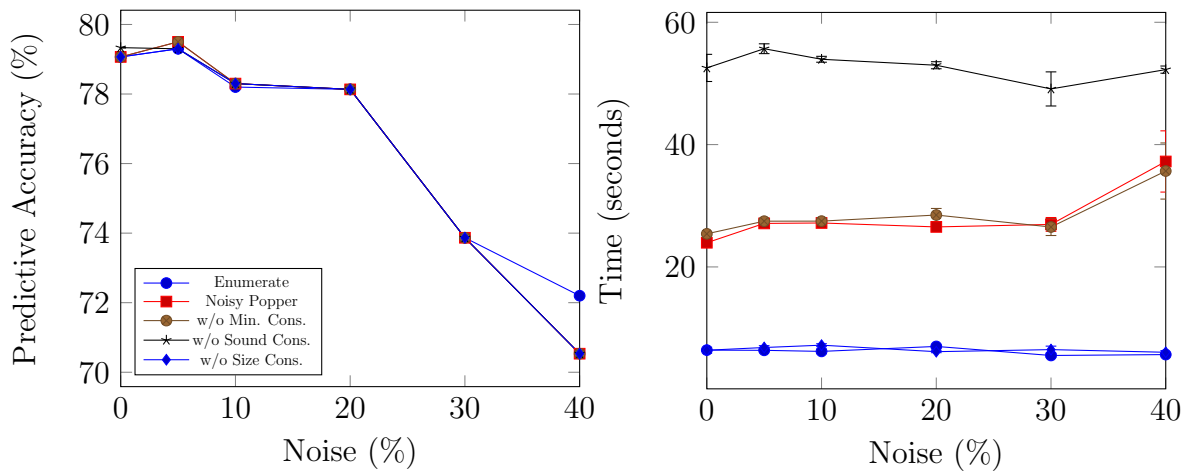
\begin{figure}[ht]
    \begin{subfigure}{.5\linewidth}
    \centering
        \begin{tikzpicture}
        \begin{axis}[
        width=\linewidth,
        legend pos=south west,
        legend style={nodes={scale=0.5, transform shape}},
        xmin=0,
        xmax=40,
        xlabel={Noise ($\%$)},
        ylabel={Predictive Accuracy ($\%$)}]
        
        \addplot+[error bars/.cd,y dir=both,y explicit]
        table [
        x expr=\thisrow{x_data} * 100,
        y=y_data,
        col sep=comma,
        y error plus expr=\thisrow{error},y error minus expr=\thisrow{error},
        ] {data2/rps-acc-vs-noise-enum.csv};
        
        \addplot+[color=red,mark=square*,error bars/.cd,y dir=both,y explicit]
        table [
        x expr=\thisrow{x_data} * 100,
        y=y_data,
        col sep=comma,
        y error plus expr=\thisrow{error},y error minus expr=\thisrow{error},
        ] {data2/rps-acc-vs-noise-noisy.csv};
        
        \addplot+[error bars/.cd,y dir=both,y explicit]
        table [
        x expr=\thisrow{x_data} * 100,
        y=y_data,
        col sep=comma,
        y error plus expr=\thisrow{error},y error minus expr=\thisrow{error},
        ] {data2/rps-acc-vs-noise-no-min.csv};
        
        \addplot+[error bars/.cd,y dir=both,y explicit]
        table [
        x expr=\thisrow{x_data} * 100,
        y=y_data,
        col sep=comma,
        y error plus expr=\thisrow{error},y error minus expr=\thisrow{error},
        ] {data2/rps-acc-vs-noise-no-learn.csv};
        
        \addplot+[error bars/.cd,y dir=both,y explicit]
        table [
        x expr=\thisrow{x_data} * 100,
        y=y_data,
        col sep=comma,
        y error plus expr=\thisrow{error},y error minus expr=\thisrow{error},
        ] {data2/rps-acc-vs-noise-no-size.csv};
        
        \legend{{Enumerate}, {Noisy Popper}, {w/o Min. Cons.}, {w/o Sound Cons.}, {w/o Size Cons.}}
        \end{axis}
        \end{tikzpicture}
    \end{subfigure}
    \begin{subfigure}{.5\linewidth}
    \centering
        \begin{tikzpicture}
        \begin{axis}[
        width=\linewidth,
        legend pos=north west,
        legend style={nodes={scale=0.3, transform shape}},
        xmin=0,
        xmax=40,
        xlabel={Noise ($\%$)},
        ylabel={Time (seconds)}]
        
        \addplot+[error bars/.cd,y dir=both,y explicit]
        table [
        x expr=\thisrow{x_data} * 100,
        y=y_data,
        col sep=comma,
        y error plus expr=\thisrow{error},y error minus expr=\thisrow{error},
        ] {data2/rps-time-vs-noise-enum.csv};
        
        \addplot+[color=red,mark=square*,error bars/.cd,y dir=both,y explicit]
        table [
        x expr=\thisrow{x_data} * 100,
        y=y_data,
        col sep=comma,
        y error plus expr=\thisrow{error},y error minus expr=\thisrow{error},
        ] {data2/rps-time-vs-noise-noisy.csv};
        
        \addplot+[error bars/.cd,y dir=both,y explicit]
        table [
        x expr=\thisrow{x_data} * 100,
        y=y_data,
        col sep=comma,
        y error plus expr=\thisrow{error},y error minus expr=\thisrow{error},
        ] {data2/rps-time-vs-noise-no-min.csv};
        
        \addplot+[error bars/.cd,y dir=both,y explicit]
        table [
        x expr=\thisrow{x_data} * 100,
        y=y_data,
        col sep=comma,
        y error plus expr=\thisrow{error},y error minus expr=\thisrow{error},
        ] {data2/rps-time-vs-noise-no-learn.csv};
        
        \addplot+[error bars/.cd,y dir=both,y explicit]
        table [
        x expr=\thisrow{x_data} * 100,
        y=y_data,
        col sep=comma,
        y error plus expr=\thisrow{error},y error minus expr=\thisrow{error},
        ] {data2/rps-time-vs-noise-no-size.csv};
        
        \end{axis}
        \end{tikzpicture}
    \end{subfigure}
\caption{IGGP rps task predictive accuracy and time of Noisy Popper variants (in seconds) when varying percentage of noisy training data. Standard error is depicted by bars.}
\label{fig:9}
\end{figure}

\section{Summary}
In this chapter, we empirically evaluated Noisy Popper's performance against Normal Popper's as well as the effects of the individual enhancements of the Noisy Popper system. Noisy Popper was shown to better generalize to noisy datasets than Normal Popper for some tasks, but many experiments suggested that Normal Popper enhanced with an anytime algorithm approach can often generalize very well to these datasets. The minimal constraints and sound hypothesis constraints are effective at pruning the hypothesis space and aiding the system in finding optimal solutions by generating fewer hypotheses than the brute force Enumerate method requires. However, this comes at the cost of the expected inefficiency when compared to Normal Popper and Enumerate. The grounding blowup of the size constraints makes generating several thousand programs to learn from infeasible for the system. The following chapter will discuss work to be done in the future to mitigate these inefficiencies and other additions which may be considered for the system. We will also give a brief summary of this project, its findings and contributions as well as its limitations.
\end{chapter}
\begin{chapter}
{Conclusions}
This paper has discussed the theoretical background, implementation details, and empirical analysis of the Noisy Popper ILP system, an extension of the Normal Popper system \cite{popper} which is capable of generalizing to noisy datasets. The following sections give a critical summary and analysis of the work and address future work which could improve Noisy Popper's capabilities.

\section{Summary and Evaluation}
 Handling missclassified training examples is an important task in machine learning, though many ILP systems are not naturally capable of doing so. We have shown that the learning from failures (LFF) approach which Normal Popper takes to prune its hypothesis search space is not naturally conducive to this task. The relaxed LFF setting introduced in this paper takes a less strict approach to hypothesis search and in doing so demonstrates better theoretical capabilities of finding hypotheses which generalize well to noisy data. We proved several theoretical claims over how comparing hypotheses can identify sets of hypotheses which perform suboptimally in this relaxed setting under two scoring measures: $S_{ACC}$ which measures the training accuracy of a hypothesis and $S_{MDL}$ which weighs training accuracy against the size of the hypothesis.
 
 \bigskip \noindent In implementation, Noisy Popper adapts the approach taken by Normal Popper, relaxing the ASP hypothesis constraints which prune the hypothesis space and instead typically only generating constraints which are sound in the relaxed setting. In this way, the theoretical claims made over the relaxed LFF system are translated into hypothesis constraints which reduce the hypothesis space during the system's search. Many of these constraints however create a large cost in the logical grounding required by the system which leads to significant runtime inefficiencies. The experimental work demonstrated that Noisy Popper never generalizes worse than Normal Popper for both noisy and non-noisy datasets and is capable of exceeding the predictive accuracy of Normal Popper on several datasets. However, enhancing Normal Popper with an anytime algorithm approach makes the system very capable of generalizing to noisy data on its own. The main deficiency of Noisy Popper is its inefficiency which future work must address to make the system viable in practice. Despite these shortcomings, Noisy Popper shows promise of being a useful ILP system which accurately and efficiently generalizes to noisy datasets.

\section{Future Work}
\paragraph{Recusive Cases}
Several of the theoretical claims proved in this paper only discussed the suboptimality of non-recursive programs. Generalizing these claims or constructing new ones which discuss the suboptimality of recursive programs would make these claims more complete. Such claims could also be used in practice to greatly improve the efficiency of hypothesis search.

\paragraph{Scoring Metrics}
In this paper, only two scoring functions were discussed, $S_{ACC}$ and $S_{MDL}$ and all theoretical claims were derived under these two settings. Additional theory should be explored under both of these scorings and additional scorings such as those which measure coverage or entropy. Whether these claims manifest as new noise handling systems or as additions to Noisy Popper itself, determining new cases for hypothesis pruning can be used by many ILP systems moving forward and can better improve the searching efficiency of such systems.

\paragraph{Grounding Bottleneck}
The largest bottleneck for Noisy Popper currently is the need to ground all programs for all applicable sizes for the hypothesis cosntraints with programs size. This leads to a blowup in number of required groundings which is the expected cause of the significant learning time difference from Normal Popper to Noisy Popper. Changing the ASP encodings for these size constraints to eliminate this blowup should massively improve the overall efficiency of the system and make it viable for much larger problems than those tested in Chpater 6. Insufficient time led to this problem remaining unresolved upon completion of this project.

\paragraph{Subsumption Checking}
Another large efficiency issue present in the implementation is the naive method in which Noisy Popper compares programs and checks for subsumption. Rather than maintaining all previously seen programs as a list, data structures such as subsumption lattices may be used to reduce the overall number of subsumption checks needed. Given how expensive the subsumption is to check, or rather the incomplete subsumption we use in implementation, this is an improvement that would again boost the overall efficiency of the system and allow the system to better handle checking exceptionally large hypotheses.

\paragraph{Parallelization}
Work to make use of multi-core machines and parallelize the Popper system has been completed and has been shown to greatly improve the learning rate of the system. Combining the noise handling approaches used in Noisy Popper with this parallized approach should vastly improve the efficiency of the system. Such work however is not trivial and determining new theoretical claims about such an environment are necessary for a sound implementation. 

\end{chapter}


\addcontentsline{toc}{chapter}{Bibliography}
\bibliography{refs}        
\bibliographystyle{plain}  

\end{document}